\documentclass[twoside,11pt]{article} 
 \usepackage{arxiv}

\ShortHeadings{Clustering is semidefinitely not that hard}{Tepper, Sengupta, Chklovskii}

\firstpageno{1} 

\title{Clustering is semidefinitely not that hard: Nonnegative SDP for manifold disentangling}

\author{\name Mariano Tepper \email mtepper@flatironinstitute.org\\ 
\addr Flatiron Institute, Simons Foundation, NY, USA\\ 
\AND 
\name Anirvan M. Sengupta \email anirvans@physics.rutgers.edu\\
\addr Dept.~of Physics and Astronomy, Rutgers University, NJ, USA\\ 
\addr Flatiron Institute, Simons Foundation, NY, USA\\
\AND 
\name Dmitri Chklovskii \email dchklovskii@flatironinstitute.org\\ 
\addr Flatiron Institute, Simons Foundation, NY, USA\\
\addr NYU Langone Medical Center, New York, NY}

\usepackage{times}
\usepackage{graphicx} 
\graphicspath{{figures/}}
\usepackage[percent]{overpic}

\usepackage[labelformat=simple]{subcaption}

\usepackage{enumitem}

\usepackage{hyperref}

\usepackage{tepper}

\begin{document} 

\maketitle

\begin{abstract}
In solving hard computational problems, semidefinite program (SDP) relaxations often play an important role because they come with a guarantee of optimality. Here, we focus on a popular semidefinite relaxation of K-means clustering which yields the same solution as the non-convex original formulation for well segregated datasets. We report an unexpected finding: when data contains (greater than zero-dimensional) manifolds, the SDP solution captures such geometrical structures.
Unlike traditional manifold embedding techniques, our approach does not rely on manually defining a kernel but rather enforces locality via a nonnegativity constraint. We thus call our approach NOnnegative MAnifold Disentangling, or NOMAD. To build an intuitive understanding of its manifold learning capabilities, we develop a theoretical analysis of NOMAD on idealized datasets. While NOMAD is convex and the globally optimal solution can be found by generic SDP solvers with polynomial time complexity, they are too slow for modern datasets. To address this problem, we analyze a non-convex heuristic and present a new, convex and yet efficient, algorithm, based on the conditional gradient method. Our results render NOMAD a versatile, understandable, and powerful tool for manifold learning.
\end{abstract}


\section{Introduction}

In the quest for an algorithmic theory of biological neural networks, some of the authors have recently proposed a soft $K$-means clustering network that may model insect olfactory processing and other computations \citep{Pehlevan2017olfaction}. This network was derived by performing online optimization on the non-convex $K$-means objective function. Whereas the network dynamics and learning rules are biologically plausible, the nonconvexity of the objective makes it difficult to analyze the solutions and algorithm convergence. 

Here, to understand the solutions computed by the clustering neural network, we consider a convex SDP relaxation of  $K$-means  \citep{Kulis2007,Peng2007_sdk-kmeans,Awasthi2015}. Given data points $\{ \vect{x}_i \}_{i=1}^{n}$ we define the Gramian matrix, $\mat{D}$, such that $(\mat{D})_{ij} =  \transpose{\vect{x}_i} \vect{x}_j$. Then, we search for a cluster co-association matrix $\mat{Q}$, such that $(\mat{Q})_{ij} = 1$ if points $i$ and $j$ belong to the same cluster and  $(\mat{Q})_{ij} = 0$ if they do not. The optimum $\mat{Q}_*$ can be found by solving the following optimization problem (the acronym will be explained below): 
\begin{equation}
	\mat{Q}_* = \argmax_{\mat{Q} \in \Real^{n \times n}}
	\traceone{\mat{D} \mat{Q}}
	\quad\text{s.t.}\quad
	\mat{Q} \vect{1} = \vect{1} ,\
	\traceone{\mat{Q}} = K ,\
	\mat{Q} \succeq 0 ,\
	\mat{Q} \geq \mat{0} .
	\tag{NOMAD}
	\label[problem]{eq:sdp_kmeans}
\end{equation}
Its link with the original $K$-means clustering formulation is explained in \cref{sec:kmeans}.%
\footnote{
\noindent\textbf{Notation.}
$(\mat{X})_{ij}$, $(\mat{X})_{:j}$, $(\mat{X})_{i:}$ denote the $(i,j)$-th entry of matrix $\mat{X}$, the $j$-th column of $\mat{X}$, and the $i$-th row of $\mat{X}$, respectively. For vectors, we employ lowercase and we use a similar notation but with a single index.
We write $\mat{X} \geq \mat{0}$ if a matrix $\mat{X}$ is entry-wise nonnegative and $\mat{X} \succeq \mat{0}$ if it is positive semidefinite.
}

First, we focus on the question: what does NOMAD compute?
Until now, theoretical efforts have concentrated on showing that NOMAD is a good surrogate for $K$-means. \citet{Awasthi2015} study its solutions on datasets consisting of linearly separable clusters and demonstrate that they reproduce hard-clustering assignments of $K$-means. Moreover, the solution to NOMAD achieves hard clustering even for some datasets on which Lloyd's algorithm \citep{Lloyd1982} fails (i.e., \citet{Iguchi2015,Mixon2016}).
Related problems have been studied by \citet{Amini2014,Javanmard2016,Yu2012regularizers}.

\begin{figure}
	\centering
	\begin{subfigure}{.37\columnwidth}
		\centering
		\includegraphics[width=.95\columnwidth]{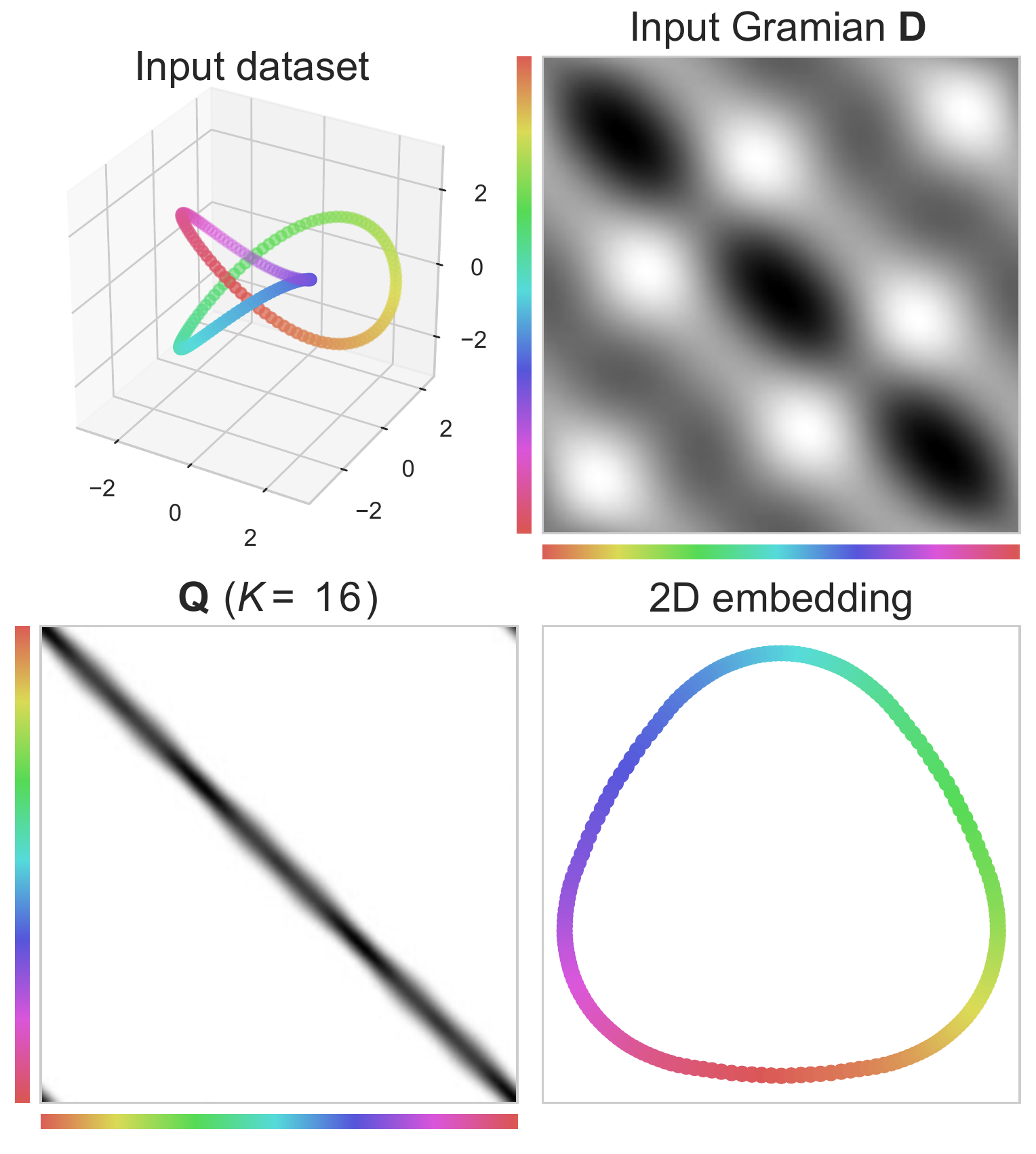}
		\caption{}
		\label{fig:trefoil}
	\end{subfigure}%
	\begin{subfigure}{.37\columnwidth}
		\centering
		\includegraphics[width=.95\columnwidth]{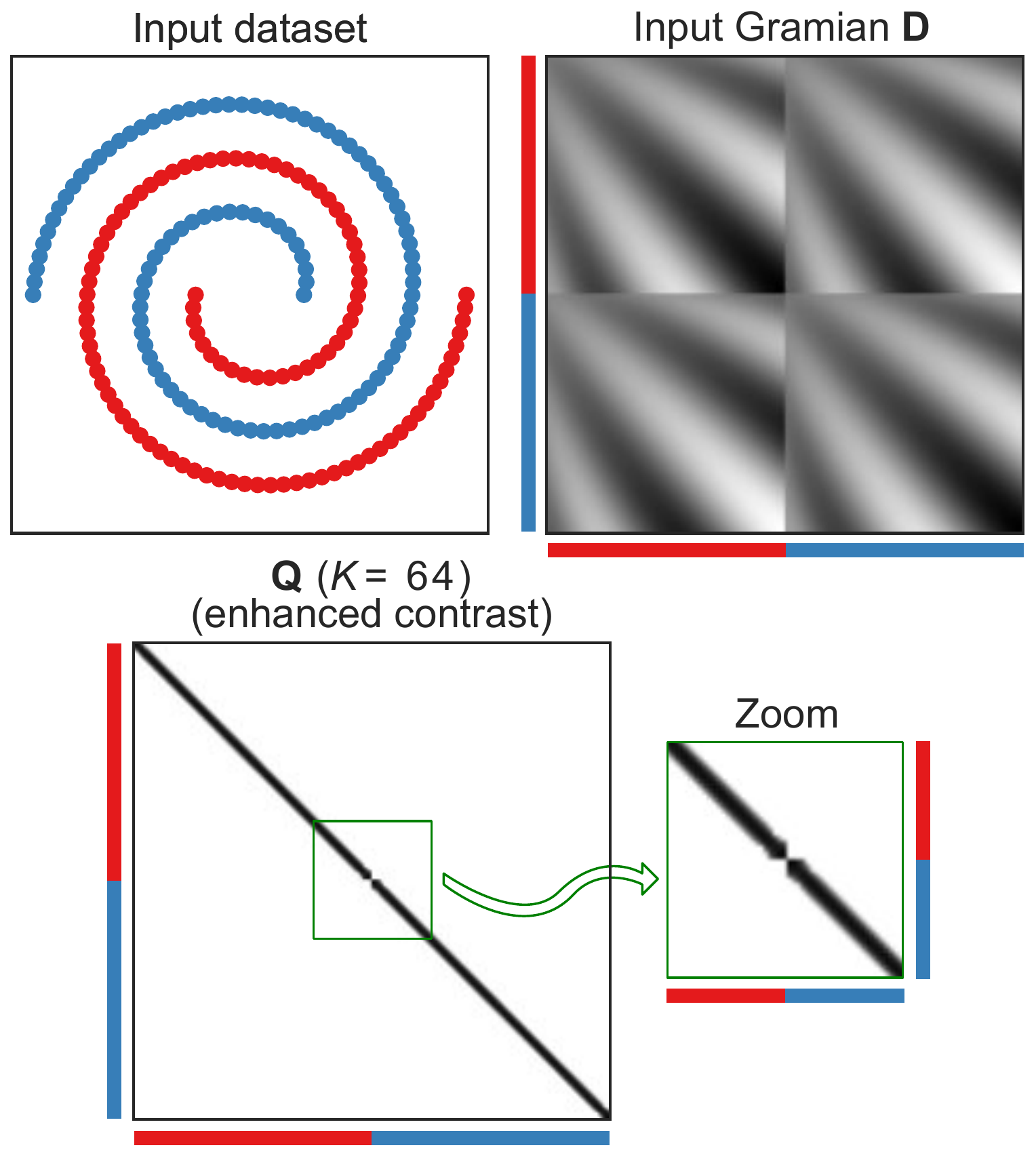}
		\caption{}
		\label{fig:double_swiss_roll}
	\end{subfigure}

	\caption{NOMAD, originally introduced as a convex relaxation of $K$-means clustering, surprisingly learns manifold structures in the data.
	\protect\subref{fig:trefoil} Learning the manifold of a trefoil knot, which cannot be ``untied'' in 3D without cutting it. NOMAD understands that this is a closed manifold, yielding a circulant matrix $\mat{Q}$, which can be ``unfolded'' in 2D.
	\protect\subref{fig:double_swiss_roll} Learning multiple manifolds with NOMAD. Although they are linearly non-separable, NOMAD correctly finds two submatrices, one for each manifold (for visual clarity, we enhance the contrast of $\mat{Q}$).
	}
	\label{fig:opening}
\end{figure}

In this work, we analyze NOMAD in a different regime than previous studies. Instead of focusing on cases and parameter settings where it approximates the original $K$-means formulation, we concentrate on alternative settings and discover that NOMAD is not merely a convex $K$-means imitator.
NOMAD finds the manifold structure in the data, even discriminating different manifolds. \cref{fig:opening} shows two examples of this unexpected behavior where NOMAD dissects the geometry of the data. Because of this and of the central role played by the nonnegativity constraint in the SDP we call it a NOnnegative MAnifold Disentangling (NOMAD).

The next question is: how can we compute these solutions?
Despite the theoretical advantages of convex optimization, in practice, the use of SDPs for clustering has remained limited. This is mainly due to the lack of efficient algorithms to solve the convex optimization problem.
We address this issue by presenting an efficient convex solver for NOMAD, based on the conditional gradient method. The new algorithm can handle large datasets, extending the applicability of NOMAD to more interesting and challenging scenarios.

\paragraph{Organization.} 
We first study the behavior of NOMAD theoretically by analyzing its solution for a simple synthetic example of a regular manifold with symmetry (\cref{sec:theory}). In this context, we demonstrate how NOMAD departs from standard $K$-means.
Building on this analysis, we suggest that NOMAD has non-trivial manifold learning capabilities (\cref{sec:manifold}) and demonstrate numerically NOMAD's good performance in non-trivial examples, including synthetic and real datasets.
Then, motivated by the relatively slow performance of standard SDP solvers, we focus on scaling NOMAD to large modern datasets. In \cref{sec:BM}, we study both theoretically and experimentally an heuristic non-convex Burer-Monteiro-style algorithm \citep{Kulis2007}.
Finally, we present a new convex and yet efficient algorithm for NOMAD. This algorithm allows us, for the first time, to study provable solutions of NOMAD on large datasets.
Our software is publicly available at \url{https://github.com/simonsfoundation/sdp_kmeans}.

\section{Theoretical analysis of manifold learning capabilities of NOMAD}
\label{sec:theory}

Starting with the appearance of Isomap~\citep{Tenenbaum2000} and locally-linear embedding (LLE) \citep{Roweis2000}, there has been outstanding progress in the area of manifold learning \citep[e.g.,][]{Belkin2003,Hadsell2006,Weinberger2006,Weiss2008}.
For a data matrix $\mat{X} = [ \vect{x}_i ]_{i=1}^{n}$ of column-vectors/points $\vect{x}_i \in \Real^d$, the majority of these modern methods have three steps:
\begin{enumerate}[nosep]
	\item Determine the neighbors of each point. This can be done in two ways: (1) keep all point within some fixed radius $\rho$ or (2) compute $\kappa$ nearest neighbors.
	\item Construct a weighting matrix $\mat{W}$, where $(\mat{W})_{ij}$ = 0 if points $i$ and $j$ are not neighbors, and $(\mat{W})_{ij}$ is inversely proportional to the distance between points $i$ and $j$ otherwise.
	\item Compute an embedding from $\mat{W}$ that is locally isometric to $\mat{X}$.
\end{enumerate}
For the third step, many different and powerful approaches have been proposed, from computing shortest paths on a graph~\citep{Tenenbaum2000}, to using graph spectral methods \citep{Belkin2003}, to using neural networks \citep{Hadsell2006}.

However, the success of these techniques depends critically on the ability to capture the data structure in the first two steps.
Correctly setting either $\rho$ or $\kappa$ is a non-trivial task that is left to the user of these techniques.
Furthermore, a kernel (most commonly an RBF) is often involved in the second step, adding an additional parameter (the kernel width/scale) to the user to determine.
Expectedly, the optimal selection of these parameters plays a critical role in the overall success of the manifold learning process.

NOMAD departs drastically from this setup as no kernel selection nor nearest neighbor search are involved. Yet, the solution $\mat{Q}_*$ is effectively a kernel which is automatically learned from the data. Because $\mat{Q}_*$ is positive semidefinite it can be factorized as $\mat{Q}_* = \transpose{\mat{Y}} \mat{Y}$, defining a feature map from $\mat{X}$ to $\mat{Y}$.

To illustrate intuitively the differences and similarities with prior work on manifold learning we use LLE \citep{Roweis2000} as an example.
LLE optimizes the cost function
\begin{equation}
	\Phi(\mat{Y}) = \traceone{\transpose{\left( \mat{I} - \mat{W} \right)} \left( \mat{I} - \mat{W} \right) \left( \transpose{\mat{Y}} \mat{Y} \right)},
\end{equation}
where $\mat{W}$ is the adjacency matrix of a weighted nearest-neighbors graph. The key to finding a matrix $\mat{Y}$ that is locally isometric to $\mat{X}$, while unwrapping the data manifold, is to remove from $\mat{W}$ the connections between distant points $(\mat{X})_{:i}$ and  $(\mat{X})_{:j}$. This is done with some technique to find nearest neighbors. 

NOMAD also tries to align the output Gramian, $\mat{Q}$, to the input Gramian, $\mat{D}$, but discards distant data points differently. As negative entries in $\mat{D}$ cannot be matched because $\mat{Q}$ is nonnegative, the best option would be to set the corresponding element of $\mat{Q}$ to zero. This effectively discards pairs of input data points whose inner product is negative thus enforcing locality in the angular space \citep{Cho2009}, see \cref{fig:distances2gramian}. In fact, this argument can be taken further by noting that the constraint $\mat{Q} \vect{1} = \vect{1}$ allows us to replace the Gramian $\mat{D}$ with the negative squared distance matrix,
\begin{equation}
	-\tfrac{1}{2} \sum_{ij} \norm{\vect{x}_{i} - \vect{x}_{j}}{2}^2 (\mat{Q})_{ij} = -\sum_{i} (\mat{D})_{ii} \sum_{j} (\mat{Q})_{ij} + \traceone{\mat{D} \mat{Q}} = -\traceone{\mat{D}} + \traceone{\mat{D} \mat{Q}} .
\end{equation}

Finally, the constraint $\traceone{\mat{Q}} = K$ allows further control of the neighborhood size of NOMAD (modulating the actual width of its kernel function, see \cref{fig:distances2gramian}).
Next, we develop further intuition about the manifold-learning capabilities of NOMAD by analyzing theoretically the dataset in \cref{fig:distances2gramian}.

\begin{figure}[t]
	\centering
	\includegraphics[width=.85\textwidth]{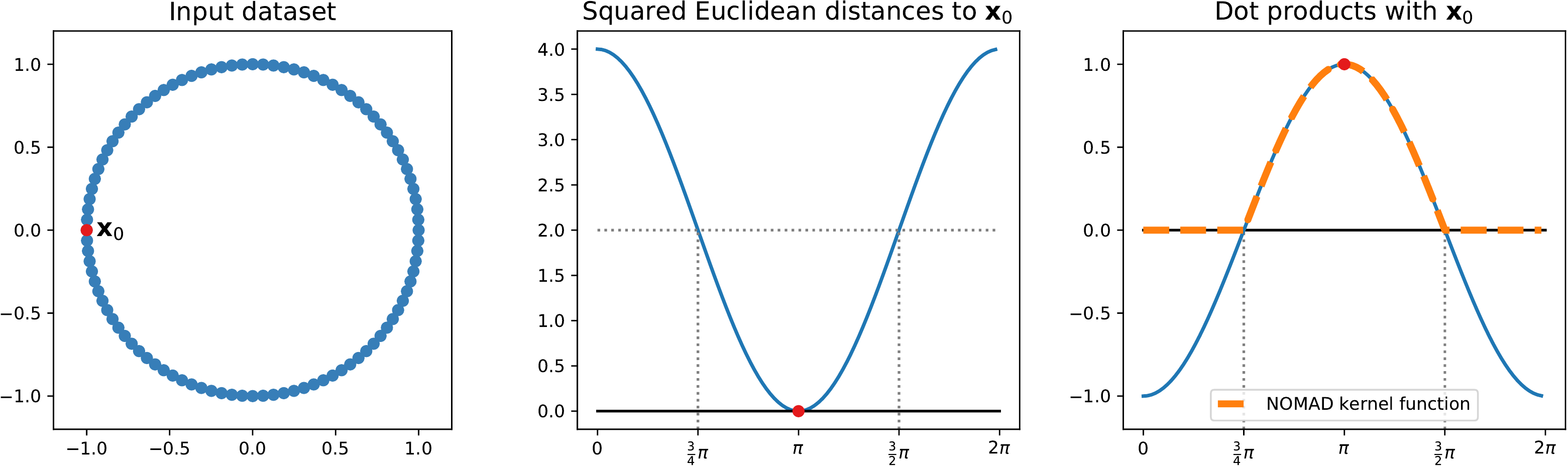}
	
	\caption{Correspondence between using a kernel/threshold in distance-space and nonnegativity of Gramian-based representations. In this toy example, the constraint $\mat{Q} \geq 0$ in NOMAD is equivalent to setting to zero distances that are greater than $\smash{\sqrt{2}}$ (squared distances greater than 2). We use $\vect{x}_0$ as a reference but rotational symmetry makes this argument valid for all points in the dataset.}
	\label{fig:distances2gramian}
\end{figure}

As we mentioned before, the SDP formulation of \citet{Peng2007_sdk-kmeans} was developed as a clustering algorithm. Whether this method actually delivers a clustered solution depends on the geometry of the dataset. When the dataset consists of well-segregated  clusters, the resulting $\mat{Q}_*$ has block diagonal structure. We empirically observe that, when the dataset is sampled from a regular manifold, the solution $\mat{Q}_*$ does not break down the dataset into artificial clusters and actually preserves the manifold structure (see \cref{sec:manifold}). In a simple example, where the manifold exhibits a high degree of symmetry, we demonstrate analytically that this behavior occurs. The following sections are devoted to this task.

\subsection{Analysis of NOMAD on a 2D ring dataset}

We analyze the case in which  the input data to NOMAD possess rotational symmetry, i.e., data are arranged uniformly on a ring, see \cref{fig:distances2gramian}.
In this case, we can write the SDP as a linear program (LP) in the circular Fourier basis. This new representation allows to visualize that NOMAD lifts the data into a high-dimensional space, with $K$ controlling its dimensionality.

In the example in \cref{fig:distances2gramian}, the entries of $\mat{D}$ can be described by
$(\mat{D})_{ij} = \transpose{\vect{x}_i} \vect{x}_j = \cos(\alpha_i - \alpha_j)$,
where $\alpha_i, \alpha_j$ are the angles of points $\vect{x}_i, \vect{x}_j$, respectively (\cref{fig:distances2gramian}). Since the points are uniformly distributed over the ring, $\mat{D}$ is a circulant matrix, i.e.,  $\cos(\alpha_i - \alpha_j) = \cos(\alpha_{i+k} - \alpha_{k+k})$.
The solution $\mat{Q}_*$ to NOMAD is circulant too \citep{Bachoc2012}.
Being circulant matrices, $\mat{D}$ and $\mat{Q}_*$ are diagonalized by the discrete Fourier transform (DFT), i.e.,
\begin{align}
	\mat{D} = \mat{F} \diag{\vect{d}} \conjugatetranspose{\mat{F}} ,
	&&
	\mat{Q}_* = \mat{F} \diag{\vect{q}} \conjugatetranspose{\mat{F}} ,	
	\label{eq:fourier_expansion}
\end{align}
where $\vect{q} \geq \vect{0}, \vect{d} \geq \vect{0}$ respectively are vectors containing the eigenvalues of $\mat{D}$ and $\mat{Q}_*$, $\conjugatetranspose{\mat{F}}$ is a Hermitian conjugate of $\mat{F}$, and $\mat{F} \in \Complex^{n \times n}$ is the unitary DFT matrix, with entries ($p,k = 0, \dots, n-1$)
\begin{equation}
	(\mat{F})_{pk} = \tfrac{1}{\sqrt{n}} \exp \left( -i 2\pi p \tfrac{k}{n} \right) .
	\label{eq:fourier_matrix}
\end{equation}
Hence, and in accord with the constraint $\mat{Q}_* \vect{1} = \vect{1}$, we have that $(\mat{F})_{0:} = \tfrac{1}{\sqrt{n}} \vect{1}$ and $(\vect{q})_0 = 1$.



\subsection{A linear program on the data manifold}

We express the objective function and the constraints of NOMAD
in terms of $\vect{d}$ and $\vect{q}$, i.e.,
\begin{align}
	\traceone{\mat{D} \mat{Q}_*} &= \transpose{\vect{d}} \vect{q} , \\
	\traceone{\mat{Q}} &= \transpose{\vect{1}} \vect{q} = K , \\
	(\mat {Q})_{kk'}
	&= (\mat{F})_{k:} \diag{\vect{q}} (\conjugatetranspose{\mat{F}})_{:k'}
	= \sum_{p=0}^{n-1} \tfrac{(\vect{q})_p}{n} \cos \left( 2\pi p \tfrac{k' - k}{n} \right) \geq 0 .
	\label{eq:lp_nneg}
\end{align}
This reformulation allows us to rewrite NOMAD as a linear program
\begin{equation}
	\max_{\vect{q}}
	\transpose{\vect{d}} \vect{q}
	\quad\text{s.t.}\quad
	(\forall \tau)\ \transpose{\vect{c}_\tau} \vect{q} \geq 0 ,
	\quad \transpose{\vect{1}} \vect{q} = K ,
	\quad \vect{q} \geq \vect{0},
	\quad (\vect{q})_0 = 1 ,
	\label[problem]{eq:lp_kmeans}
\end{equation}
where $(\vect{c}_\tau)_p = \tfrac{1}{n} \cos \left( 2 \pi p \tfrac{\tau}{n} \right)$.

\cref{eq:lp_kmeans} sheds light on the inner workings of NOMAD. First, the constraint $\transpose{\vect{1}} \vect{q} = K$ ensures that $\vect{q}$ does not grow to infinity and acts as a budget constraint. Let us assume for a moment that we remove the constraint $\transpose{\vect{c}_\tau} \vect{q} \geq 0$ (the equivalent of $\mat {Q} \geq \mat{0}$). Then, the program will try to set to $K$ the entry of $\vect{q}$ corresponding to the largest eigenvalue of $\vect{d}$; this $\vect{q}$ will violate as $K$ gets bigger the removed constraint (since $(\vect{c}_\tau)_p$ is a sinusoid). Then the effect of this constraint is to spread the allocated budget among several eigenvalues (instead of just the largest). The experiment in \cref{fig:ring_k-evolution} confirms this: the number of active eigenvalues of $\mat{Q}_*$ grows with $K$. We can interpret this as increasing the intrinsic dimensionality of the problem in such a way that only local interactions are considered.

\noindent\textbf{Interpretation of $K$.}
The circulant property of $\mat{Q}_*$ for the 2D ring sheds further light on the meaning of $K$. In \cref{fig:ring_k-evolution_diags}, we observe that the number of significant elements in each of $\mat{Q}_*$ is $\lceil n / K \rceil$. Thus, we can interpret $K$ as a parameter that effectively sets the size of the local neighborhood on the manifold. In standard manifold learning methods this size is set by a combination of the number of nearest neighbors and the shape and scale of the kernel function. In NOMAD, all these variables are incorporated into a single parameter and balanced with the help of the remaining problem constraints.

In general, for non-symmetric and irregularly sampled manifolds, $K$ is chosen to capture the manifold underlying the dataset: the neigborhood size needs to be small enough to capture the desired manifold features, but big enough to avoid capturing unwanted structure (e.g., noise). If sampling density differs in different areas, the size will adjust locally as needed.

\begin{figure}[t]
	\centerline{%
	\hfill
	\begin{subfigure}{.18\linewidth}
		\begin{tabu} to \linewidth {@{\hspace{0pt}} X[c,m] @{\hspace{2pt}} X[c,m] @{\hspace{0pt}}}
			\begin{overpic}[width=\linewidth]{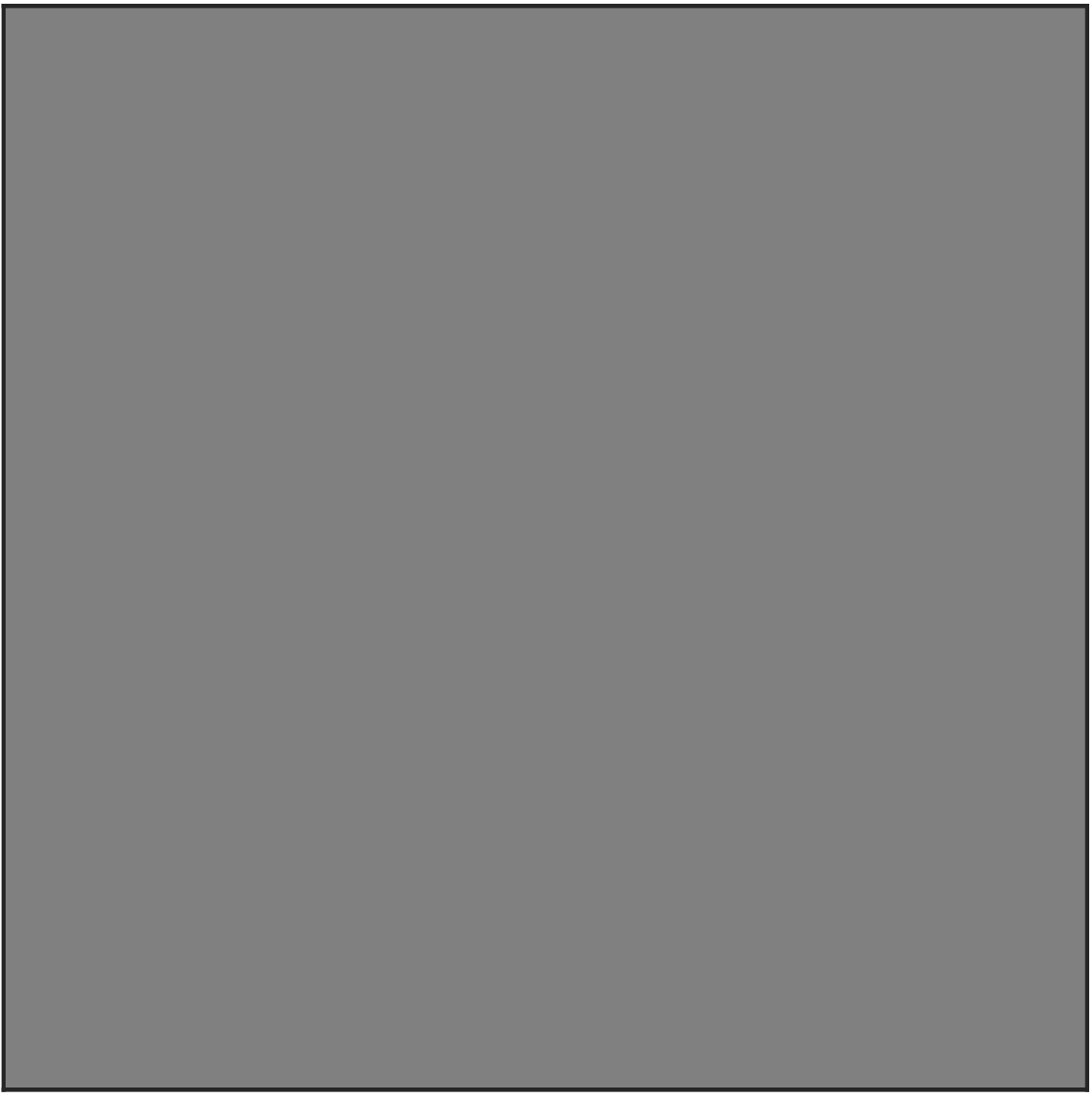}
				\put (5,5) {\scriptsize$K=1$}
			\end{overpic} &
			\begin{overpic}[width=\linewidth]{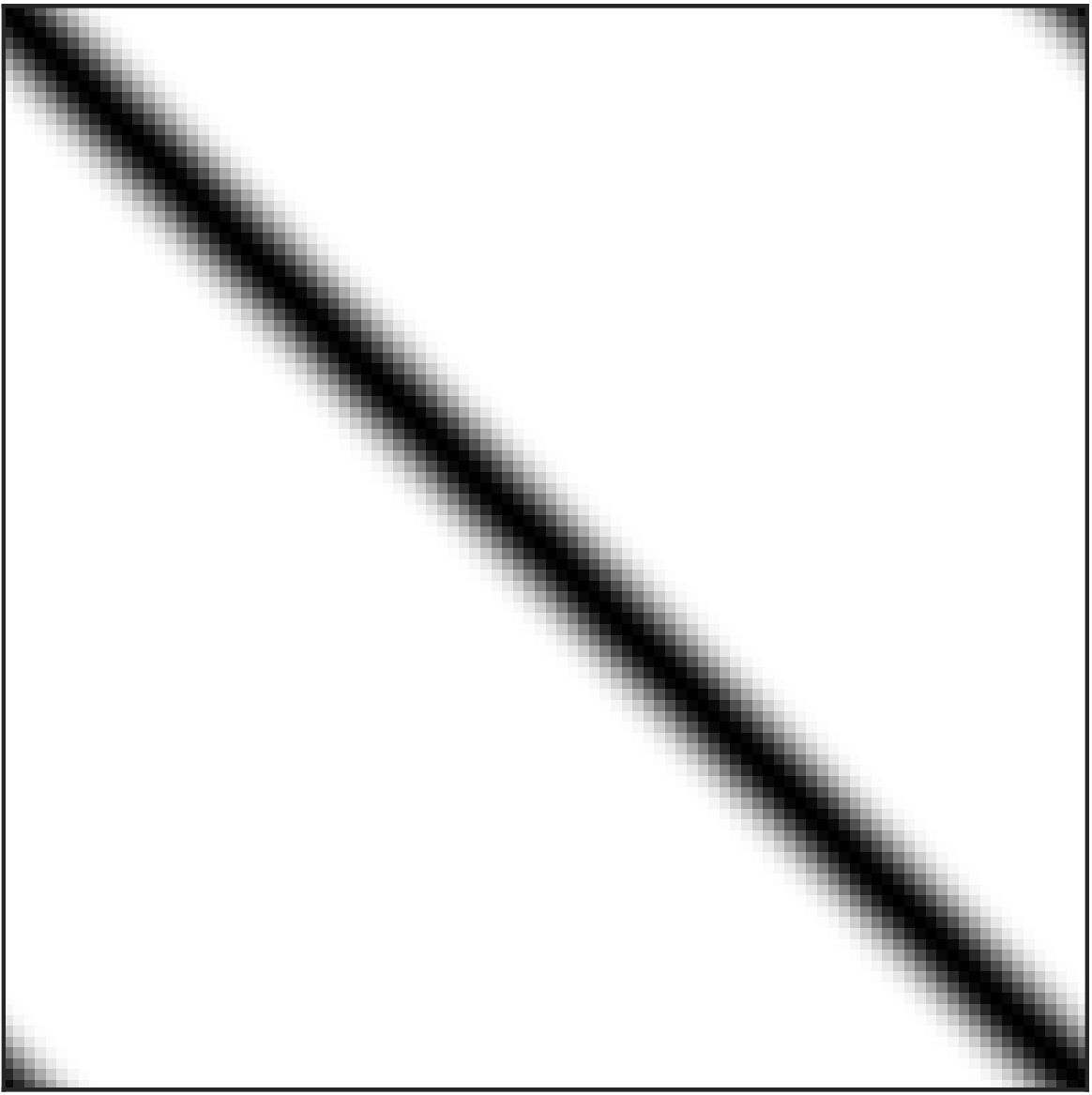}
				\put (5,5) {\scriptsize$K=12$}
			\end{overpic} \\
			\\[-9pt]
			\begin{overpic}[width=\linewidth]{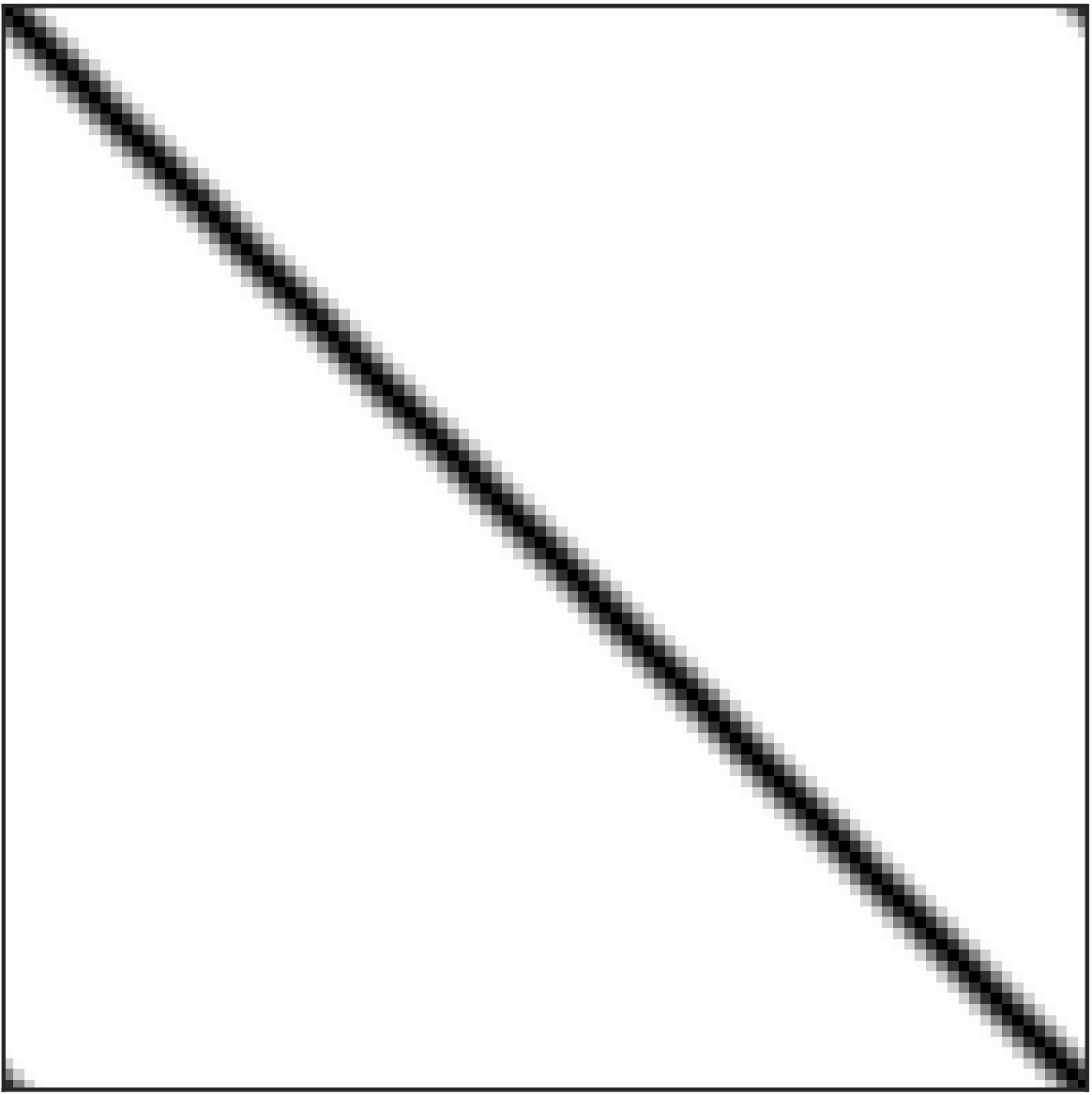}
				\put (5,5) {\scriptsize$K=25$}
			\end{overpic} &
			\begin{overpic}[width=\linewidth]{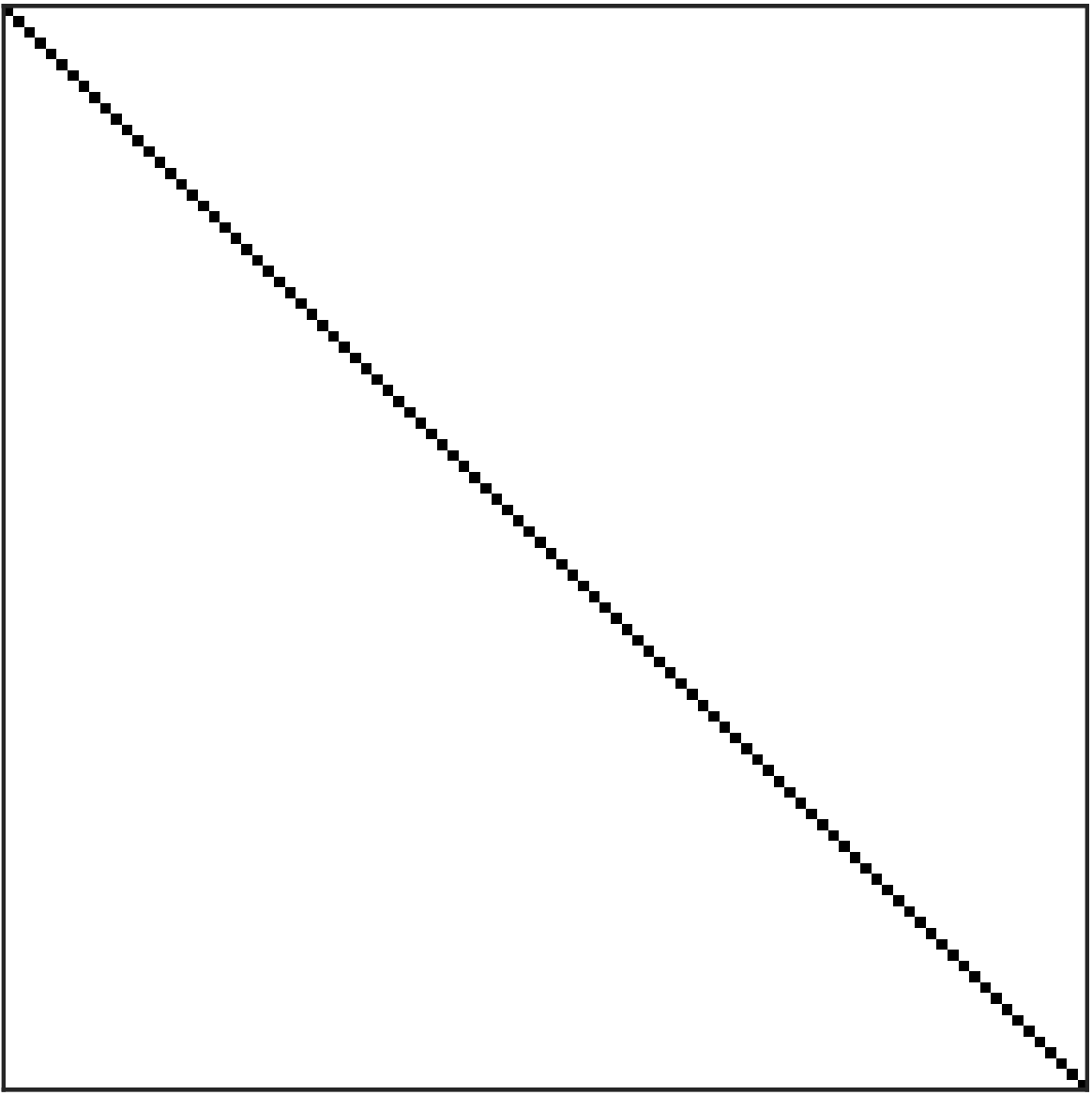}
				\put (5,5) {\scriptsize$K=100$}
			\end{overpic}
		\end{tabu}
		\caption{}
		\label{fig:ring_k-evolution_solutions}
	\end{subfigure}%
	\hfill
	\begin{subfigure}{.26\linewidth}
		\includegraphics[width=\linewidth]{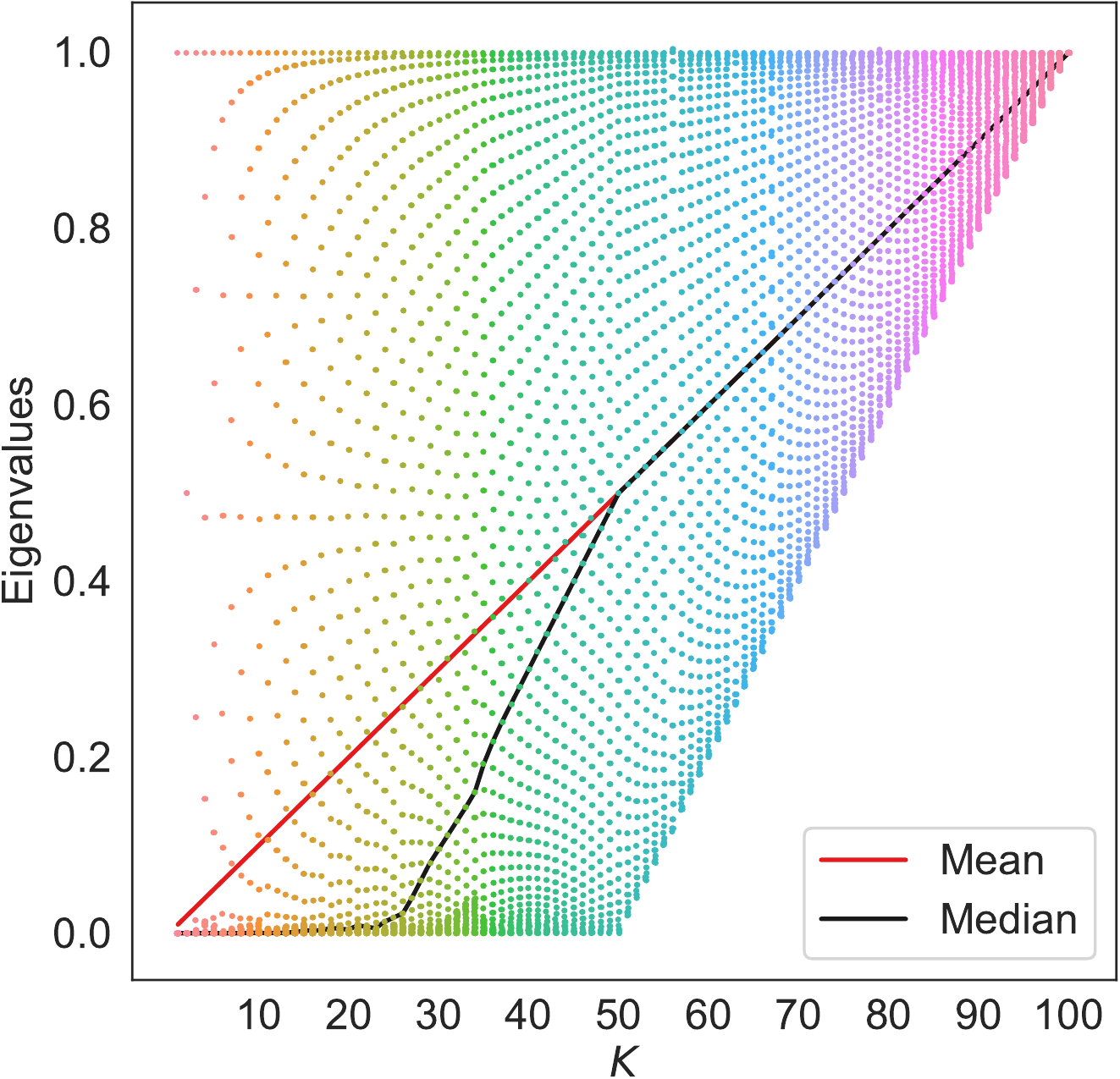}
		\caption{}
		\label{fig:ring_k-evolution_eigs}
	\end{subfigure}%
	\hfill
	\begin{subfigure}{.34\linewidth}
		\includegraphics[width=\linewidth]{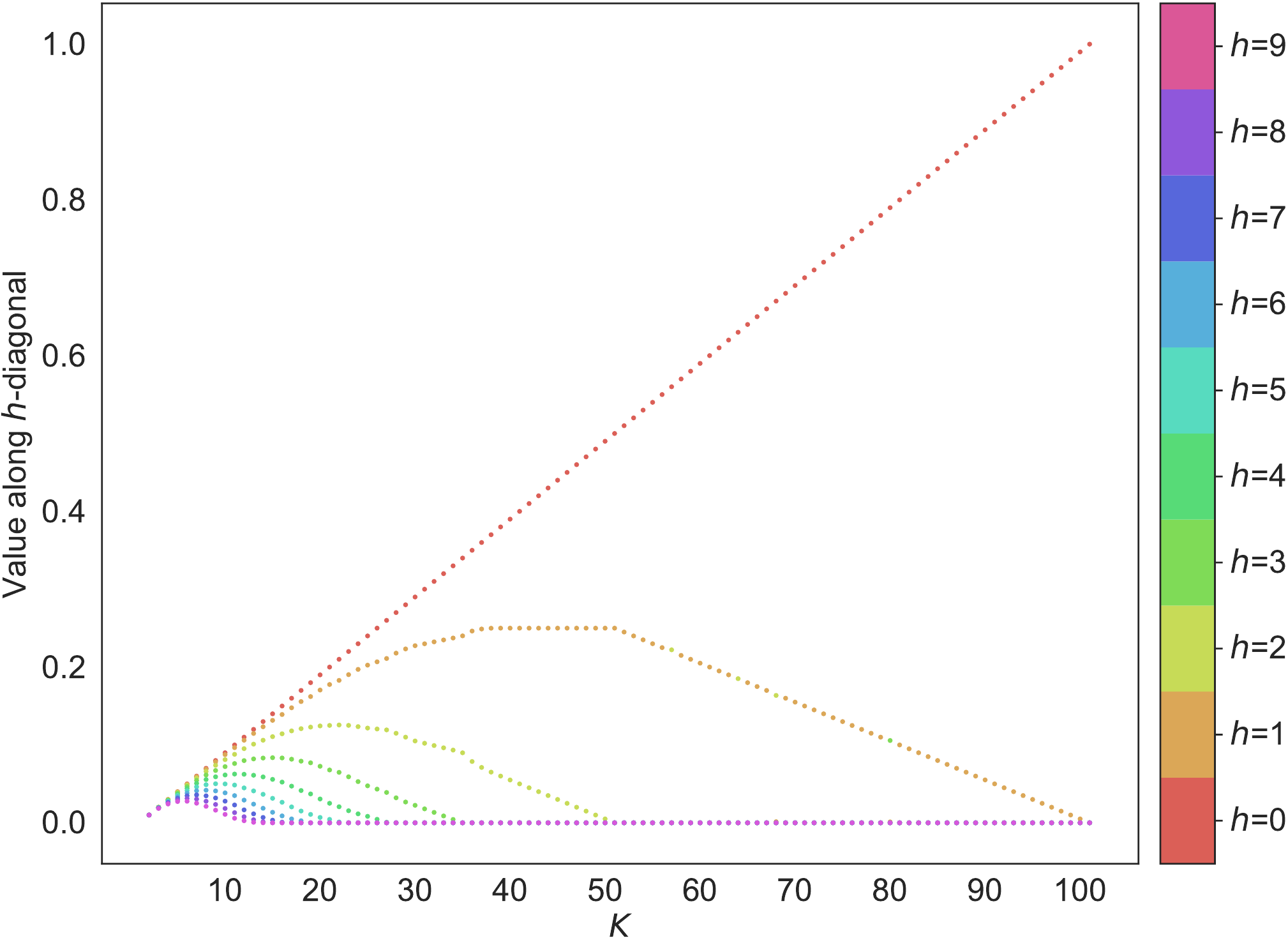}
		\caption{}
		\label{fig:ring_k-evolution_diags}
	\end{subfigure}%
	\hfill
	\begin{subfigure}{.18\linewidth}
		\centering
		\includegraphics[width=\linewidth]{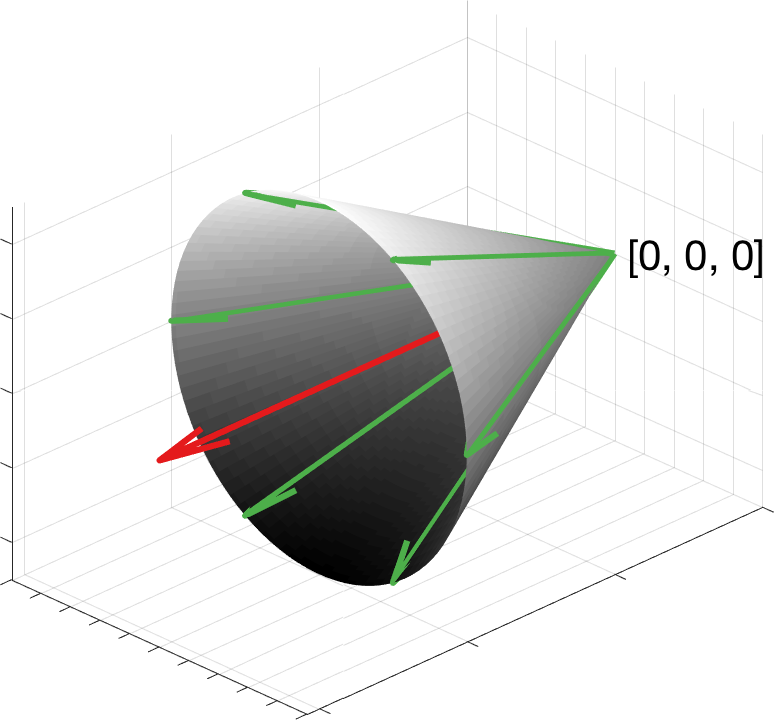}
		\caption{}
		\label{fig:cone_structure}
	\end{subfigure}	
	\hfill
	}

	\caption{Evolution of the NOMAD solution for the 2D ring dataset (with 100 points, see \cref{fig:distances2gramian}) with increasing parameter $K$. 
		\protect\subref{fig:ring_k-evolution_solutions} As $K$ increases, the solution, $\mat{Q}_*$, concentrates more and more towards the diagonal.
		\protect\subref{fig:ring_k-evolution_eigs} As $K$ increases, the number of active eigenvalues in the solution, $\mat{Q}_*$, grows resulting in the more uniform distribution of eigenvalues and greater mean/median (notice that the mean being linear comes from the trace constraint).
		\protect\subref{fig:ring_k-evolution_diags} 
		We define the $h$-diagonal of $\mat{Q}_*$ as the entries $(i, j)$ for which $i - j = h$.
		As $\mat{Q}_*$ is a circulant matrix, each $h$-diagonal contains a single repeated value. We plot these values, assigning a different color to each $h$. The effect of the scaling constraint $\traceone{\mat{Q}} = K$ becomes evident: when one $h$-diagonal becomes inactive, all remaining $h'$-diagonals need to be upscaled.
		\protect\subref{fig:cone_structure} The eigenvectors of $\mat{Q}_*$ form a high-dimensional cone (cartoon representation, cone axis in red and eigenvectors in green).
		}
	\label{fig:ring_k-evolution}
\end{figure}

\subsection{Lifting the ring to a high-dimensional cone}

Here, we show that NOMAD effectively embeds the data manifold into a space where its structure, i.e., rotational symmetry, is preserved.
We now make use of the half-wave symmetry in $\mat{Q}_*$, noting that they can be fully represented with only one half of the Fourier basis. We can then decompose it with the real Fourier basis
\begin{equation}
	\mat{Q}_* = \mat{\tilde{F}} \diag{\vect{\tilde{q}}} \transpose{\mat{\tilde{F}}} ,
	\label{eq:fourier_expansion_real}
\end{equation}
where
$\vect{\tilde{q}} = \transpose{ [ (\vect{q})_0, (\vect{q})_1, (\vect{q})_1, \dots, (\vect{q})_{n-1}, (\vect{q})_{n-1} ] }$
and
$\mat{\tilde{F}} \in \Real^{n \times n}$ has entries ($p,k = 0, \dots, n-1$)
\begin{equation}
	(\mat{\tilde{F}})_{pk} =
	\begin{cases}
		\tfrac{2}{\sqrt{n}} \cos \left( 2\pi p \tfrac{k}{n} \right) & \text{if $k$ is even,} \\
		\tfrac{2}{\sqrt{n}} \sin \left( 2\pi p \tfrac{k-1}{n} \right) & \text{if $k$ is odd.} \\
	\end{cases}
	\label{eq:fourier_matrix_real}
\end{equation}
Let $\mat{\tilde{Y}} = \diag{\vect{q}}^{1/2} \transpose{\mat{\tilde{F}}}$.
Notice that
$
	\langle \mat{\tilde{Y}}_{:i} / \norm{\mat{\tilde{Y}}_{:i}}{F}, \mat{\tilde{F}}_{:0} \rangle = \langle\mat{\tilde{F}}_{:i}, \mat{\tilde{F}}_{:0} \rangle = \tfrac{4}{n}
$,
meaning that the vectors $\mat{\tilde{Y}}_{:i}$ are the extreme rays of a right circular cone with the eigenvector $\mat{\tilde{F}}_{:0} = \tfrac{2}{\sqrt{n}} \transpose{[ 1, 0, \dots, 0 ]}$ as its symmetry axis, see \cref{fig:cone_structure}. Thus, we can interpret the solution to NOMAD as lifting the 2D ring structure into a cone. As mentioned before, this cone is high-dimensional, with as many directions as needed to preserve the nonnegativity of $\mat{Q}$.

We identify the rank of the solution $\mat{Q}$ with the number of active eigenvalues. 
The bigger the $K$, the higher the rank. The constraint $\mat{Q} \vect{1} = \vect{1}$ in NOMAD leads to a fanning-out effect in the data representation.
Intuitively, this fan-out effect is key to the disentanglement of datasets with complex topologies. Spin-model-inspired SDPs for community detection~\citep{Javanmard2016} achieve a similar fanning-out by dropping the constraint $\mat{Q} \vect{1} = \vect{1}$ and adding the related term $-\gamma \transpose{\vect{1}} \mat{Q} \vect{1}$ to the objective function.

With the LP framework and the geometric picture in place, we can begin to understand how the solution evolves as the parameter $K$ increases from $1$ to $n$. At $K=1$, only the eigenvalue $(\vect{q})_0$ is active and every vector $\smash{(\mat{\tilde{Y}})_{:i}}$ is the same with each entry equal to $1/n$. When $K$ slightly above 1, the eigenvalue $(\vect{q})_1$ becomes active (nonzero), introducing the first  nontrivial Fourier component. Geometrically, the vectors $\smash{\{ (\mat{\tilde{Y}})_{:i} \}}$ now open up into a narrow cone. As $K$ increases, the cone widens and, at some point, the angle between two of the  vectors reaches $\pi/2$ (this activates the nonnegativity constraint in \cref{eq:lp_nneg}). Further increase of $K$ necessitates use of a larger number of Fourier modes. Finally, at $K=n$ all modes are active and all vectors $\smash{\{ (\mat{\tilde{Y}})_{:i} \}}$ become orthogonal to each other. \cref{fig:ring_k-evolution_eigs} depicts the progression with $K$ of the number of active modes.

\noindent\textbf{Summary.}
Previous studies \citep{Kulis2007,Peng2007_sdk-kmeans,Awasthi2015}, focus solely on cases where NOMAD exhibits $K$-means-like solutions (i.e., hard-clustering). \cref{sec:theory} provides a characterization of the NOMAD solutions on a simple example with a high degree of symmetry, showing that they are drastically different from $K$-means. These solutions connect neighboring points, with the neighborhood size determined by $K$. These neighborhoods overlap, as they would in soft-clustering, in a way that preserves global features of the manifold, including its symmetry. This is a feature sought after by manifold learning methods and help place NOMAD among reliable manifold analysis techniques.

\section{Analyzing data manifolds with NOMAD: Experimental results}
\label{sec:manifold}

In the previous section, we showed that NOMAD recovers the data manifold in an idealized 2D ring dataset. Here, we extend this observation numerically to more complex datasets for which analytical form of the transformation that diagonalizes $\mat{Q}_*$ (nor $\mat{D}$) is not known, see \cref{fig:opening,fig:moons,fig:joint}. 
We visualize the solution $\mat{Q}_*$ by embedding it in a low-dimensional space. While our goal is not dimensionality reduction, we learn the data manifold with NOMAD, and use standard spectral dimensionality reduction to visualize the results.

%
%

\begin{figure}[t]
	\centering
    
	\begin{subfigure}{.1203\linewidth}
		\includegraphics[width=\linewidth]{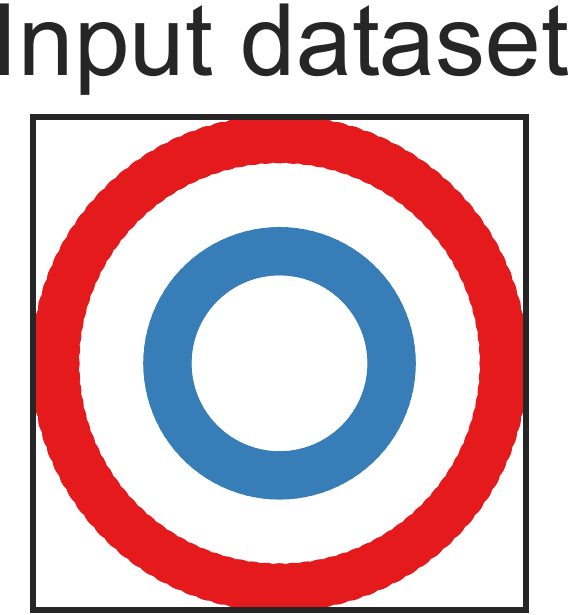}
		\caption{}
		\label{fig:2circles}
	\end{subfigure}%
	\hspace{2em}
    \begin{minipage}{.4\linewidth}
	\begin{subfigure}{.643\linewidth}
		\includegraphics[width=\linewidth]{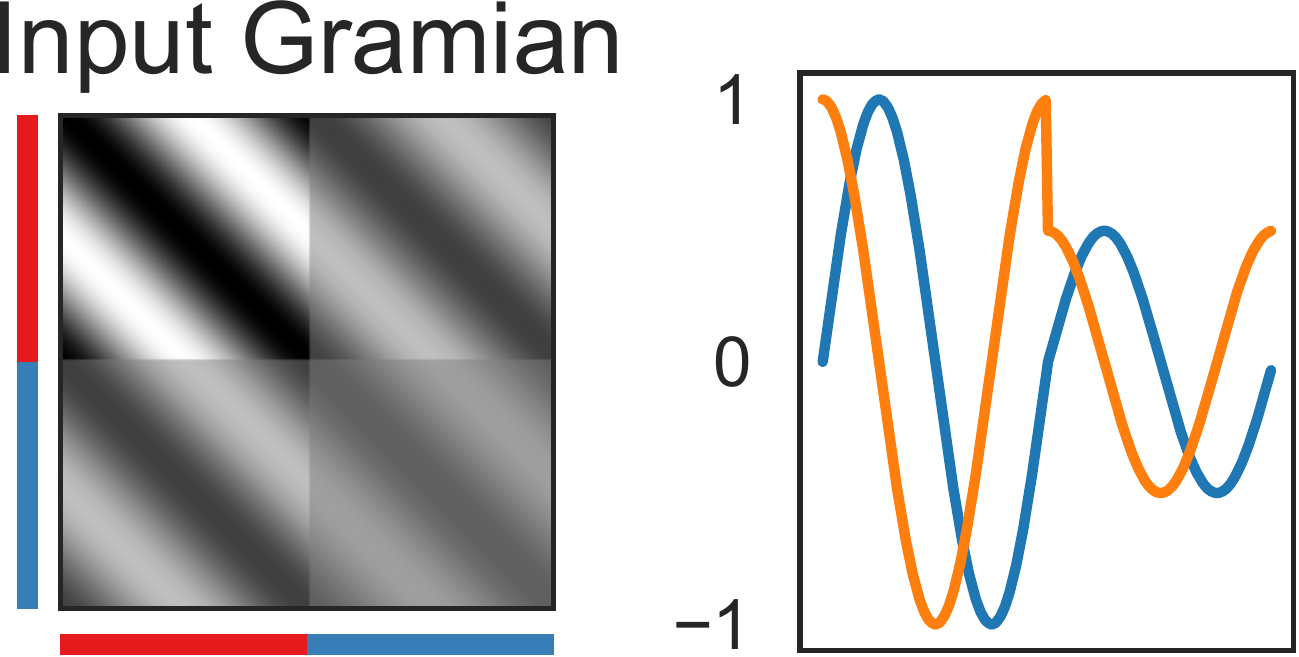}
		\caption{}
		\label{fig:2circlesGramian}
	\end{subfigure}%
	\\
	\begin{subfigure}{\linewidth}
		\includegraphics[width=\linewidth]{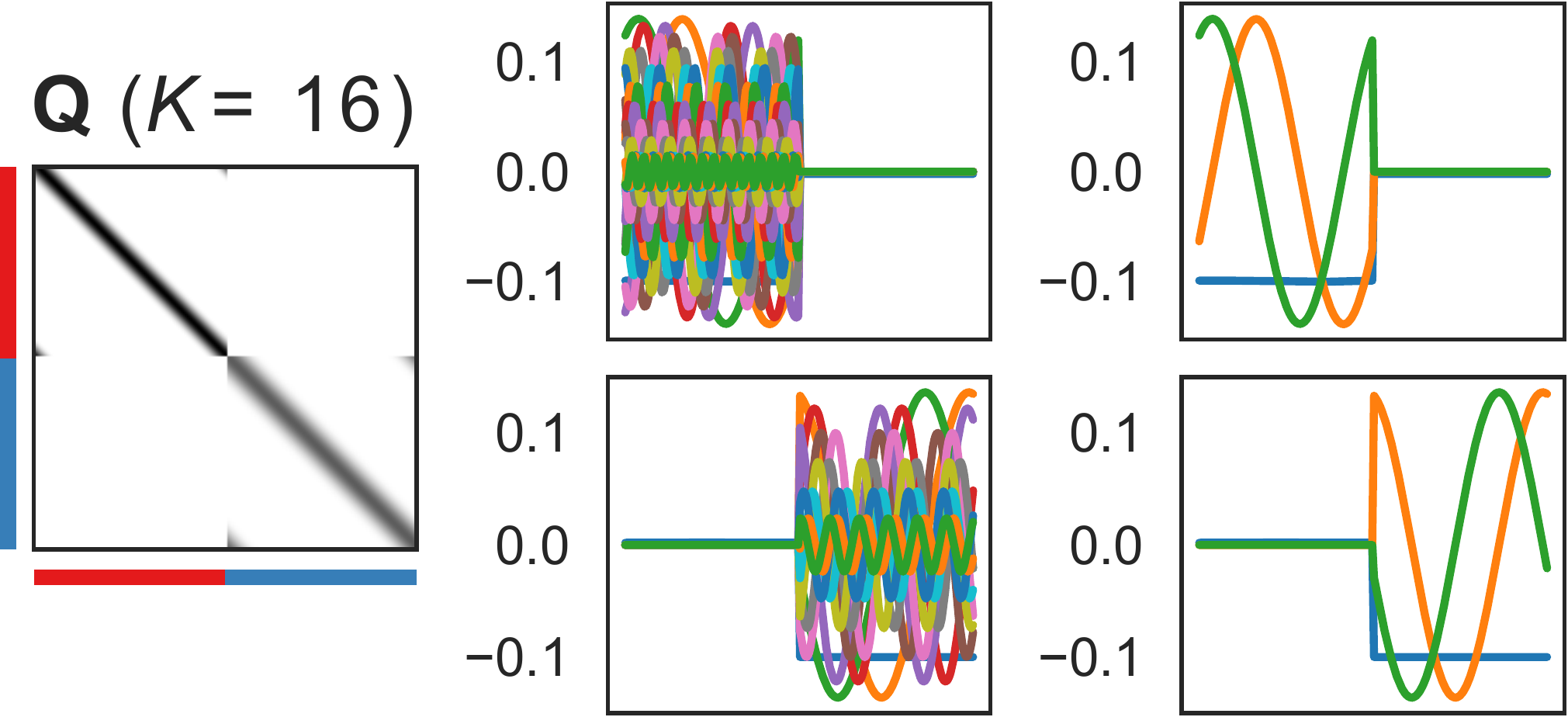}
		\caption{}
		\label{fig:circles_eigendecomposition3}
	\end{subfigure}%
    \end{minipage}
	\hspace{2em}
	\begin{subfigure}{.25\linewidth}
		\centering
		\includegraphics[width=\linewidth]{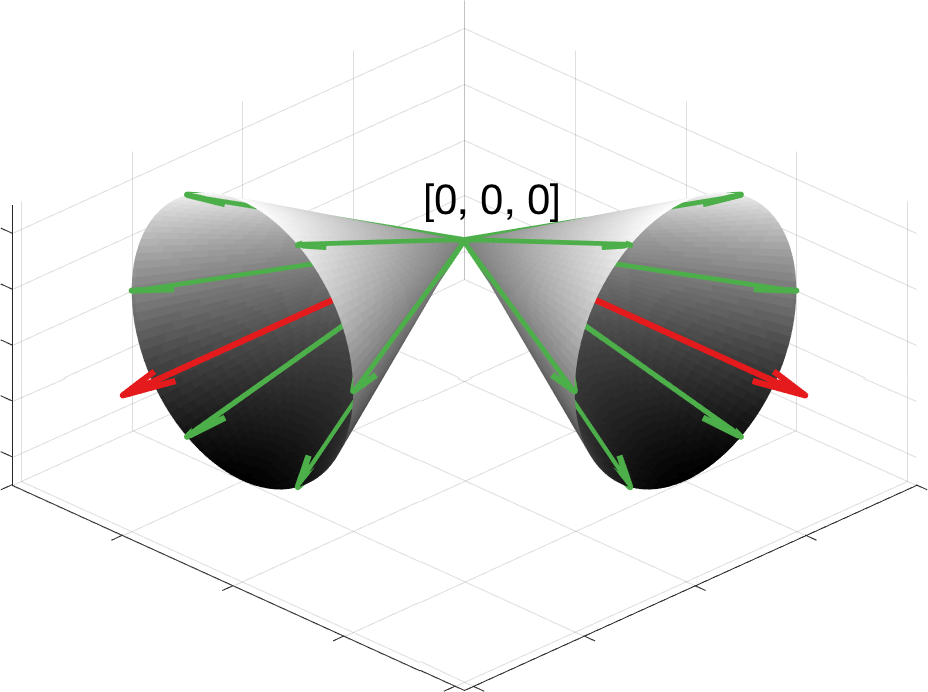}
		\caption{}
		\label{fig:two_cones_structure}
	\end{subfigure}
	
	\caption{Solution of NOMAD on the dataset consisting of two 2D rings.
    \protect\subref{fig:2circles} Two-ring dataset.
    \protect\subref{fig:2circlesGramian} Input Gramian, $\mat{D}$, and its two eigenvectors ($\mat{D}$ has rank 2). Note that the eigenvectors of  $\mat{D}$ do not segregate the rings.
	\protect\subref{fig:circles_eigendecomposition3} The solution, $\mat{Q}$, of NOMAD contains two sets of eigenvectors with disjoint support: one set describing the points in each ring (we show all eigenvectors and a detail on the first 3 within each set).
	\protect\subref{fig:two_cones_structure} The eigenvectors of $\mat{Q}$ form two orthogonal high-dimensional cones: one cone for each ring (cartoon representation, cone axis in red and eigenvectors in green). Notice how these cones become linearly-separable.}
	\label{fig:circles_eigendecomposition}
\end{figure}

\begin{figure}[p]
	\centering
	\includegraphics[width=.45\linewidth]{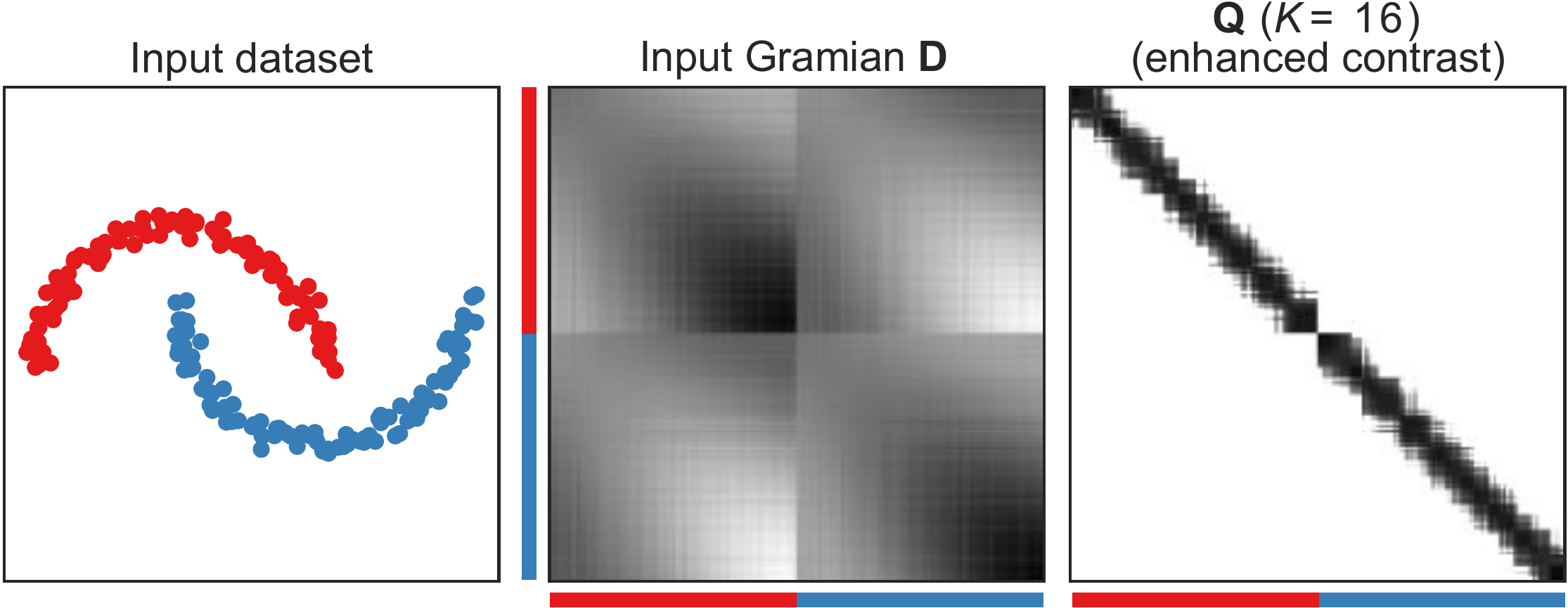}
	
	\caption{Learning multiple manifolds with NOMAD. The points are arranged in two semicircular manifolds and contaminated with Gaussian noise. Although the manifolds are linearly non-separable, NOMAD correctly finds two submatrices, one for each manifold (for visual clarity, we enhance the contrast of $\mat{Q}$).}
	\label{fig:moons}
\end{figure}

\begin{figure*}[p]
	\centering
	\begin{subfigure}[t]{.252\textwidth}
		\centering
		\includegraphics[width=\linewidth]{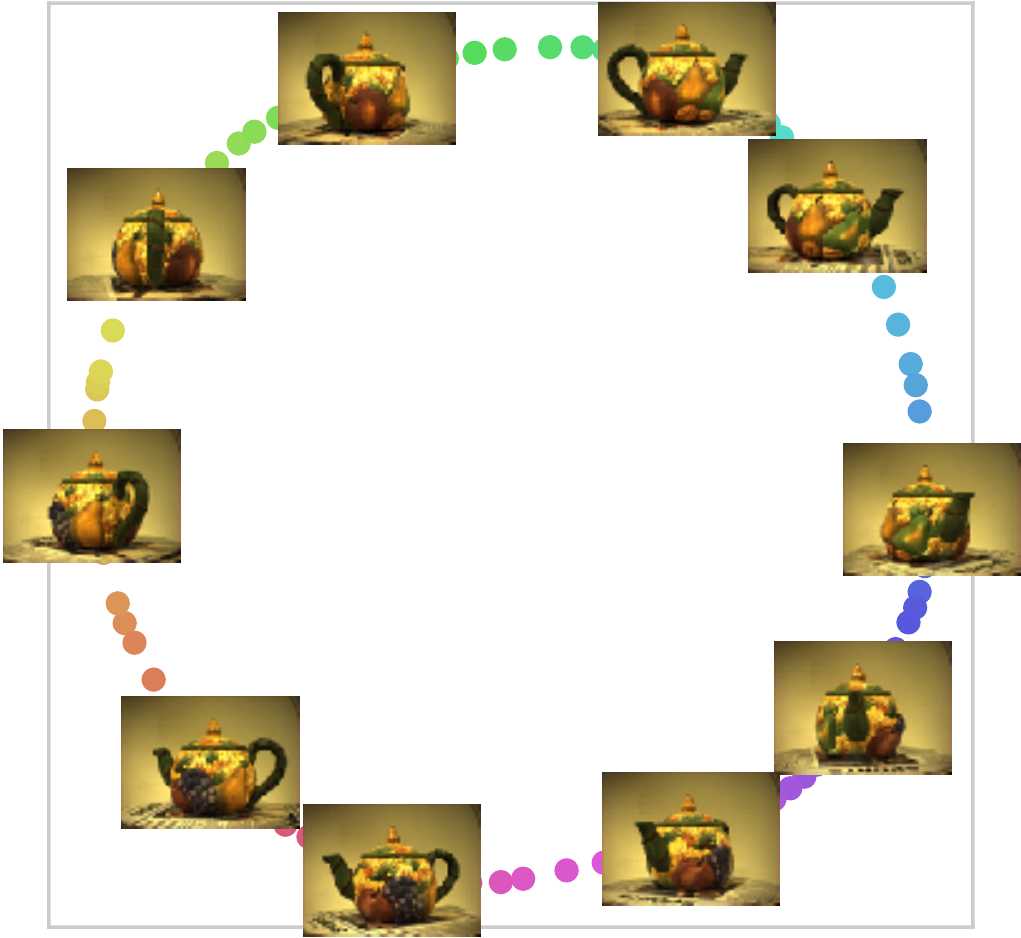}
		
		\caption{Teapots ($K = 20$)}
		\label{fig:embedding_real_teapot}
	\end{subfigure}%
	\hfill%
	\begin{subfigure}[t]{.24\textwidth}
		\centering
		\includegraphics[width=\linewidth]{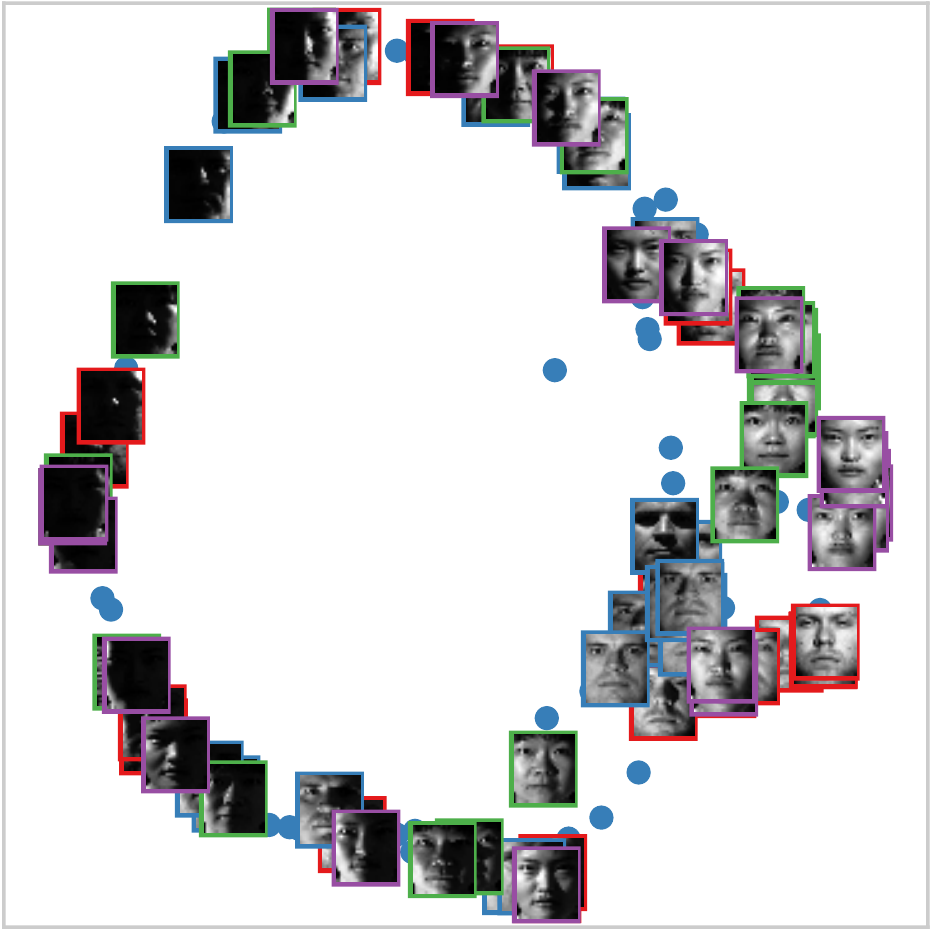}
		
		\caption{Yale Faces ($K = 16$)}
		\label{fig:embedding_real_faces}
	\end{subfigure}%
	\hfill%
	\begin{subfigure}[t]{.48\textwidth}
		\centering
		\includegraphics[width=.5\linewidth]{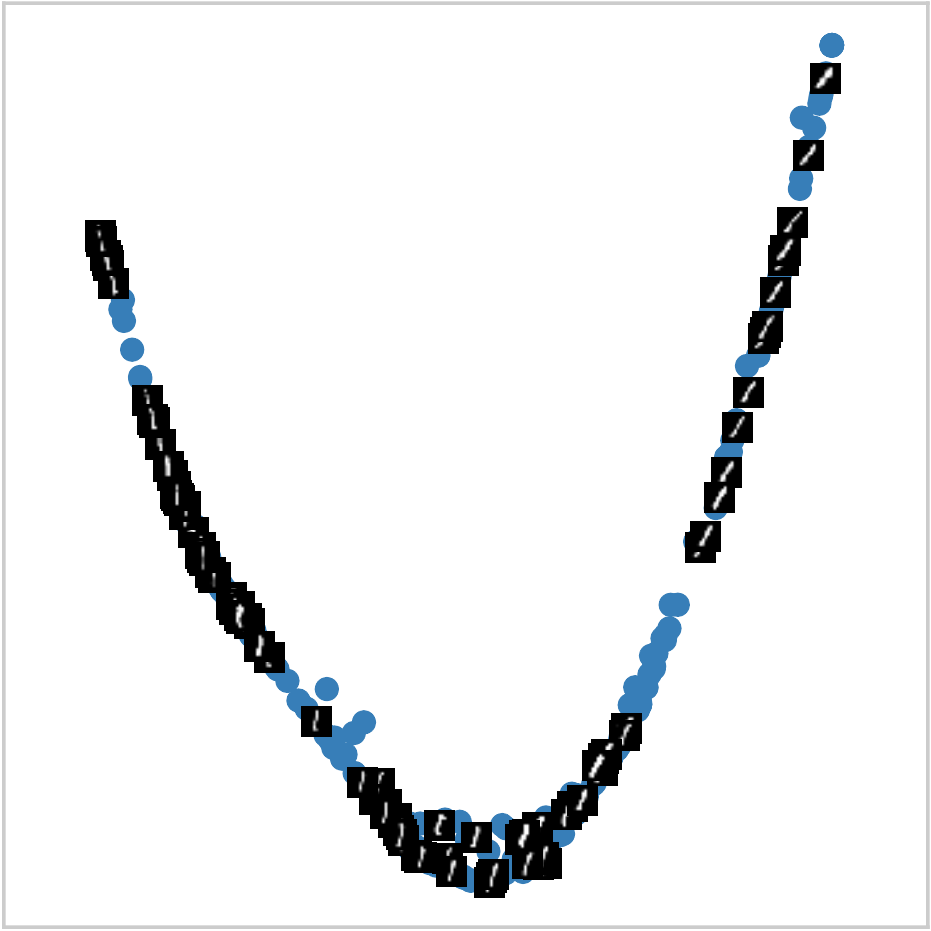}%
		\includegraphics[width=.5\linewidth]{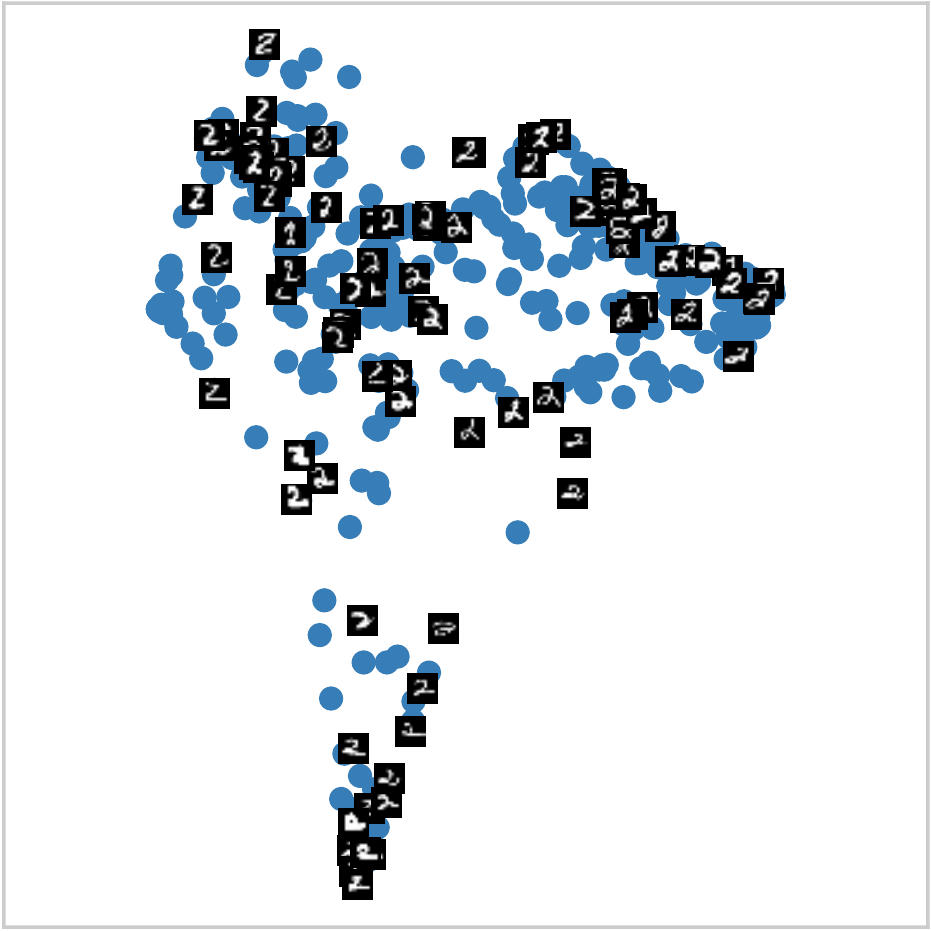}
		
		\caption{MNIST digits  ($K = 16$)}
		\label{fig:embedding_real_mnist}
	\end{subfigure}
	
	\caption{Finding two-dimensional embeddings with NOMAD.
		\protect\subref{fig:embedding_real_teapot} 100 images obtained by viewing a teapot from different angles in a plane. The input vectors size is 23028 ($76 \times 101$ pixels, 3 color channels). The manifold uncovers the change in orientation.
		\protect\subref{fig:embedding_real_faces} 256 images from 4 different subjects (each subject is marked with a different color in the figure), obtained by changing the position of the illumination source.  The input vectors size is 32256 ($192 \times 168$ pixels).  The manifold uncovers the change in illumination (from frontal, to half-illuminated, to dark faces, and back).
		\protect\subref{fig:embedding_real_mnist} 500 images handwritten instances of the same digit.  The input vectors size is 784 ($28 \times 28$ pixels). On the left and on the right, images of the digits 1 and 2, respectively. The manifold of 1s uncovers their orientation, while the manifold of 2s parameterizes features like size, slant, and line thickness.
		Details are better perceived by zooming on the plots.
	}
	\label{fig:embedding_real}
\end{figure*}

\begin{figure*}[p]
	\centering
	\begin{minipage}[t]{.59\textwidth}
    	\vspace{0pt}
		\centering
		\includegraphics[width=\linewidth]{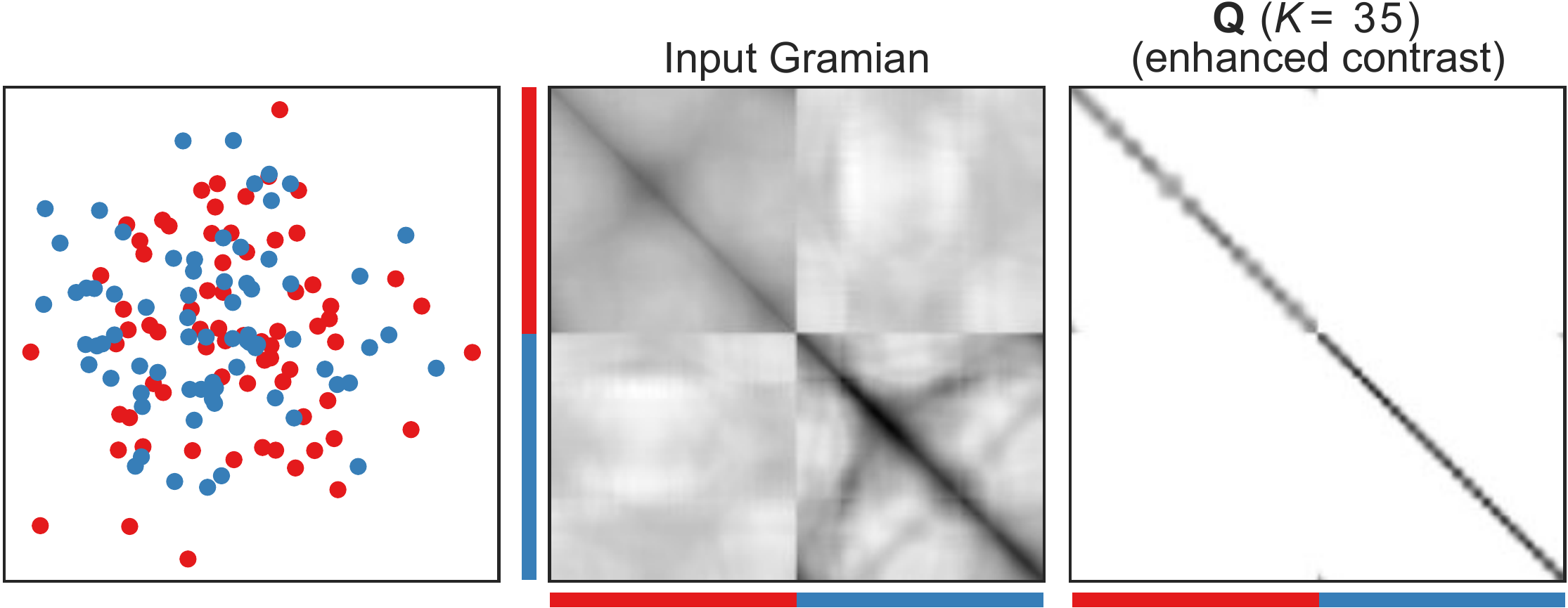}
	\end{minipage}%
	\hfill
	\begin{minipage}[t]{.39\textwidth}
    	\vspace{0pt}
		\centering
        \begin{tabu} to \linewidth {@{\hspace{0pt}} X[c,m] @{\hspace{3pt}} X[c,m] @{\hspace{0pt}}}
			{\scriptsize `Horse' manifold} &
            {\scriptsize `Lamp' manifold} \\	
			\includegraphics[width=\linewidth]{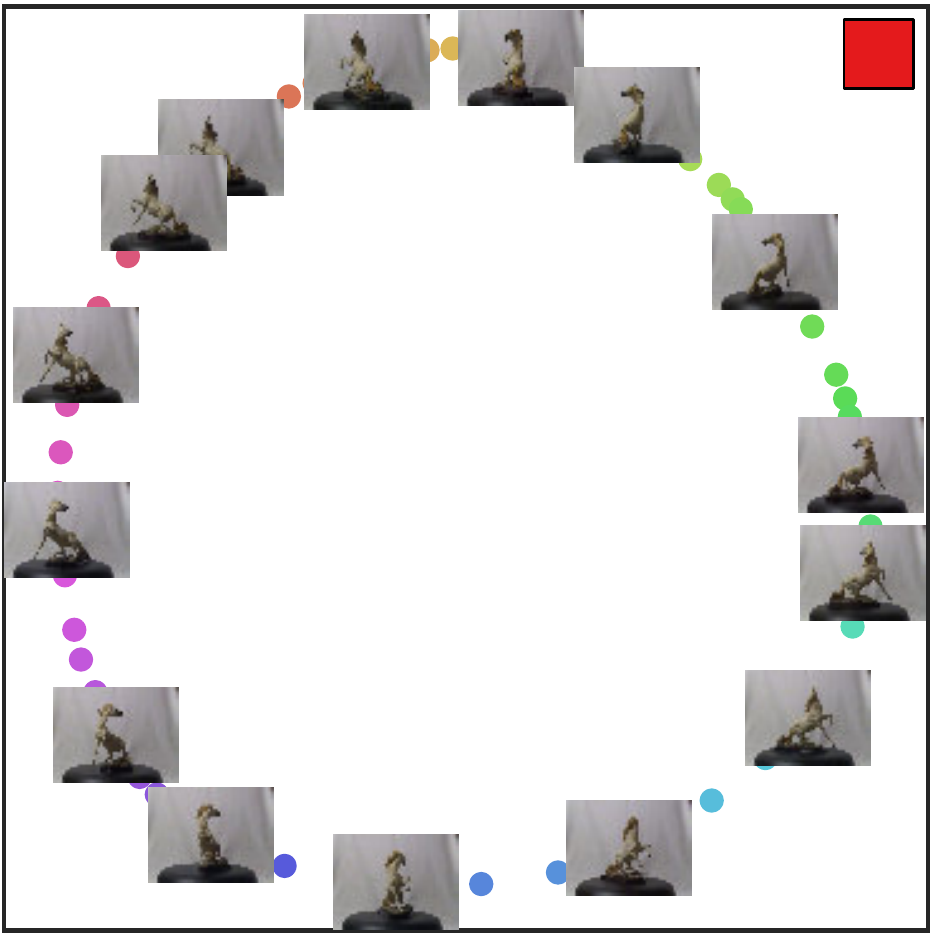} &
			\includegraphics[width=\linewidth]{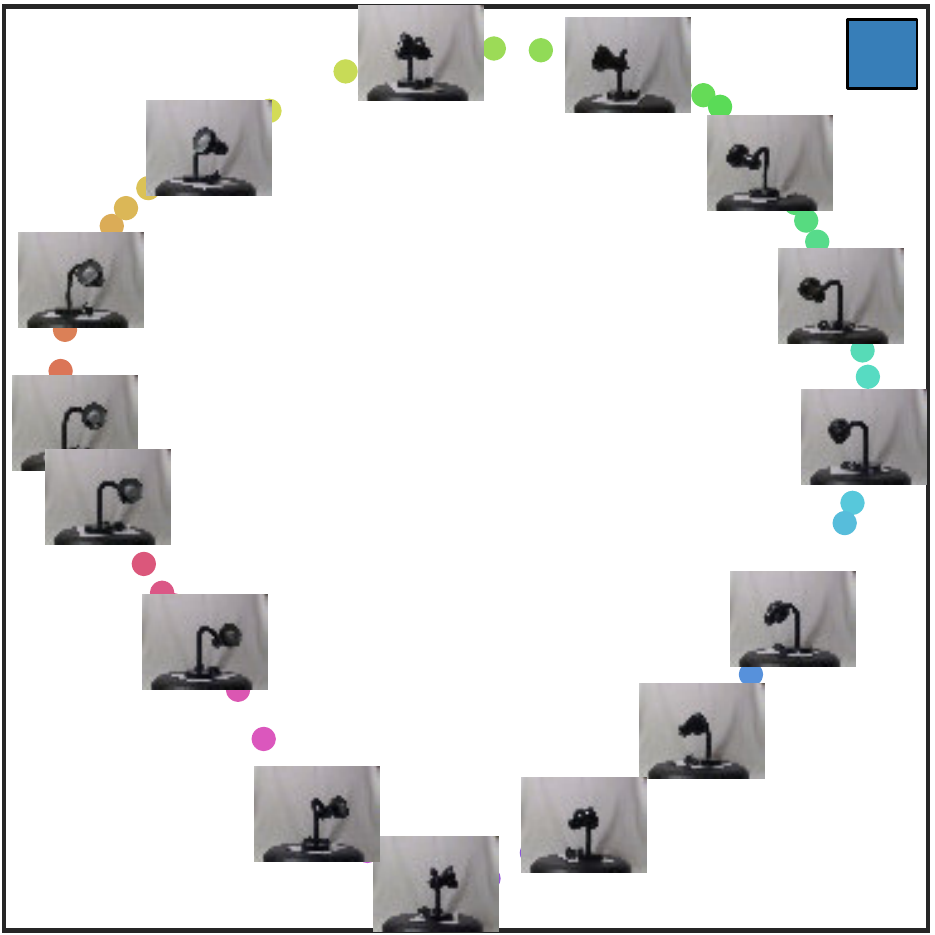}%
		\end{tabu}
	\end{minipage}
	
	\caption{144 images obtained by viewing a lamp and a horse figurine from different angles in a plane. The input vectors size is 589824 ($384 \times 512$ pixels, 3 color channels).
	We plot the input data using a 2D spectral embedding (the points corresponding to each object are colored differently).
	NOMAD correctly finds two submatrices, one for each manifold (for visual clarity, we enhance the contrast of $\mat{Q}$); furthermore, NOMAD recovers closed manifolds.
	}
\label{fig:joint}
\end{figure*}

\begin{figure}[t]
	\centering
	\begin{subfigure}[b]{.8\textwidth}
		\centering
		\includegraphics[width=\linewidth]{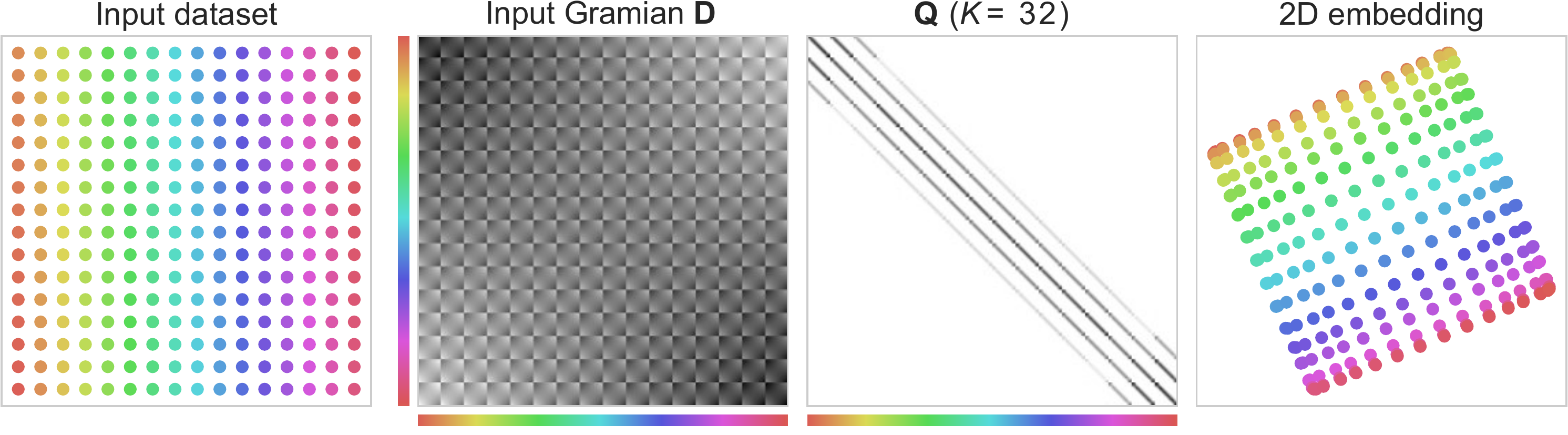}
		
		\caption{}
		\label{fig:square_embedding}
	\end{subfigure}
	\begin{subfigure}[b]{.8\textwidth}
		\centering
        \includegraphics[width=\linewidth]{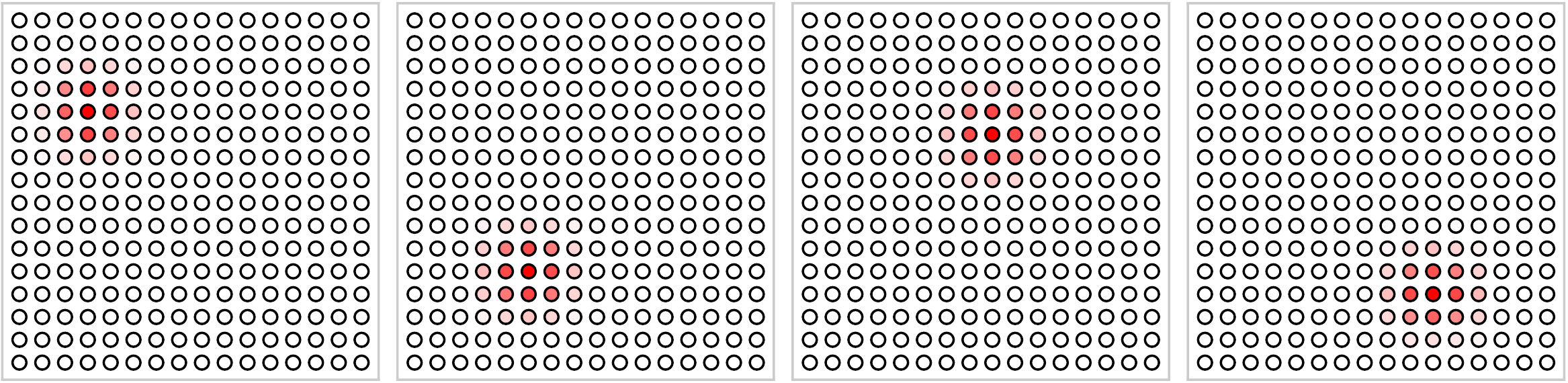}
		
		\caption{}
		\label{fig:square_Q_on_data}
	\end{subfigure}
	
	\caption{Learning a 2D manifold embedded in a 10-dimensional ambient space; the first two dimensions are regular samples on a 2D grid (shown on the left) and the remaining ones are Gaussian noise. \protect\subref{fig:square_embedding} NOMAD recovers the right 2D structure (of course, some distortion is expected along the edges).
	\protect\subref{fig:square_Q_on_data} We show a few columns of $\mat{Q}$ on top of the data itself: the red level indicates the value of the corresponding entry in $\mat{Q}$ (red maps to high values, white maps to a zero).
	NOMAD effectively tiles the dataset with a collections of overlapping local neighborhoods centered at each data point. These patches contain all the information necessary to reconstruct the intrinsic manifold geometry. 
	}
	\label{fig:square}
\end{figure}

\noindent\textbf{Recovering multiple manifolds.}
$K$-means cannot effectively recover multiple distinct manifolds (although in some very particular cases, with well separated and linearly separable manifolds, it may group the points correctly). Interestingly, NOMAD does not inherit this limitation. Of course, if we set the NOMAD parameter $K$ to the number of manifolds that we want to recover, there is no hope in the general case to obtain a result substantially better than the one obtained with Lloyd's algorithm \citep{Lloyd1982}. However, setting the NOMAD parameter $K$ to be higher than the number of manifolds leads to a correct identification and characterization of their structures.  Note that setting a similarly large $K$ would not help $K$-means, as it is designed to partition the data, thus breaking each manifold into several pieces.

An example with two rings is presented in \cref{fig:circles_eigendecomposition}. We can expect that, as the single ring in \cref{sec:theory} is described by Fourier modes, NOMAD describes two rings with two sets of Fourier modes with disjoint support; the solution is now arranged as two orthogonal high-dimensional cones, see \cref{fig:two_cones_structure}. In a sense, the manifold learning problem is already solved, as there are two circulant submatrices, one for each manifold, with no interactions between them.
If the user desires a hard assignment of points to manifolds, we can simply consider $\mat{Q}_*$ as the adjacency matrix of a weighted graph and compute its connected components.

\paragraph{Discussion of the experimental results.}
To demonstrate the manifold-learning capabilities of the NOMAD, we present several examples, both synthetic and real-world. The trefoil knot in \cref{fig:trefoil} is a 1D manifold in 3D; it is the simplest example of a nontrivial knot, meaning that it is not possible to ``untie'' it in three dimensions without cutting it. However, the manifold learning procedure in \cref{sec:manifold} learns a closed 1D manifold.
We also present examples using real-world high-dimensional datasets, recovering in every case structures of interest, see \cref{fig:embedding_real}. In \cref{fig:embedding_real_teapot,fig:embedding_real_faces,fig:embedding_real_mnist}, NOMAD respectively uncovers the camera rotation, the orientation of the lighting source, and specific handwriting features.

To demonstrate the multi-manifold learning and manifold-disentangling capabilities of NOMAD, we use several standard synthetic datasets, see \cref{fig:double_swiss_roll,fig:circles_eigendecomposition,fig:moons}.  In all of these examples, NOMAD is able to disentangle clusters that are not linearly separable. We also present results for a real-world dataset (\cref{fig:joint}) which is similar to the one in \cref{fig:embedding_real_teapot} but with two objects. NOMAD recovers two closed manifolds, each of which containing the viewpoints of one object. The structure of the solution is similar to the one in \cref{fig:two_cones_structure}.

Finally, we include an example in which NOMAD captures the structure of a 2D manifold living in a 10-dimensional space, see \cref{fig:square}. NOMAD assigns a local patch to each data point (non-zero values for neighboring points, zeros elsewhere). These  local patches tile the manifold with overlap (as in soft-clustering), allowing to recover its grid structure. Such tiling takes place in all of the examples included in the paper.

\subsection{Manifold disentangling with multi-layer NOMAD}

The recursive application of NOMAD, with successively decreasing values of $K$, enhances its manifold-disentangling capabilities. The pseudocode is as follows:

\newlength{\interspacetitleruledtemp}%
\newlength{\algotitleheightruletemp}%
\setlength{\interspacetitleruledtemp}{\interspacetitleruled}%
\setlength{\algotitleheightruletemp}{\algotitleheightrule}%
\setlength{\interspacetitleruled}{0pt}%
\setlength{\algotitleheightrule}{0pt}%
\begin{algorithm2e}[H]

	$\mat{D}_1 \gets \transpose{\mat{X}} \mat{X}$\;
	\For{$l = 1, 2, \dots$}{
		Choose $K_l$ (for all $l>1$ we require $K_l \leq K_{l-1}$)\;
		Find the solution $\mat{Q}_l$ of NOMAD with input matrix $\mat{D}_l$ and parameter $K_l$\;
		$\mat{D}_{l+1} \gets \mat{Q}_l$\;
	}
	\Return $\{ \mat{Q}_{l} \}_{l=1,2,\dots}$
\end{algorithm2e}
\setlength{\interspacetitleruled}{\interspacetitleruledtemp}%
\setlength{\algotitleheightrule}{\algotitleheightruletemp}%

In \cref{fig:multilayer}, we present the evolution of successive matrices $\mat{Q}_{l}$. In all of these examples, multi-layer NOMAD is able to correctly identify clusters that are not linearly separable, something unattainable with single-layer NOMAD or with $K$-means clustering. Interestingly, we find that the manifolds are already segregated after one application of NOMAD in the direction of the leading eigenvectors of $\mat{Q}$ (see \cref{fig:circles_eigendecomposition}). The rest of the NOMAD layers little-by-little sieve out the (unwanted) smaller eigenvalues in an unsupervised fashion.

To turn this algorithm into a general data-analysis tool, we need an automated selection of the values $\{ K_l\}$ which is a non-trivial task in general.
Additional results, using different sequences $\{ K_l\}$, can be found in \cref{sec:additional_multilayer_results}. Further research is needed to develop such algorithm and fully understand multi-layer NOMAD's interesting behavior.

\begin{figure}[t]
    \includegraphics[width=.744\linewidth]{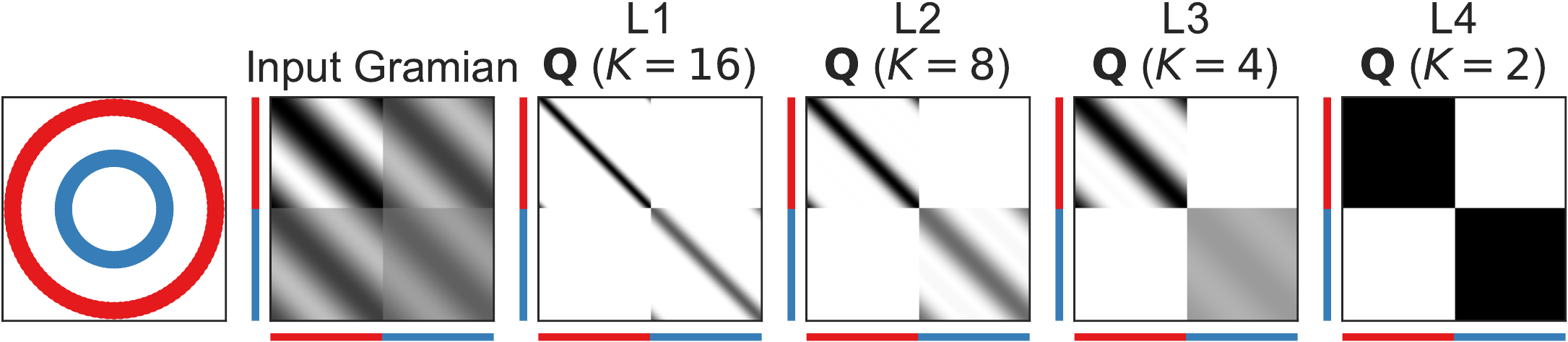}%
    \\[2pt]
    \includegraphics[width=.744\linewidth]{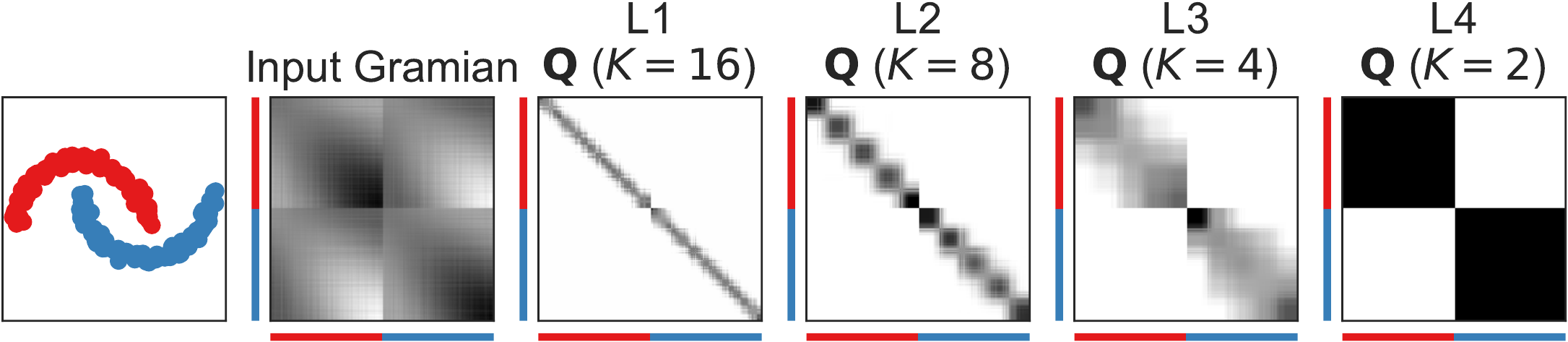}%
    \\[2pt]
	\includegraphics[width=\linewidth]{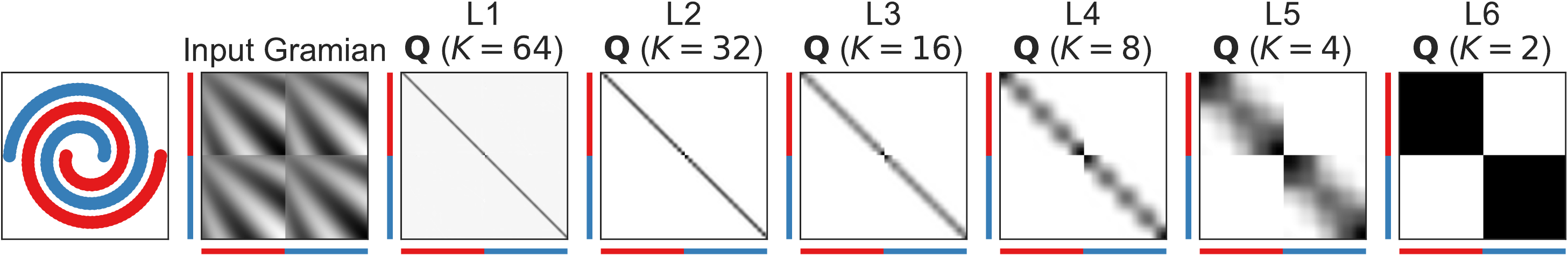}%
    
	\caption{Results of recursive NOMAD application (multi-layer NOMAD).
    For each example, we show matrices $\mat{Q}_*$ computed by the successive application of the algorithm. Multi-layer NOMAD untangles these linearly non-separable manifolds and, in the final layer, assigns each manifold to one cluster.}
	\label{fig:multilayer}
\end{figure}

\begin{figure}
	\centering   
    \begin{subfigure}{\textwidth}
    \begin{footnotesize}
    \begin{tabu} to \textwidth {X[1,c,m] *{3}{X[15,c,m]}}
        &
        No noise &
        Noise 0.05 &
        Noise 0.10 \\
        
        \begin{sideways}
        	One ring (\cref{fig:distances2gramian})
        \end{sideways} &
		\includegraphics[width=\linewidth]{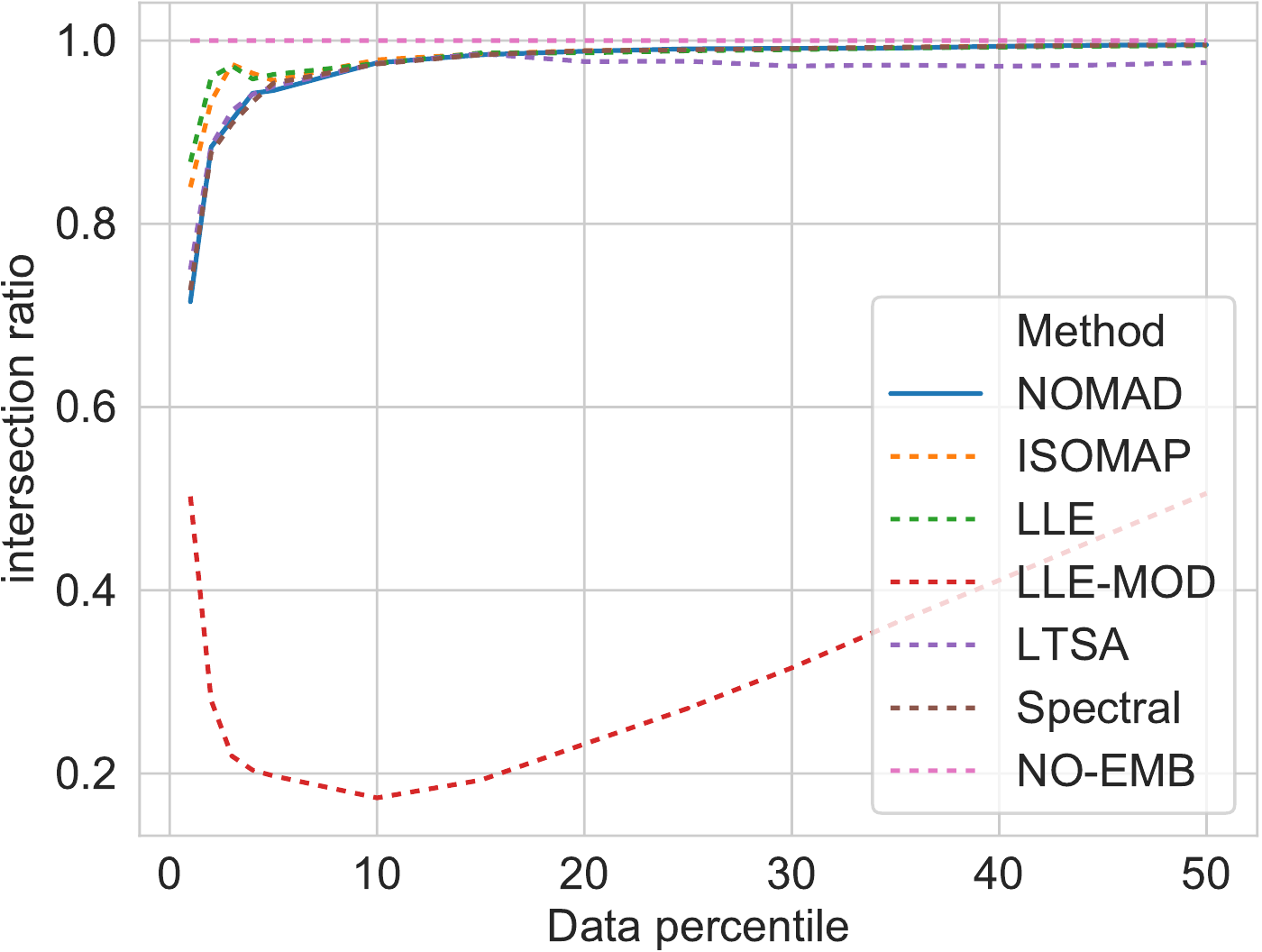} &
		\includegraphics[width=\linewidth]{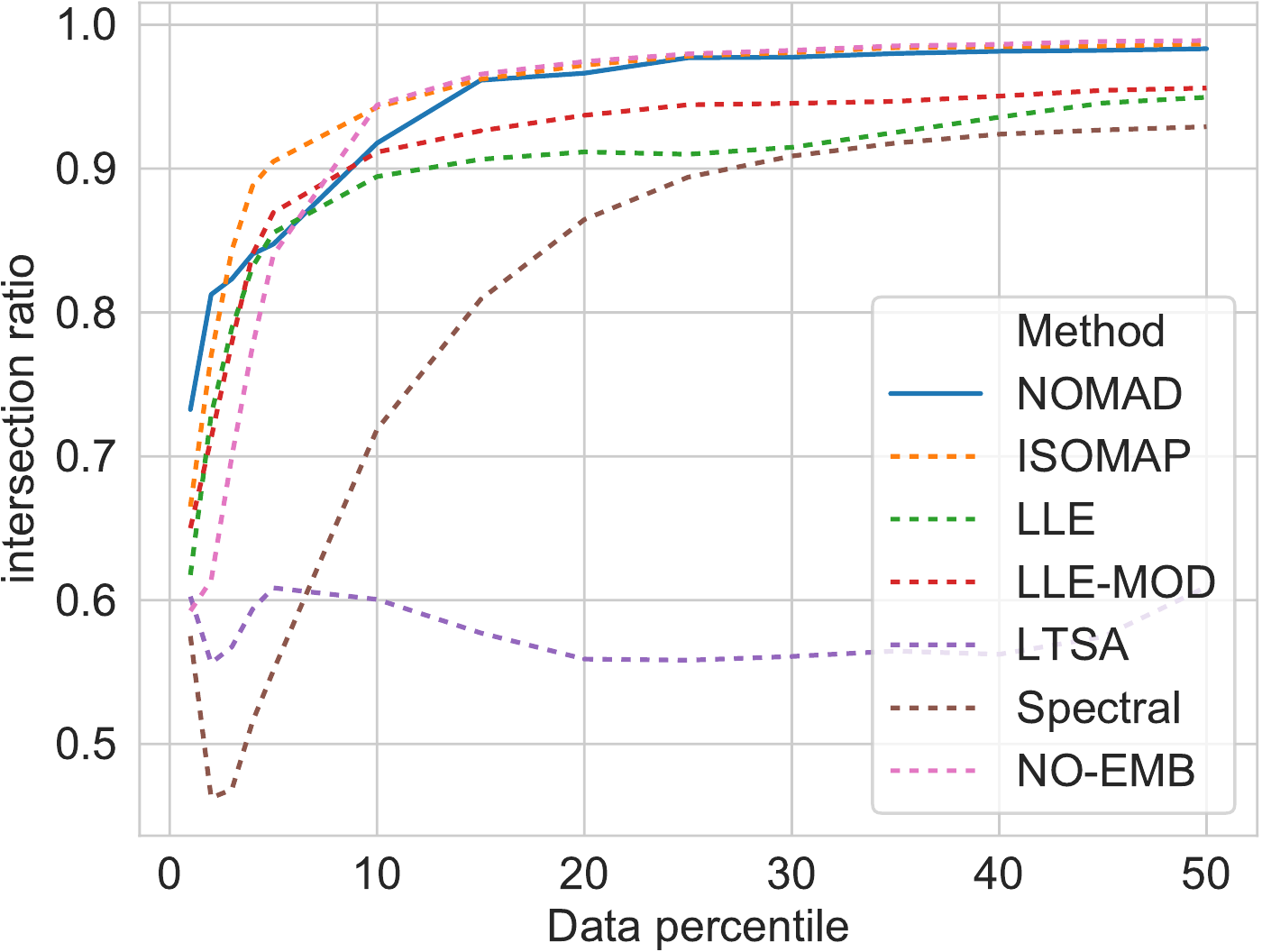} &
		\includegraphics[width=\linewidth]{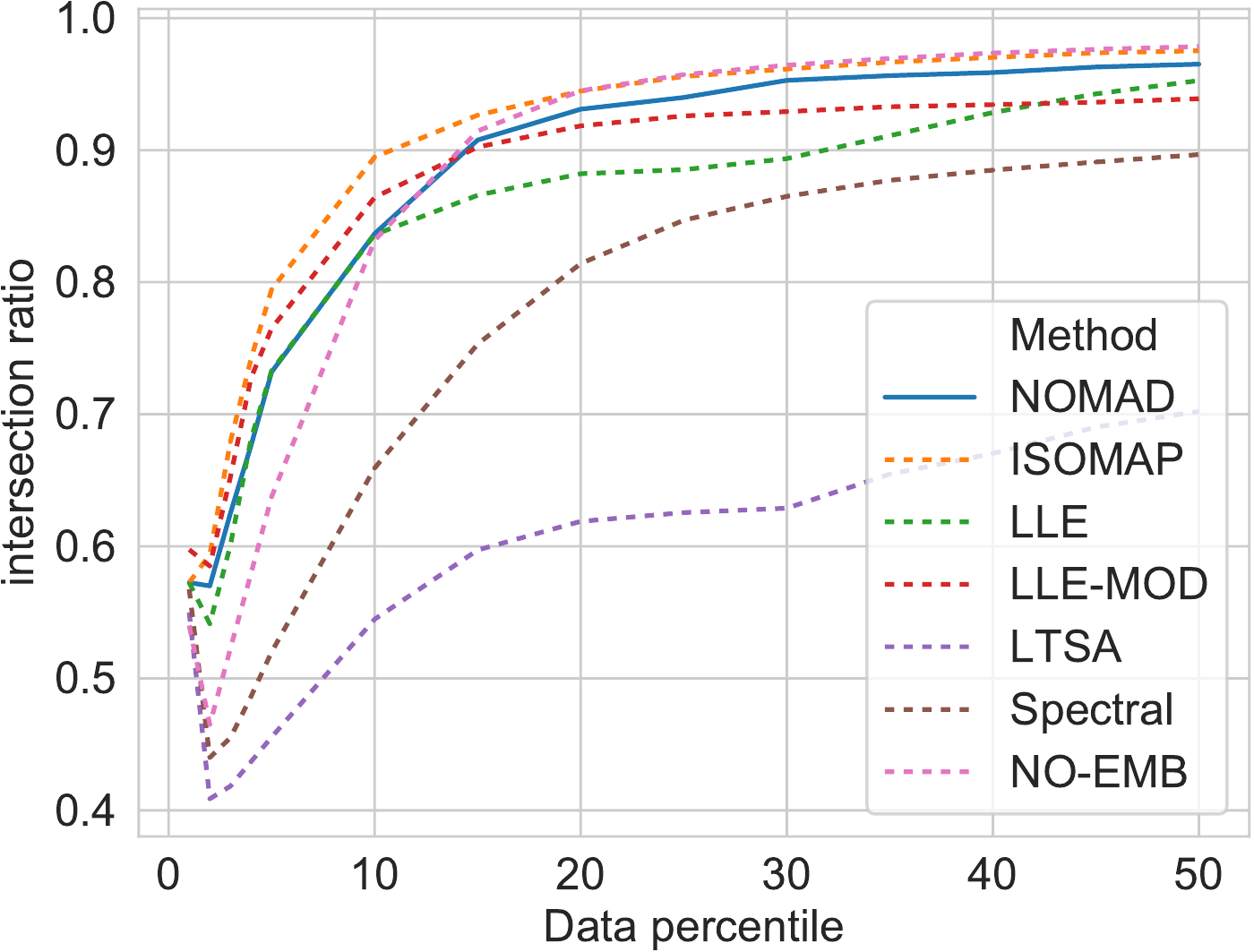} \\

    	\begin{sideways}
        	One half-ring
        \end{sideways} &
		\includegraphics[width=\linewidth]{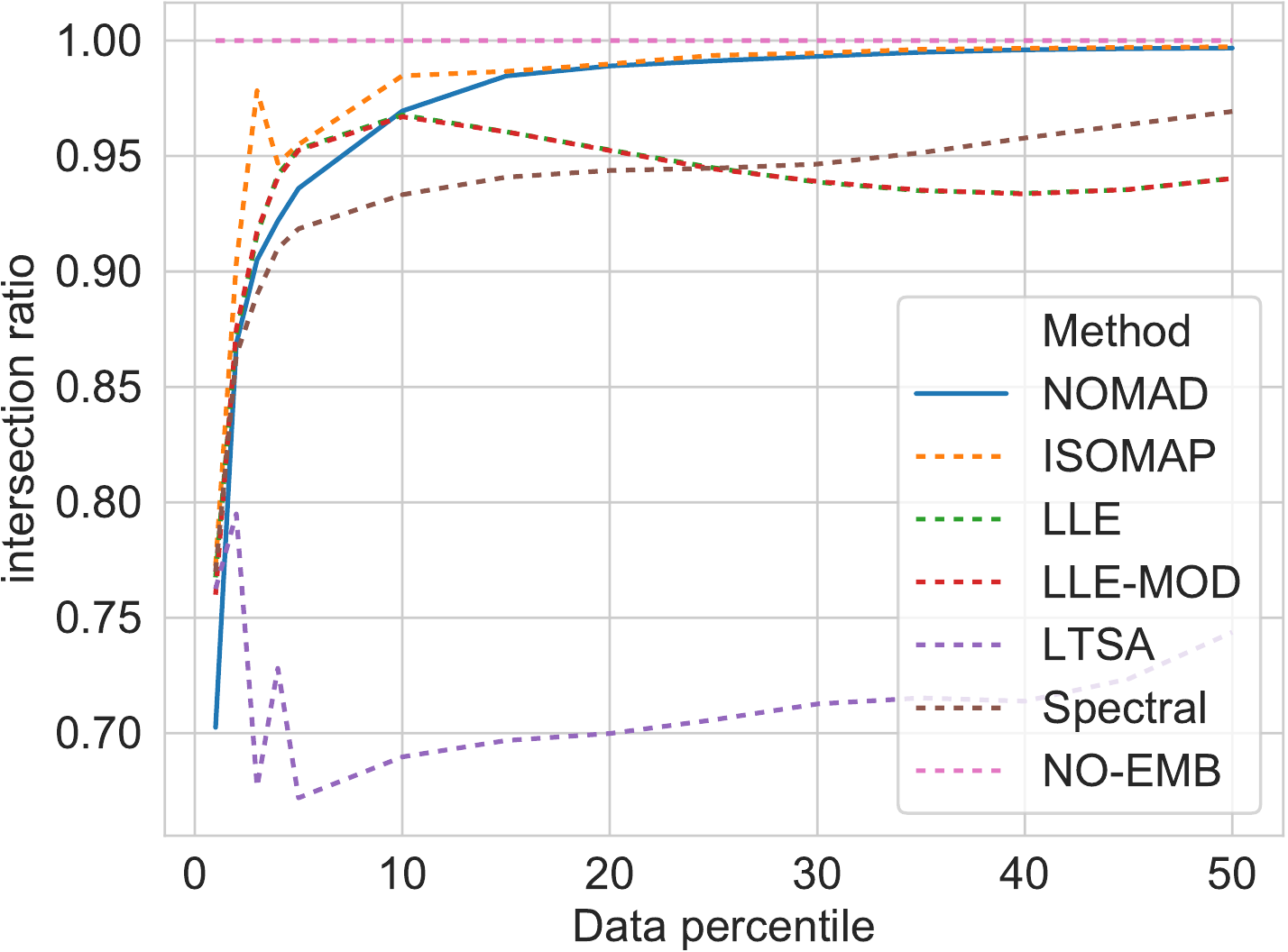} &
		\includegraphics[width=\linewidth]{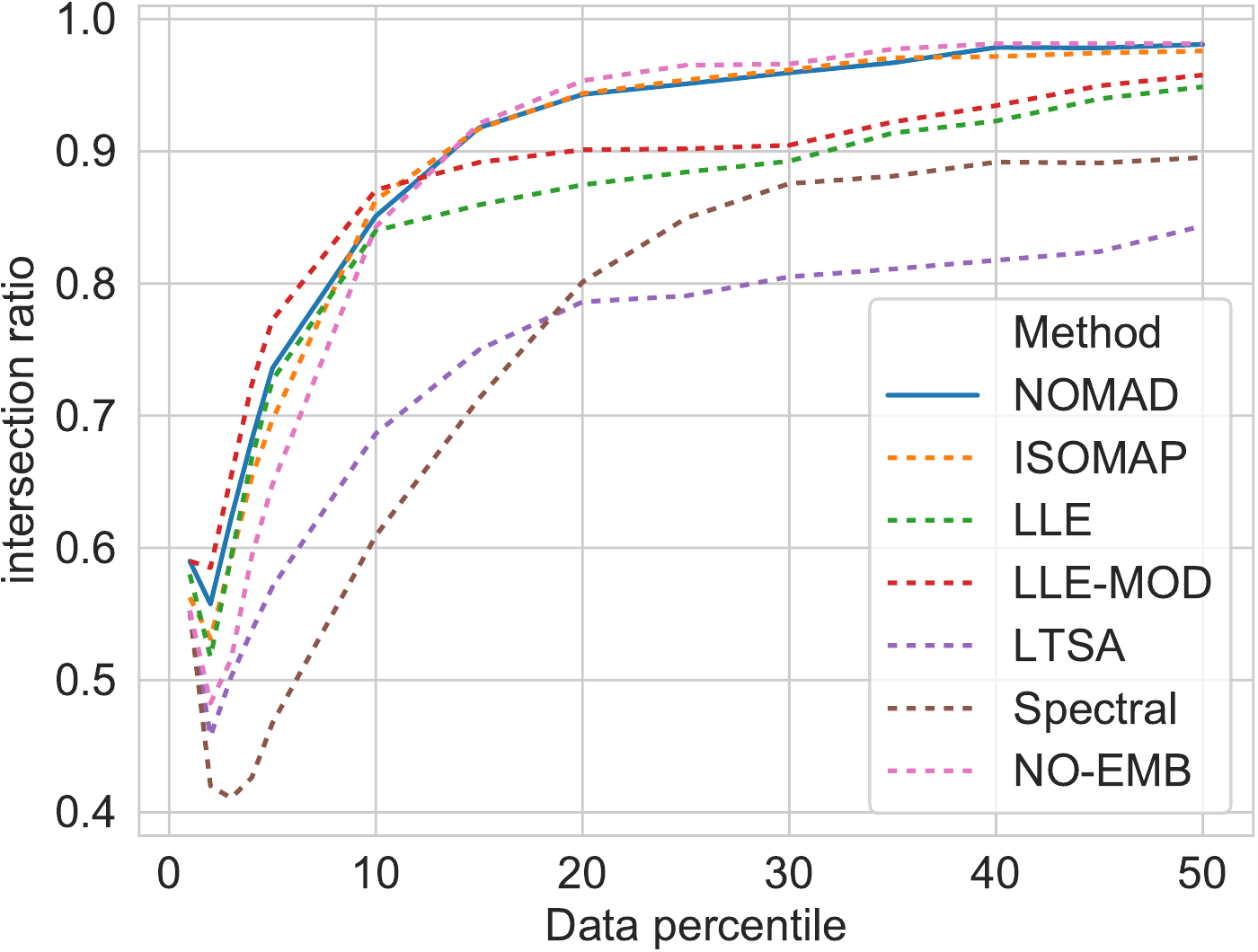} &
		\includegraphics[width=\linewidth]{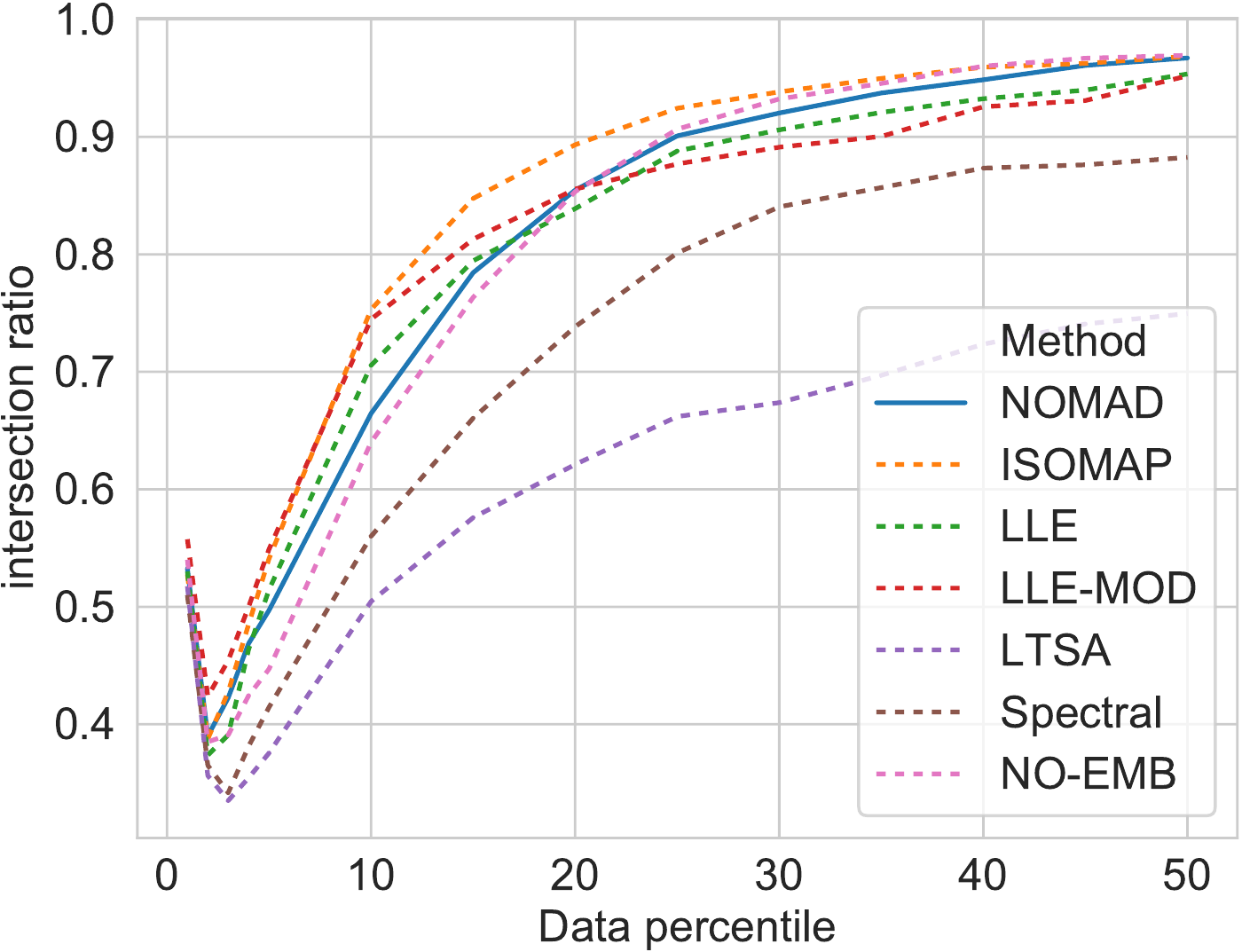} \\

    	\begin{sideways}
        	Teapots (\cref{fig:embedding_real_teapot})
        \end{sideways} &
		\includegraphics[width=\linewidth]{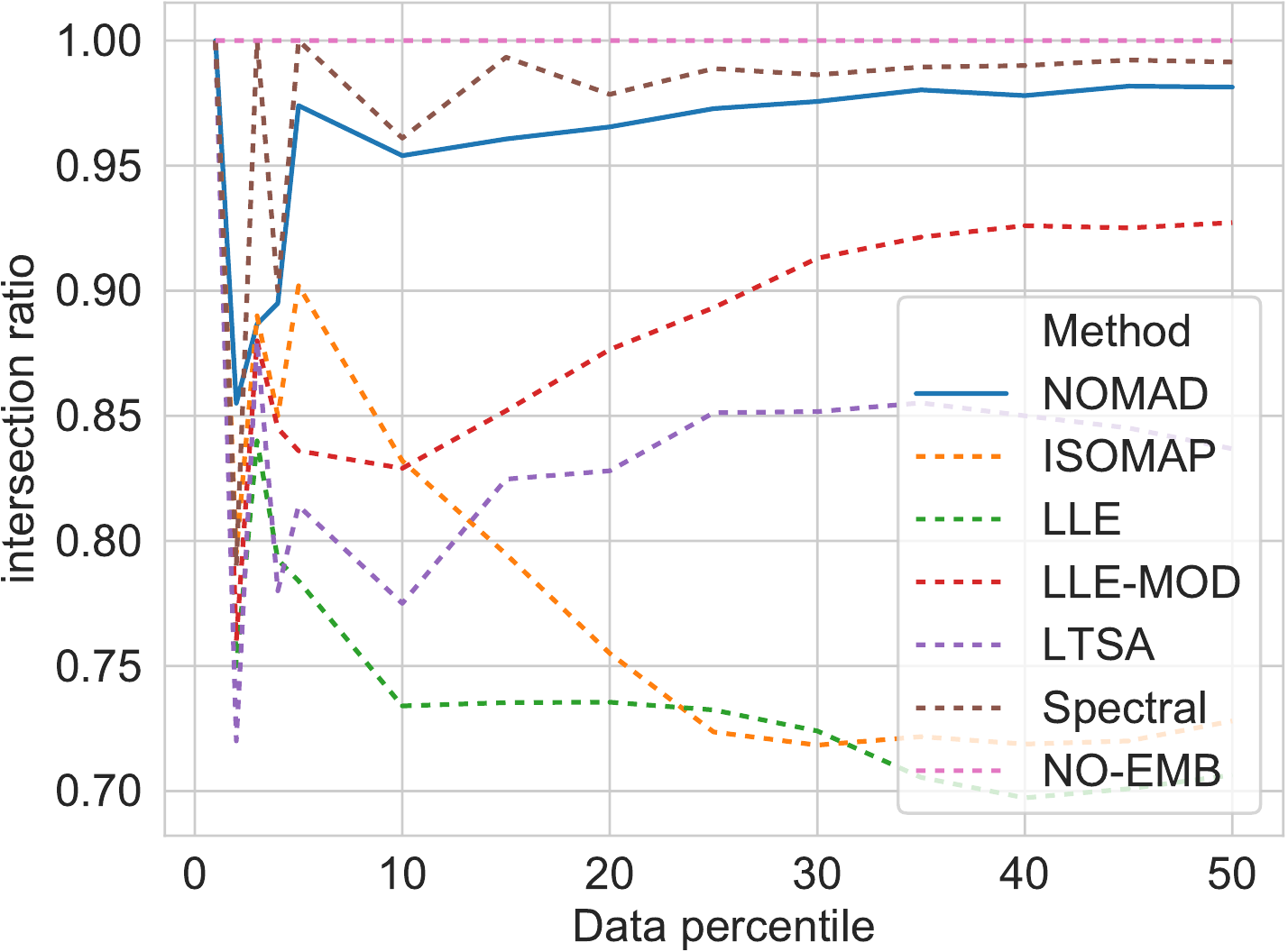} &
		\includegraphics[width=\linewidth]{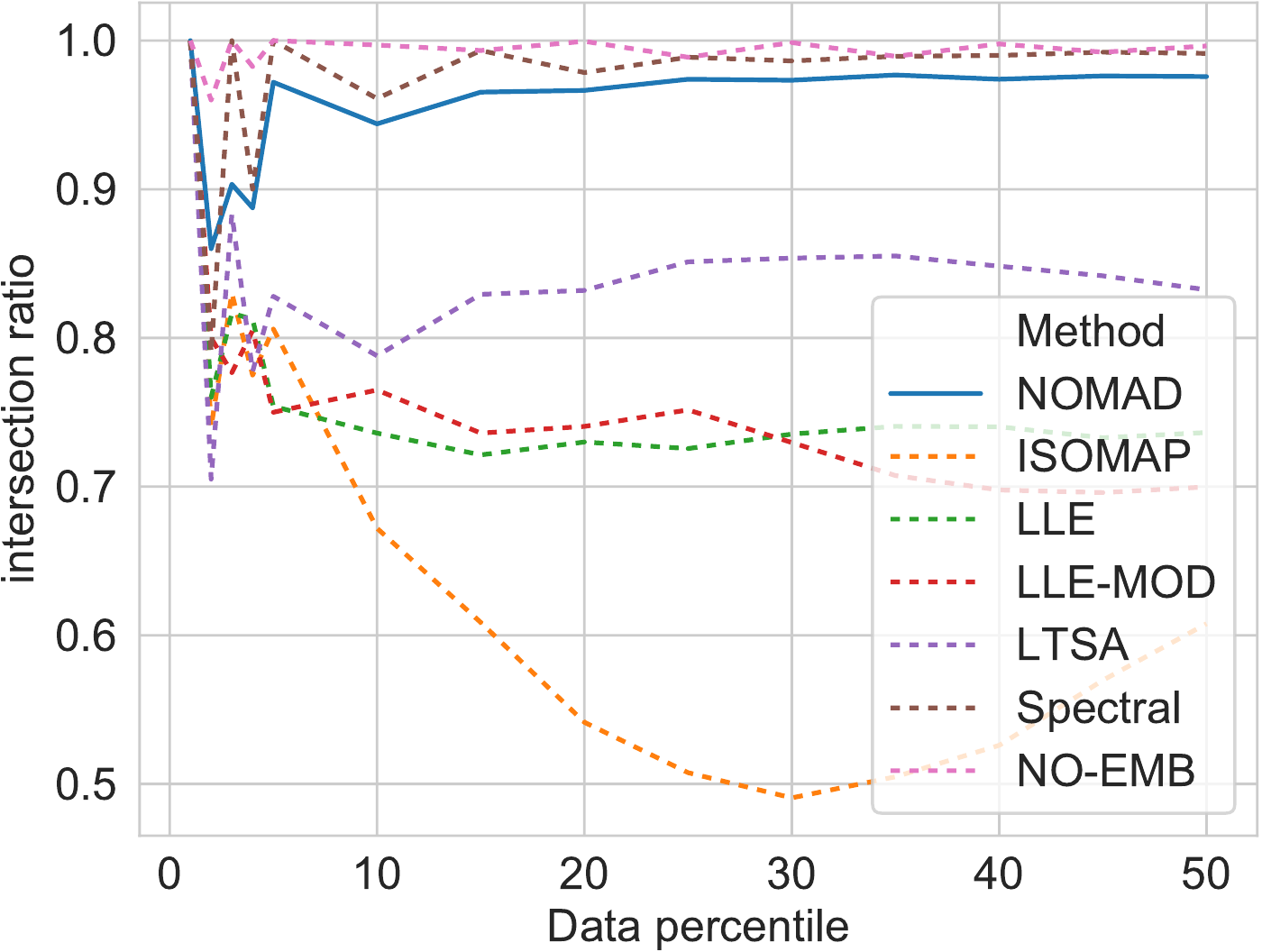} &
		\includegraphics[width=\linewidth]{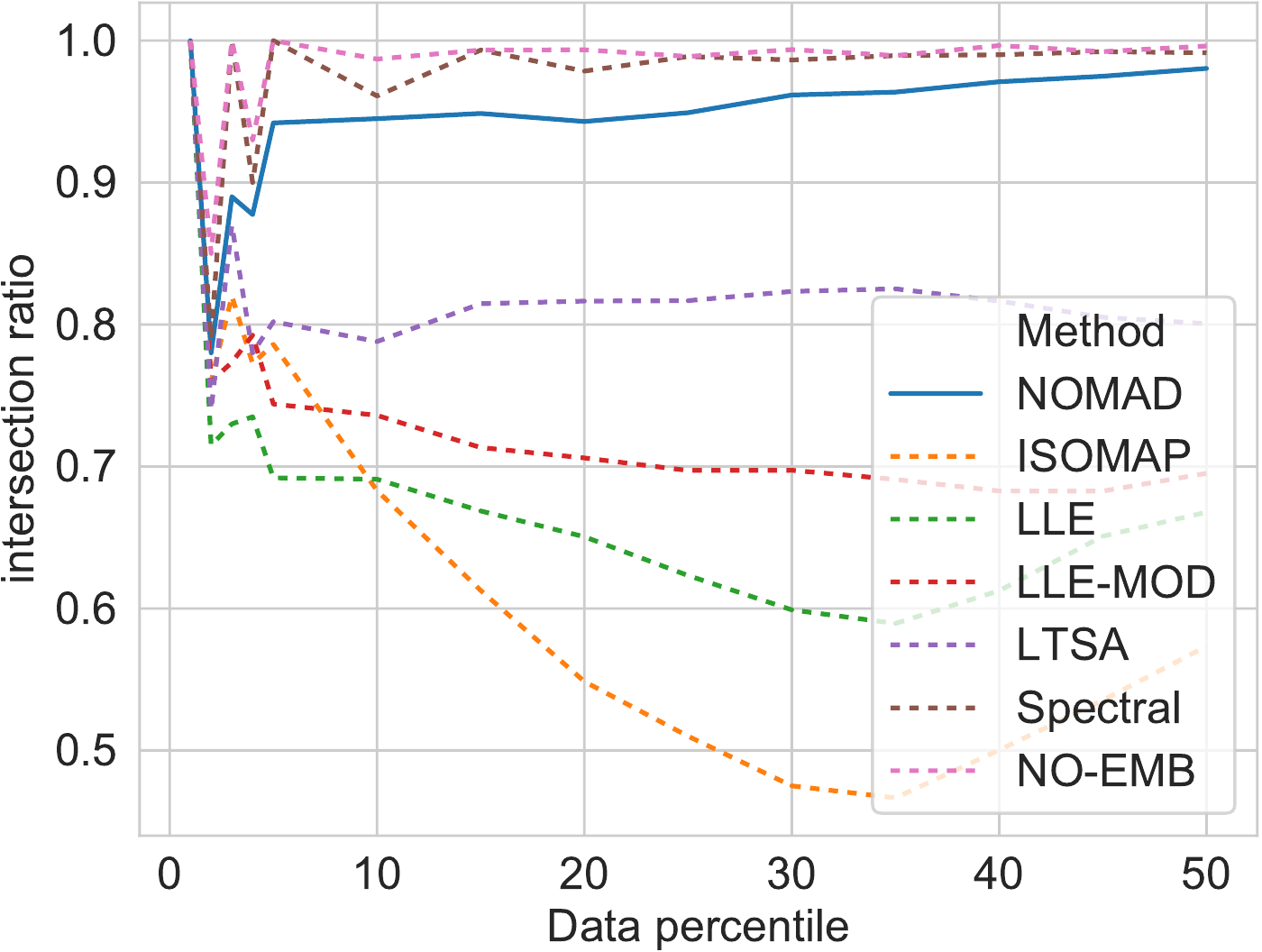} \\
    \end{tabu}
    \end{footnotesize}
    \caption{One manifold: NOMAD is among the best-in-class methods.}
    \end{subfigure}
    \\
    \begin{subfigure}{\textwidth}
    \begin{footnotesize}
    \begin{tabu} to \textwidth {X[1,c,m] *{3}{X[15,c,m]}}
    	&
        No noise &
        Noise 0.05 &
        Noise 0.10 \\
        
        \begin{sideways}
        	Two rings (\cref{fig:2circles})
        \end{sideways} &
		\includegraphics[width=\linewidth]{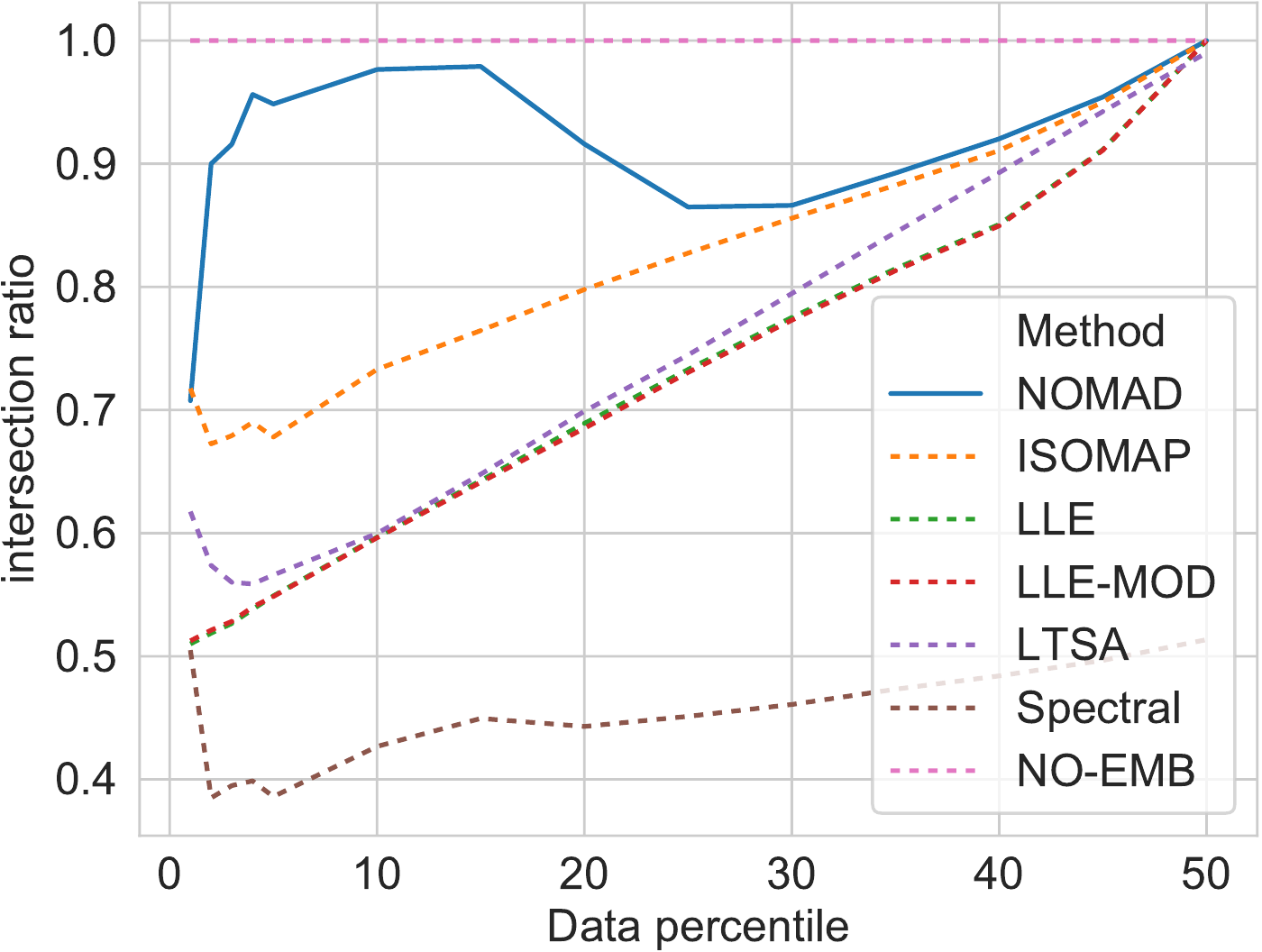} &
		\includegraphics[width=\linewidth]{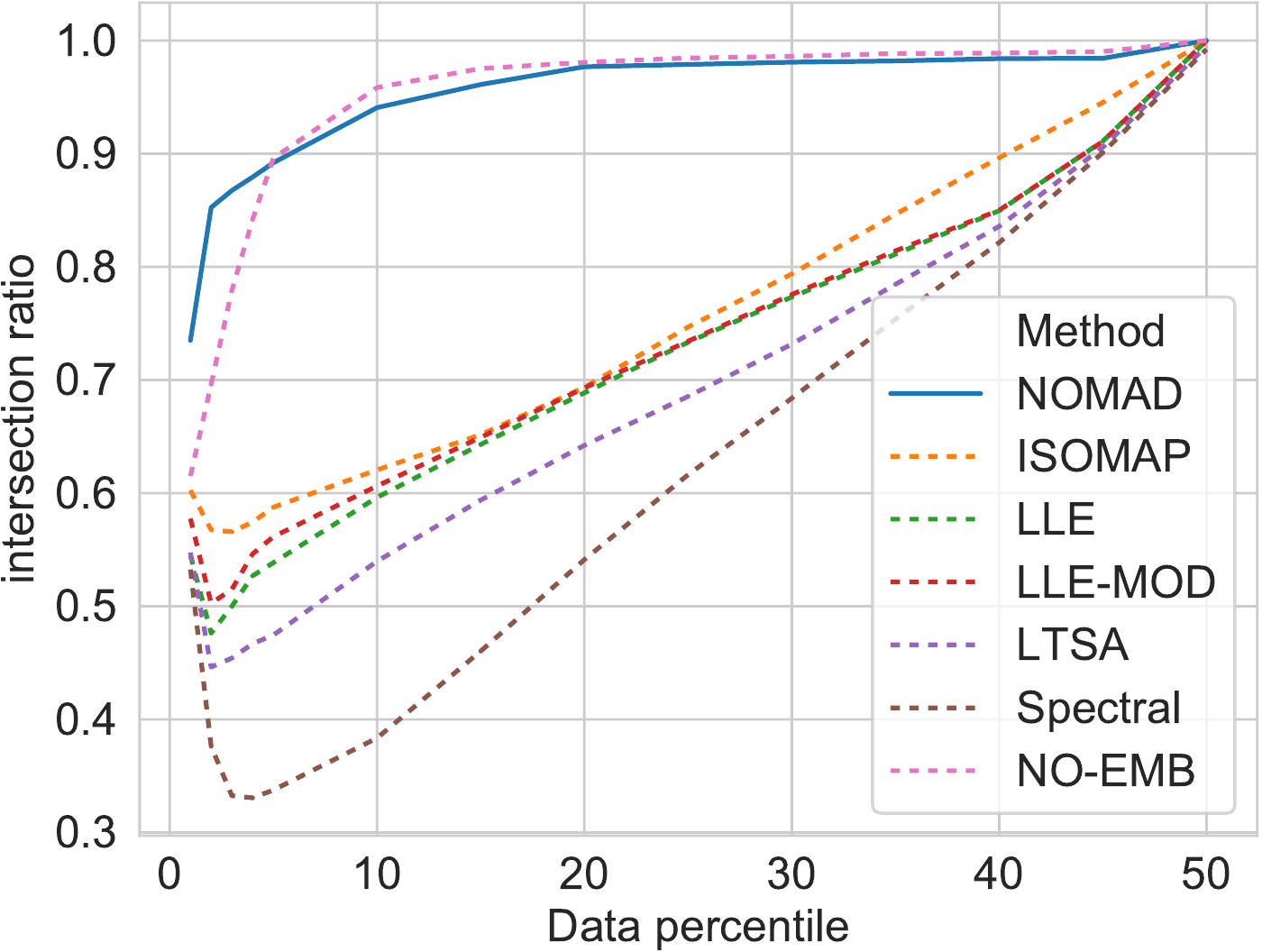} &
		\includegraphics[width=\linewidth]{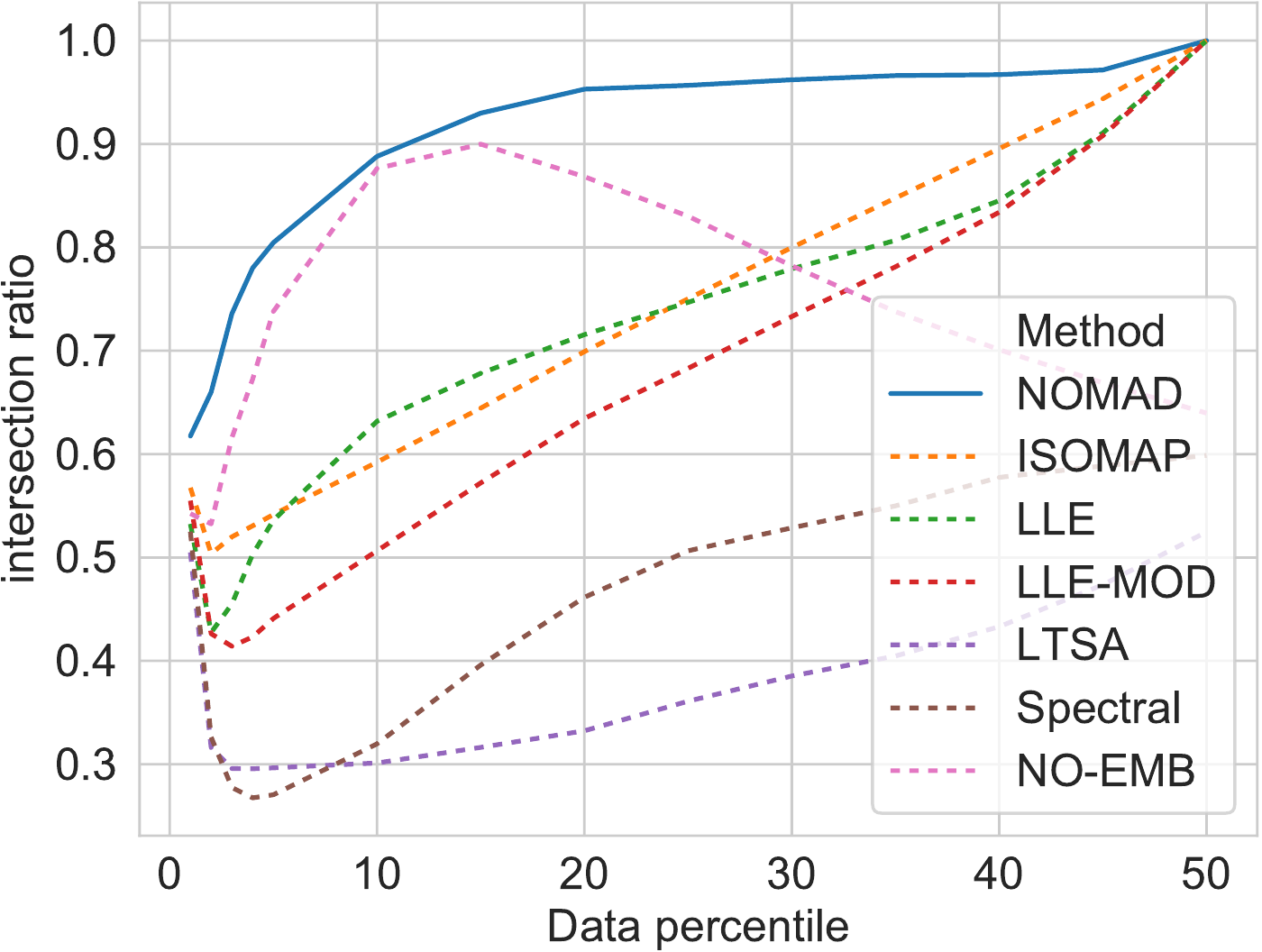} \\

		\begin{sideways}
        	Two half-rings (\cref{fig:moons})
        \end{sideways} &
		\includegraphics[width=\linewidth]{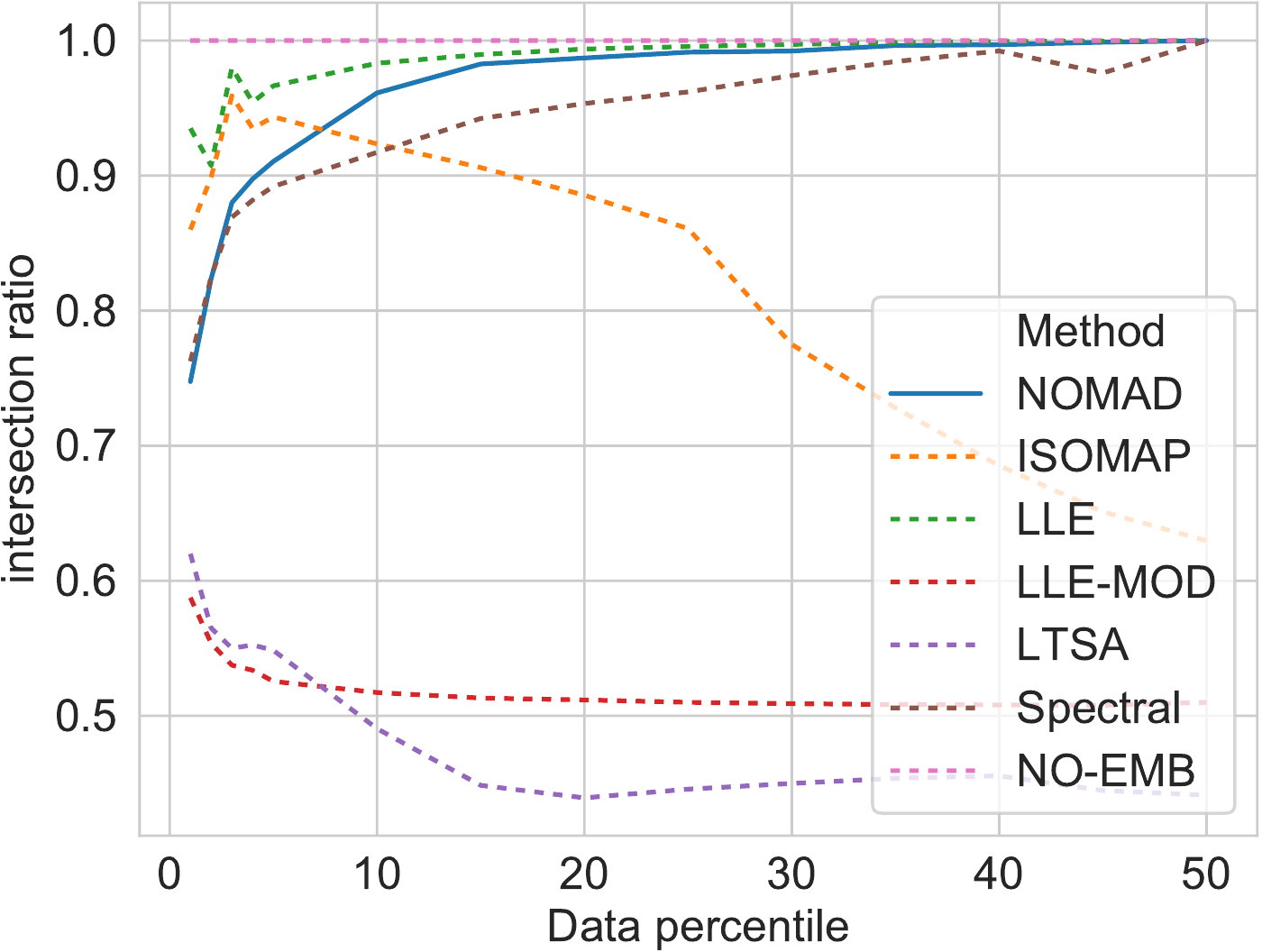} &
		\includegraphics[width=\linewidth]{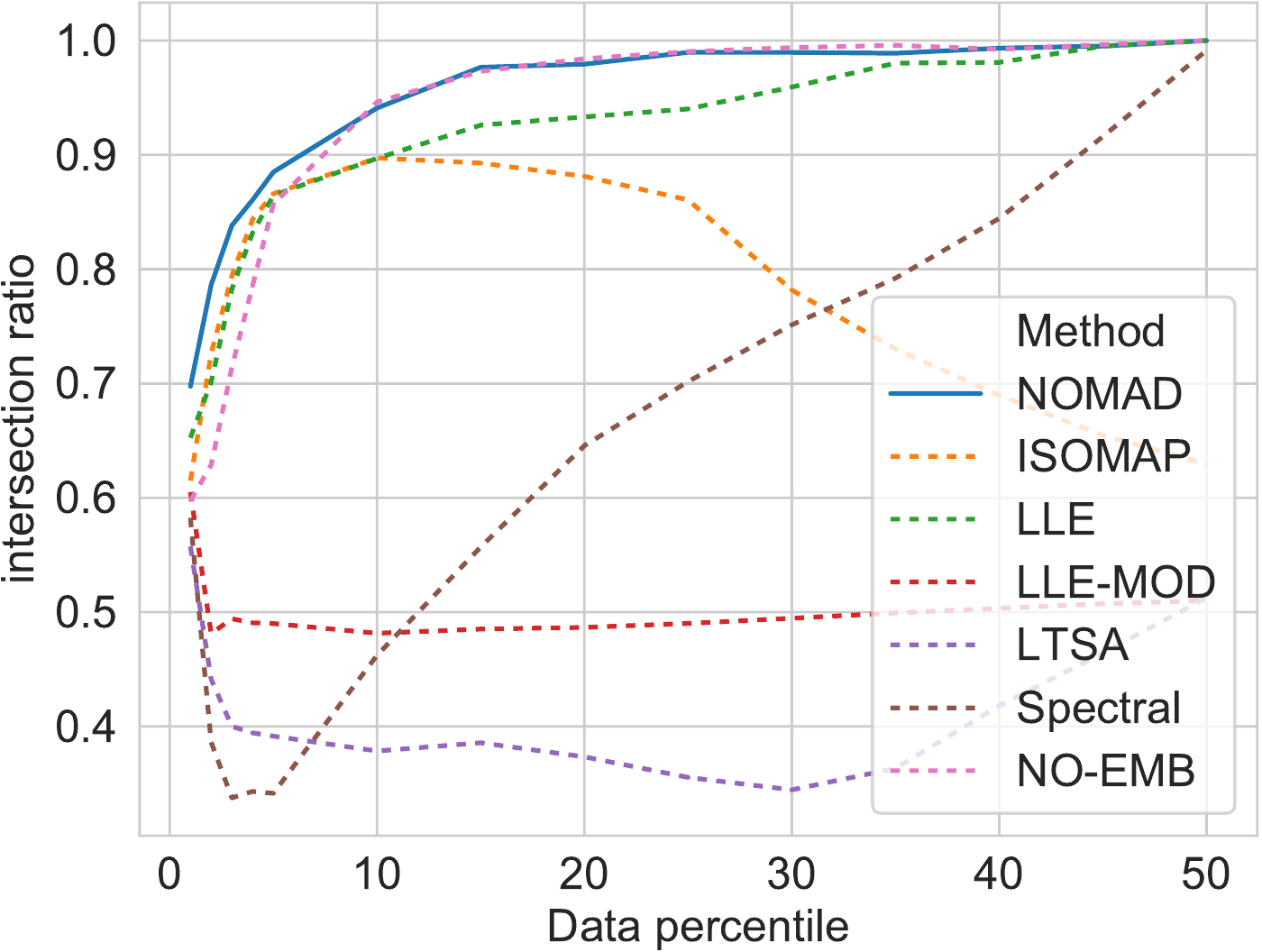} &
		\includegraphics[width=\linewidth]{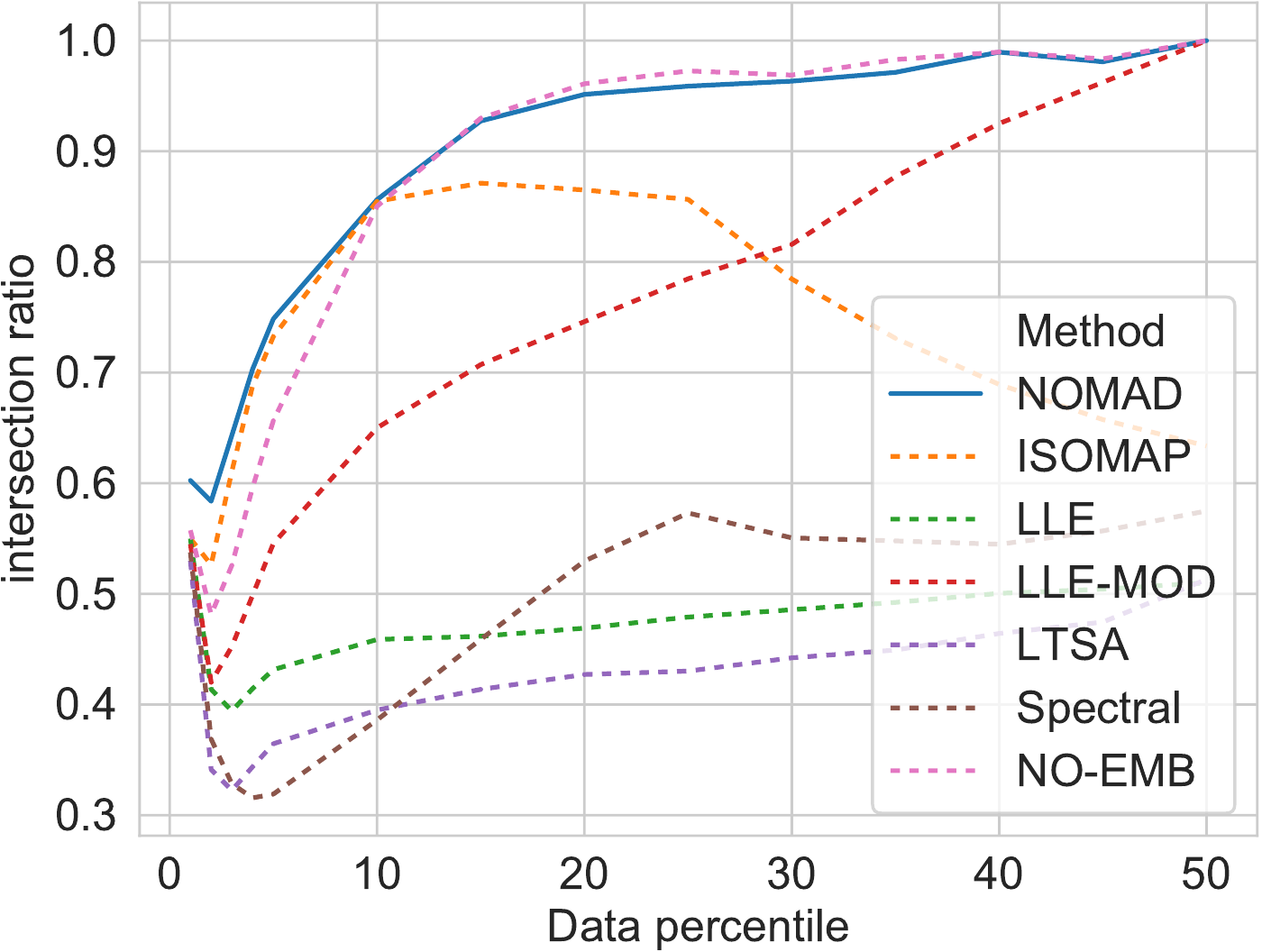} \\
    \end{tabu}
    \end{footnotesize}
    \caption{Two manifolds: NOMAD outperforms all considered methods.}
    \end{subfigure}
    
    \caption{Comparison of the robustness of geodesic distances to the addition of noise for different manifold learning methods. The description of the experimental protocol is in \cref{sec:geodesics}. The method termed NO-EMB computed distances on the noisy data directly.}
    \label{fig:distance_preservation}
\end{figure}

\subsection{Geodesic-distance preservation: NOMAD versus existing manifold learning techniques}
\label{sec:geodesics}

In order to compare different methods for manifold discovery, we need to agree upon appropriate metrics, which itself is an active area of research \citep[e.g.,][]{Zhang2012}. In particular, we want a metric that allows fair comparison among outputs of methods with very different objectives. Since our method is not explicitly geared towards dimension reduction, or towards variance maximization, we prefer metrics that emphasize the preservation of intrinsic structure of the manifold. In particular, we hope to preserve the ordering of intrinsic distances along the manifold, something that guarantees that the neighborhood structure remain similar. 

Concretely, (1) we compute $N$ nearest neighbors for each dataset point and build a weighted graph with these distances as edges; (2) we use this graph to compute geodesic distances using Dijkstra's algorithm; (3) we finally sort the distances in increasing order. We consider this ordering our ground truth.

We then add noise to the each point in the dataset. Noise was added by creating $5d$ additional dimensions that contain Gaussian noise with standard deviation $0.05$ and $0.10$.
Using the noisy dataset, we repeat the geodesic-distance-sorting process using the embeddings produced by different manifold learning algorithms using $N$ nearest neighbors. For NOMAD, instead of resorting to nearest neighbors, we set $K=n/N$ and use the non-zero entries in $\mat{Q}$ to determine the graph connectivity. As NOMAD yields a similarity matrix $\mat{Q}$, we derive distances from it with the formula $(\mat{Q})_{ii} + (\mat{Q})_{jj} - 2 (\mat{Q})_{ij}$. Using this weighted graph, we compute and sort the geodesic distances.

For our distance-preservation measure, we use a bullseye score: for each method, we count  the fraction of points in the top $p$ percentile of distances that are also present in the top $p$ percentile of ground truth distances.

As seen in \cref{fig:distance_preservation}, when the data is sampled from a single manifold, NOMAD performs very well, on par with the best algorithms included in our comparison. However, in the two-manifolds case, NOMAD clearly outperforms all other methods, nearly matching the performance of direct distance computations on the noisy data.

\section{Heuristic non-convex solvers for large-scale NOMAD}
\label{sec:BM}

Standard SDPs involve $O(n^2)$ variables and their resulting time complexity is often $O(n^3)$. Consequently, standard solvers \citep{ODonoghue2016} will struggle with large datasets.
NOMAD lends itself to a fast and big-data-friendly implementation \citep{Kulis2007}. This is done by posing a related problem
\begin{equation}
	\max_{\mat{Y} \in \Real^{r \times n}}
	\traceone{\mat{D} \transpose{\mat{Y}} \mat{Y}}
	\enskip\text{s.t.}\enskip
	\traceone{\transpose{\mat{Y}} \mat{Y}} = K ,\
	\transpose{\mat{Y}} \mat{Y} \vect{1} = \vect{1} ,\
	\mat{Y} \geq \mat{0} .
	\label[problem]{eq:sdp_kmeans_lowrank_fast}
\end{equation}
In this new problem, we have forgone convexity in exchange of reducing the number of unknowns from $O(n^2)$ to $rn$. For example, \citet{Kulis2007} set $r = K$.
The problematic constraint $\mat{\transpose{\mat{Y}} \mat{Y}} \geq \mat{0}$, involving $O(n^2)$ terms, has been replaced by the much stronger but easier to enforce $\mat{Y} \geq \mat{0}$. 
The speed gain is shown in \cref{fig:mnist_timing}. See \cref{sec:burer-monteiro} for a description of the algorithm.

However, strictly speaking, the new constraint is equivalent to the old one only if $\mat{Q}$ is completely positive.
An $n \times n$ matrix $\mat{A}$  is called completely positive (CP) if there exists $\mat{B} \geq \mat{0}$ such that $\mat{A} = \transpose{\mat{B}} \mat{B}$. The least possible number of rows of $\mat{B}$ is called the cp-rank of $\mat{A}$. Whereas matrix $\mat{A}$ is doubly nonnegative (DN), i.e. $\mat{A} \geq \mat{0}$ and $\mat{A} \succeq 0$, not every DN matrix (with $n  > 4$) is CP \citep{Maxfield1962}.

We are thus interested in two questions.
First, is the solution $\mat{Q}_*$ to NOMAD completely positive? Answering this question in the affirmative would allow for theoretically sound and fast implementations of NOMAD. Whereas the set of CP matrices forms a convex cone, the problems of determining whether a matrix is inside the set and of projecting a matrix into the set are NP-hard leading us to the second question: What is the cp-rank of $\mat{Q}_*$? This issue is critical because it determines the number of unknowns. For example, if $\operatorname{cp-rank}(\mat{Q}_*) \leq K$, \eqref{eq:sdp_kmeans_lowrank_fast} would be easier to solve. These questions are difficult only when NOMAD produces a soft-clustering $\mat{Q}_*$, as in all of the examples in this paper. Indeed, it is not hard to prove that, whenever NOMAD produces a hard-clustering $\mat{Q}_*$, $\mat{Q}_*$ is CP \citep[see][for such conditions]{Awasthi2015}.

Let us now go back to the example in \cref{sec:theory} (points arranged regularly on a ring). For this example, we can establish a simple sufficient condition on $K$, for $\mat{Q}_*$ to be CP. Recall that if $\mat{D}$ is circulant, $\mat{Q}_*$ is circulant \citep{Bachoc2012}.
In \cref{prop:circulantCP} of \cref{sec:proofs}, we prove that if the solution $\mat{Q}_*$ to NOMAD is a circulant matrix, then it is CP for every $K \leq 3/2$ or $K \geq \tfrac{n}{2}$.
Naturally, more theory is needed to shed light onto this problem in general scenarios as it is unclear whether similar results exist.

Complementarily, we have studied the questions raised in this section from an experimental viewpoint. We use the symmetric nonnegative matrix factorization (SNMF) of $\mat{Q}_*$, see \cref{sec:snmf}, as a proxy for checking whether $\mat{Q}_*$ is CP. The rationale is that if the approximation with SNMF is very tight, it is highly likely that $\mat{Q}_*$ is CP. These experiments are presented in \cref{fig:completely_positive_snmf}. We found that, with a properly chosen rank $r$, SNMF can  indeed accurately approximate $\mat{Q}_*$. However, setting $r=K$ is in general not enough and leads to a poor reconstruction. These two facts support the idea that $\mat{Q}_*$ is CP, but has a cp-rank much higher than $K$.

\begin{figure}
	
	\centerline{
		\hfill
		\begin{overpic}[width=.35\linewidth]{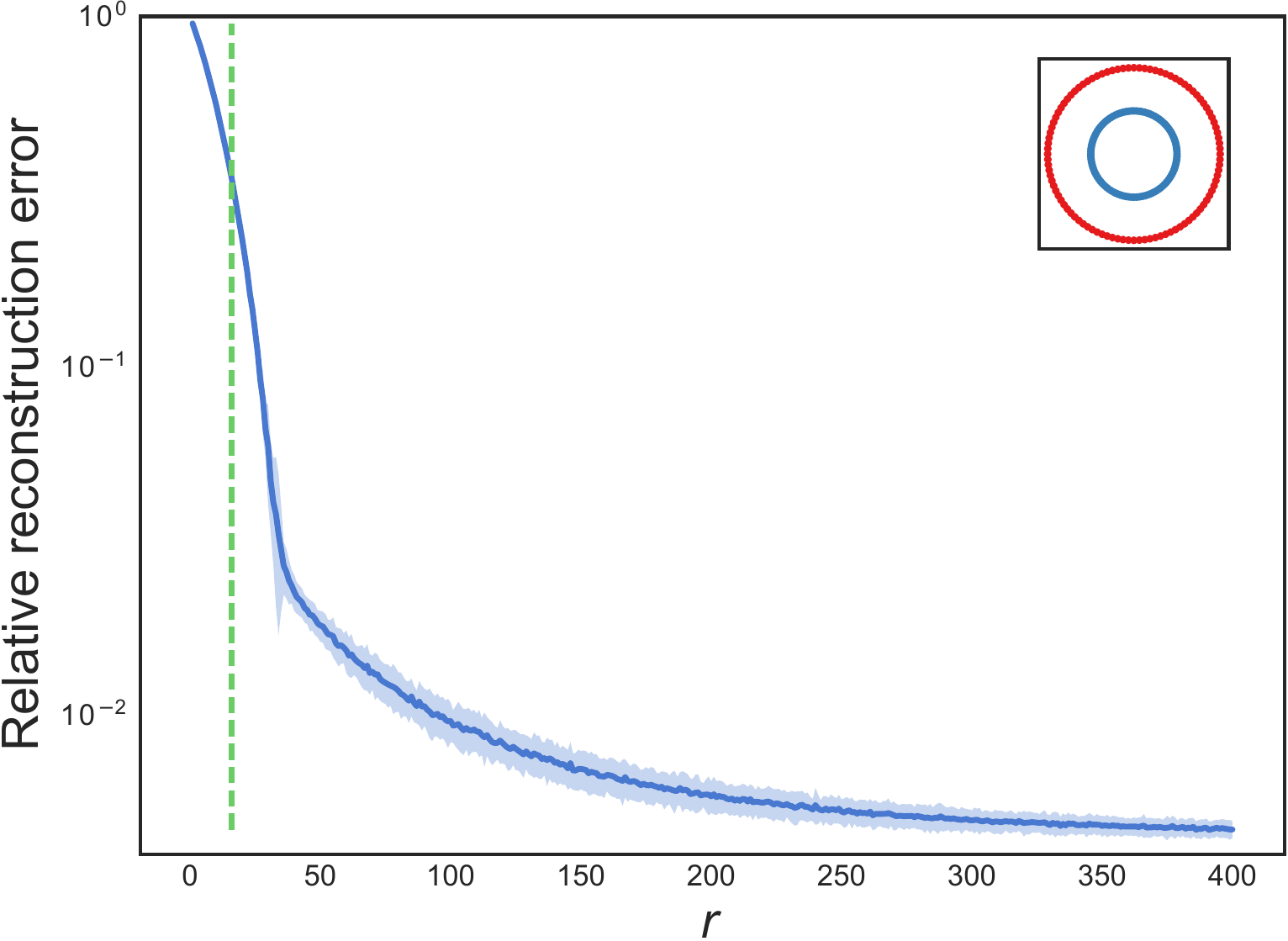}
			\put (15, 8) {\scriptsize$K$}
		\end{overpic}
		\hspace{3em}
		\begin{overpic}[width=.35\linewidth]{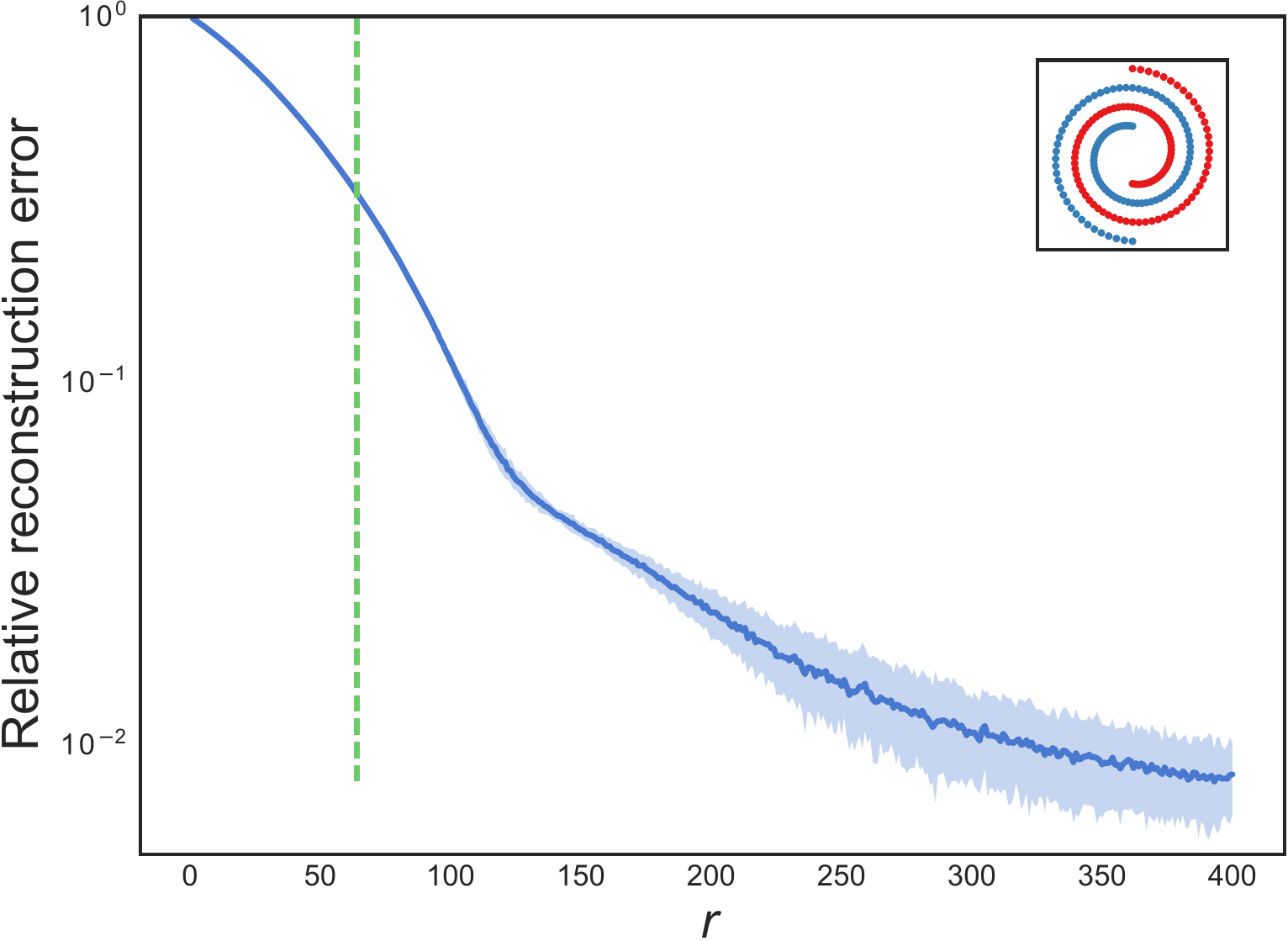}
			\put (25, 8) {\scriptsize$K$}
		\end{overpic}
		\hfill
	}
	
	\caption{We empirically study the cp-rank of $\mat{Q}_*$. As a proxy of the exact nonnegative decomposition, we compute the rank-$r$ symmetric NMF $\mat{Q}_* \approx \transpose{\mat{Y}_+} \mat{Y}_+$ for different values of $r$.
		We show the mean plus/minus two standard deviations of the relative error $\norm{\mat{Q}_* - \transpose{\mat{Y}_+} \mat{Y}_+}{F} / \norm{\mat{Q}_*}{F}$ computed from 50 different SNMFs for each $r$ (their differences stem from the random initialization). Both datasets have 200 points. Clearly, setting $r=K$ is not enough to properly reconstruct $\mat{Q}_*$.
	}
	\label{fig:completely_positive_snmf}
\end{figure}

Our  experiments with the non-convex algorithm in \cref{sec:burer-monteiro} lead to similar conclusions as those with SNMF, see \cref{fig:completely_positive_burer-monteiro}. Setting $r=K$, leads to a poor approximation of $\mat{Q}_*$ and, as observed by \citet{Kulis2007}, to hard-clustering. Setting $r \gg K $ leads to much improved reconstructions, at the expense of speed.

\begin{figure*}
	\centering
	\includegraphics[width=.85\linewidth]{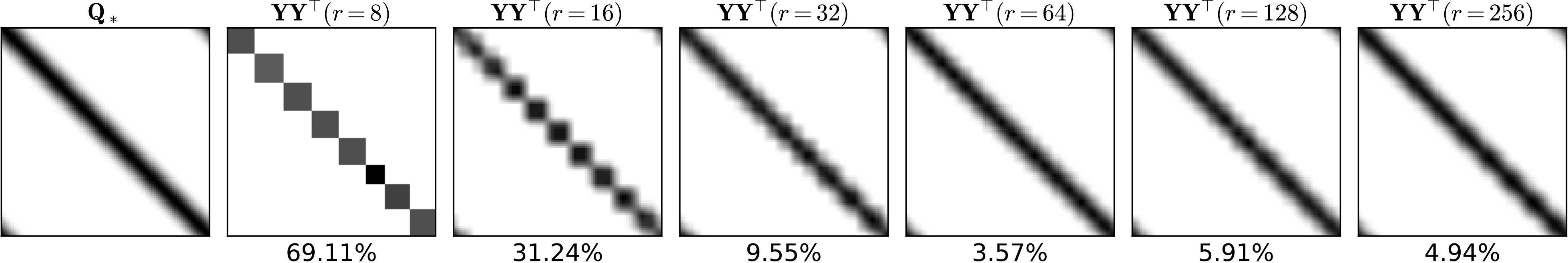}

	\includegraphics[width=.85\linewidth]{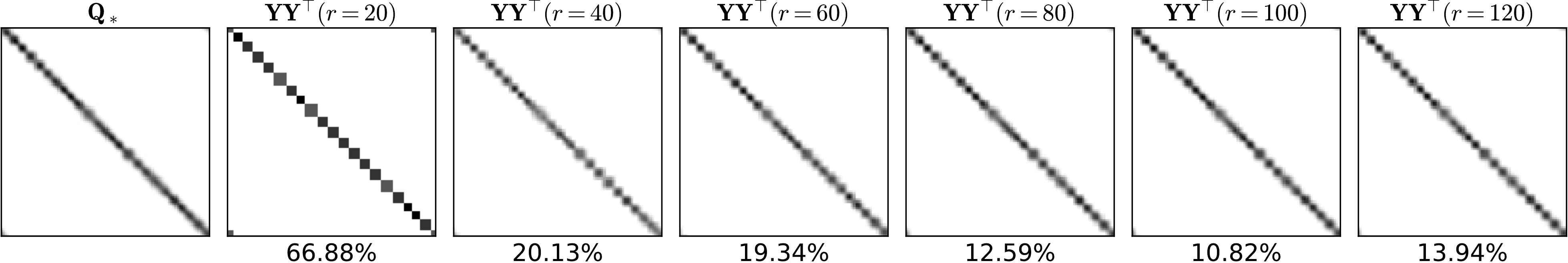}

	\caption{Comparison of the results obtained with a standard SDP solver (first column) and with a low-rank non-convex approach (remaining columns, $r$ denotes the rank of the obtained solution). \textbf{(Top)} Dataset in \cref{fig:circles_eigendecomposition}; we set $K=8$ in all cases. \textbf{(Bottom)} Dataset in \cref{fig:embedding_real_teapot}; we set $K=20$ in all cases. In each case, we also display the relative error between the matrix $\transpose{\mat{Y}} \mat{Y}$ and $\mat{Q}_*$. Interestingly, setting $r=K$ produces hard clustering solutions (see the block diagonal structure of the matrices on the second column), while increasing $r$ produces ``softer'' solutions. This suggests that the cp-rank of $\mat{Q}_*$ is (much) greater than $K$.}
	\label{fig:completely_positive_burer-monteiro}
\end{figure*}

\section{A fast and convex algorithm for NOMAD}

The Burer-Monteiro solver forgoes convexity in favor of speed. However, as discussed in the previous section, this conversion carries theoretical and practical difficulties that are not easily overcome.
In this section, we propose an algorithm for NOMAD that is fast and yet convex.

\subsection{Augmented Lagrangian formulation}
First, we redefine the variables in NOMAD by setting $\mat{P} = \mat{Q} - \mat{E}_n$, where $\mat{E}_n = \tfrac{1}{n} \vect{1} \transpose{\vect{1}}$. Then,
\begin{equation}
	\max_{\mat{P}}
	\traceone{\mat{D} \mat{P}}
	\quad\text{s.t.}\quad
	\mat{P} \vect{1} = \vect{0} ,
    \quad
	\traceone{\mat{P}} = K - 1 ,
    \quad
	\mat{P} \succeq \mat{0} ,
    \quad
	\mat{P} + \mat{E}_n \geq \mat{0} .
	%
	\label[problem]{eq:sdp_kmeans_minus_constant}
\end{equation}
As usual in the optimization literature, we handle this constraint with an augmented Lagrangian method. The augmented Lagrangian of \cref{eq:sdp_kmeans_minus_constant} with respect to the constraint $\mat{P} + \mat{E}_n \geq \mat{0}$ is
\begin{equation}
	g(\mat{P}, \mat{\Gamma}) =
    - \traceone{\mat{D} \mat{P}}
    + \traceone{\mat{\Gamma} ( \mat{P} + \mat{E}_n )}
    + \tfrac{\gamma}{2} \norm{\left[ \mat{P} + \mat{E}_n \right]_{-} }{F}^2 ,
	\label{eq:sdp_kmeans_minus_constant_lagrangian}
\end{equation}
where $\mat{\Gamma} \geq \mat{0}$ is the associated Lagrange multiplier, and $[\cdot]_{-} = \min(\cdot, 0)$ is the projection operator onto the negative orthant.
We can then pose \cref{eq:sdp_kmeans_minus_constant} as
\begin{equation}
	\min_{\mat{P}}
    \max_{\mat{\Gamma} \geq 0}
	g(\mat{P}, \mat{\Gamma})
    \quad\text{s.t.}\quad
	\mat{P} \vect{1} = \vect{0} ,
    \quad
	\traceone{\mat{P}} = K - 1 ,
    \quad
	\mat{P} \succeq \mat{0} .
    \label{eq:maxmin_sdp-km}
\end{equation}
We solve it using the method of multipliers, i.e.,
\begin{subequations}
\begin{align}
	\mat{P}_{t+1} &= \argmin_{\mat{P}}
	g(\mat{P}, \mat{\Gamma}_{t})
	\quad\text{s.t.}\quad
    	\mat{P} \vect{1} = \vect{0} ,
    \quad
	\traceone{\mat{P}} = K - 1 ,
    \quad
	\mat{P} \succeq \mat{0} ,
	\label[problem]{eq:cgm_sdp-kmeans_inner} \\
    \mat{\Gamma}_{t+1} &= \left[ \mat{\Gamma}_t + \tau ( \mat{P}_{t+1} + \mat{E}_n ) \right]_{-} .
    \label[problem]{eq:cgm_sdp-kmeans_outer}
\end{align}
\end{subequations}

\subsection{A conditional gradient method for SDPs with an orthogonality constraint}

In this section, we introduce a very efficient algorithm to solve
\begin{equation}
	\max_{\mat{Z}}
    f(\mat{Z})
    \quad\text{s.t.}\quad
    \mat{Z} \succeq 0 ,\
    \traceone{\mat{Z}} = s ,\
    \mat{Z} \vect{b} = \vect{0} .
    \label[problem]{eq:generic_sdp_ortho_intro}
\end{equation}
of which \cref{eq:cgm_sdp-kmeans_inner} is an instance.

To this end we modify an algorithm to efficiently solve the SDP \citep{Hazan2008sdp}
\begin{equation}
	\max_{\mat{Z}}
    f(\mat{Z})
    \quad\text{s.t.}\quad
    \mat{Z} \succeq 0 ,\ \traceone{\mat{Z}} = s ,
    \label[problem]{eq:generic_sdp}
\end{equation}
where function $f$ is differentiable and concave.
The iterative algorithm consists, at each iteration $t = 0 \dots$, of the following steps:
\begin{enumerate}
    \item Let $\vect{v}_{t}$ be the largest algebraic eigenvector of $\nabla f(\mat{Z}_t)$.
    \item $\mat{Z}_{t+1} = (1 - \alpha) \mat{Z}_{t} + \alpha s \vect{v}_{t} \transpose{\vect{v}_{t}}$ with $\alpha = 2 / (t + 2)$.
\end{enumerate}
This algorithm is an instance of the Frank-Wolfe/conditional-gradient algorithm \citep{FrankWolfe}. As such it provides a solution without performing any projections. First, $\mat{Z}_{t+1}$ is a non-negative linear combination of two positive semidefinite matrices, and is thus positive semidefinite itself. Second, the iterations maintain the invariant $\traceone{\mat{Z}_{t}} = s$ as $\traceone{\mat{Z}_{t+1}} = (1 - \alpha) \traceone{\mat{Z}_{t}} + \alpha s \traceone{\vect{v}_{t} \transpose{\vect{v}_{t}}} = (1 - \alpha) \traceone{\mat{Z}_{t}} + \alpha s$.

We now show how to extend this algorithm to handle an orthogonality constraint.
Let $\set{P}_s$ be the convex cone of positive semidefinite matrices with trace $s$ that are orthogonal to a given vector $\vect{b}$, i.e.,
\begin{equation}
	\set{P}_s = \left\{ \mat{Z} \succeq 0 ,\ \traceone{\mat{Z}} = s ,\ \mat{Z} \vect{b} = \vect{0} \right\} .
    \label{eq:Ps_set}
\end{equation}
Notice that setting $\vect{b} = \vect{1}$ yields the constraints of \cref{eq:cgm_sdp-kmeans_inner}.
We seek to solve
\begin{equation}
	\max_{\mat{Z}}
    f(\mat{Z})
	\quad\text{s.t.}\quad
    \mat{Z} \in \set{P}_s .
	%
	\label[problem]{eq:generic_sdp_ortho}
\end{equation}
Fortunately, we can push the constraint $\mat{Z} \vect{b} = \vect{0}$ into the eigenvector computation. We begin by noticing that the final solution is a weighted sum of the matrices $\vect{v}_{t} \transpose{\vect{v}_{t}}$. It then suffices to require that, for every $t$, $\vect{v}_{t} \transpose{\vect{v}_{t}} \vect{b} = \vect{0}$, which reduces to $\transpose{\vect{v}_{t}} \vect{b} = \vect{0}$.
This naturally yields a new iterative method, summarized in \cref{algo:cgm_generic_sdp_ortho}.
This algorithm has the same performance guarantee as \citeauthor{Hazan2008sdp}'s \citeyearpar{Hazan2008sdp}, given by the following proposition, which we prove in \cref{sec:cgm_proofs}.
\begin{proposition}
 	Let $\mat{X}, \mat{Z} \in \set{P}_s$ and $\mat{Y} = \mat{X} + \alpha (\mat{Z} - \mat{X})$ and $\alpha \in \Real$.
    The curvature constant of $f$ is
	\begin{equation}
	C_f \defeq
    \sup_{\mat{X}, \mat{Z}, \alpha}
    \tfrac{1}{\alpha^2}
        [ f(\mat{X}) - f(\mat{Y}) + \tracetwo{(\mat{Y} - \mat{X})}{\nabla f(\mat{X})} ] .
        \label{eq:curvature_constant}
	\end{equation}
	Let $\mat{Z}^{\star}$ be the solution to \cref{eq:generic_sdp_ortho}.
	The iterates $\mat{Z}_t$ of \cref{algo:cgm_generic_sdp_ortho} satisfy for all $t > 1$
    \begin{equation}
    	f(\mat{Z}^{\star}) - f(\mat{Z}_t) \leq  \tfrac{8 C_f}{t+2} .
    \end{equation}
	\label{theo:accuracy}
\end{proposition}

\begin{algorithm2e}[t]
	\SetNoFillComment
	\SetKwInOut{Input}{input}
	\SetKwInOut{Output}{output}

	\begin{small}
    \Input{function $f$ to minimize, scale parameter $s$.}
	\Output{solution $\mat{Z}_{t+1} \in \set{P}_s$ to \cref{eq:generic_sdp}.}
    
	Initialize
    $\mat{Z}_0 = \mat{0}$\;
    \For{$t = 0, \dots, \infty$}{
    	Let $\vect{v}_{}$ be the largest algebraic eigenvector of $\nabla f(\mat{Z})$ such that $\transpose{\vect{v}} \vect{b} = 0$\;
    	$\alpha \gets 2 / (t + 2)$\;
    	$\mat{Z}_{t+1} \gets (1 - \alpha) \mat{Z}_t + \alpha s \vect{v} \transpose{\vect{v}}$\;
    	\lIf{converged}{
    	break
        }
    }
    \end{small}
	
	\caption{Conditional gradient algorithm for SDPs with an orthogonality constraint}
	\label{algo:cgm_generic_sdp_ortho}
\end{algorithm2e}

\subsection{A conditional gradient algorithm for NOMAD}

\cref{algo:cgm_sdp-km} summarizes the proposed method of multipliers, see iterations \labelcref{eq:cgm_sdp-kmeans_inner,eq:cgm_sdp-kmeans_outer}, to solve \cref{eq:sdp_kmeans_minus_constant}. The inner problem \labelcref{eq:cgm_sdp-kmeans_inner} is solved using \cref{algo:cgm_generic_sdp_ortho}. A few remarks are in order:
\begin{itemize}[nosep]
	\item When using the method of multipliers, it is often not necessary (nor desirable) to solve the inner problem to a high precision \citep{Goldstein2009}. In our implementation we set $N_{\text{inner}} = 10$.
    \item There is no need to need for a highly accurate eigenvector computation \citep{Hazan2008sdp}. We use the Lanczos algorithm and set its accuracy to $(t + 1)^{-1}$.
    \item \cref{algo:cgm_generic_sdp_ortho} solves a maximization problem and requires the eigenvector with the largest algebraic eigenvalue. To solve the minimization problem \labelcref{eq:cgm_sdp-kmeans_inner}, we simply compute the eigenvector with the smallest algebraic eigenvalue \citep{Jaggi2013}.
	\item As $\vect{b} = \vect{1}$, we can enforce the orthogonality constraint $\transpose{\vect{v}_t} \vect{1} = 0$ by computing the maximum eigenvalue of $\mat{A} = (\mat{I} - \tfrac{1}{n} \vect{1} \transpose{\vect{1}}) \nabla g(\mat{P}, \mat{\Gamma}) (\mat{I} - \tfrac{1}{n} \vect{1} \transpose{\vect{1}})$. This operation can be carried out very efficiently.
\end{itemize}

\begin{algorithm2e}[t]
	\SetNoFillComment
	\SetKwInOut{Input}{input}
	\SetKwInOut{Output}{output}

	\begin{small}
    \Input{matrix $\mat{D}$, scale parameter $k$.}
	\Output{solution $\mat{Q}$ to NOMAD.}
    
	Initialize $\mat{P}_0 = \mat{0}$;\quad
    $\mat{\Gamma} \gets \mat{0}$;\quad
    $\gamma = 1$\;
    \For{$t = 1, \dots, \infty$}{
		\For{$t_\text{inner} = 1, \dots, N_{\text{inner}}$}{
        	Let $
				\nabla g(\mat{P}, \mat{\Gamma}) =
    			- \mat{D}
    			+ \mat{\Gamma}
    			+ \gamma \left[ \mat{P} + \mat{E}_n \right]_{-}
				$\;
            Let $\mat{A} = (\mat{I} - \tfrac{1}{n} \vect{1} \transpose{\vect{1}}) \nabla g(\mat{P}, \mat{\Gamma}) (\mat{I} - \tfrac{1}{n} \vect{1} \transpose{\vect{1}})$\;
    		Let $\vect{v}_{}$ be the smallest algebraic eigenvector of $\mat{A}$, such that $\transpose{\vect{v}} \vect{1} = 0$\;
        	$\mat{H} \gets (K - 1) \vect{v} \transpose{\vect{v}}$\;
        	$\alpha \gets 2 / (t + t_\text{inner} + 2)$\;
			$\mat{P} \gets (1 - \alpha) \mat{P} + \alpha \mat{H}$\;
		}
        $\mat{\Gamma} \gets \left[ \mat{\Gamma} + \tau \left( \mat{P} + \mat{E}_n \right) \right]_{-}$\;
        \lIf{converged}{
        	break
        }
    }
	$\mat{Q} \gets \mat{P} + \mat{E}_n$
    \end{small}
	
	\caption{Conditional gradient algorithm for NOMAD}
	\label{algo:cgm_sdp-km}
\end{algorithm2e}

\paragraph{Complexity.}
The complexity of \cref{algo:cgm_generic_sdp_ortho} is similar to that of \citeauthor{Hazan2008sdp}'s \citeyearpar{Hazan2008sdp}, plus an additional factor to compute $\mat{A}$.
From \cref{theo:accuracy}, \cref{algo:cgm_generic_sdp_ortho} yields a solution with accuracy $\varepsilon$, i.e., $f(\mat{Z}_t) \geq f(\mat{Z}^{\star}) - \varepsilon$, in $\tfrac{4 C_f}{\varepsilon} - 1$ iterations. Computing $\nabla g(\mat{P}, \mat{\Gamma})$, $\mat{A}$, and $\mat{H}_t$ require $n^2$ operations.
Let $T_{\text{EIG}}$ be the number of iterations of the eigensolver, each iteration taking $O(n^2)$ operations. Additional operations require $O(n)$ time. Then, the overall complexity of \cref{algo:cgm_generic_sdp_ortho} is
\begin{equation}
	O \left( \tfrac{C_f}{\varepsilon} \left[ n + n^2 + n^2 T_{\text{EIG}} \right] \right) .
\end{equation}
For the Lanczos algorithm, and our accuracy setting of $(t + 1)^{-1}$, we have $T_{\text{EIG}} = O((t + 1) \log n)$.
In this case, the complexity per iteration is $O(n^2 \log n)$. As a comparison, standard SDP solvers have a complexity of $O(n^3)$ per iteration. These solvers also involve significant memory usage, while our algorithm has an optimal space complexity of $O(n^2)$.

\subsection{Experimental analysis}

Throughout the iterations of \cref{algo:cgm_sdp-km}, $\mat{P} \in \set{P}_{k-1}$, see \cref{eq:Ps_set}. Thus, we only need to keep track of the constraint $\mat{Q} = \mat{P} + \mat{E}_n \geq \mat{0}$ and of the value of the objective $\traceone{\mat{D} \mat{P}}$.

We illustrate with two typical examples the empirical convergence of these values in \cref{fig:cvx_vs_cgm_convergence}.
the convergence the objective value is clearly superlinear, while we observe a linear convergence for the nonnegativity constraint. Accelerating the latter rate is an interesting line of future research.

We can see in \cref{fig:cvx_vs_cgm_convergence} that standard solvers enforce the nonnegativity constraint more accurately. However, they do not exactly enforce $\mat{P} \in \set{P}_{k-1}$. There is a trade-off between what can be enforced up to which precision, making the solutions sometimes not exactly comparable.

\begin{figure}
	\centering
	\includegraphics[width=.5\linewidth]{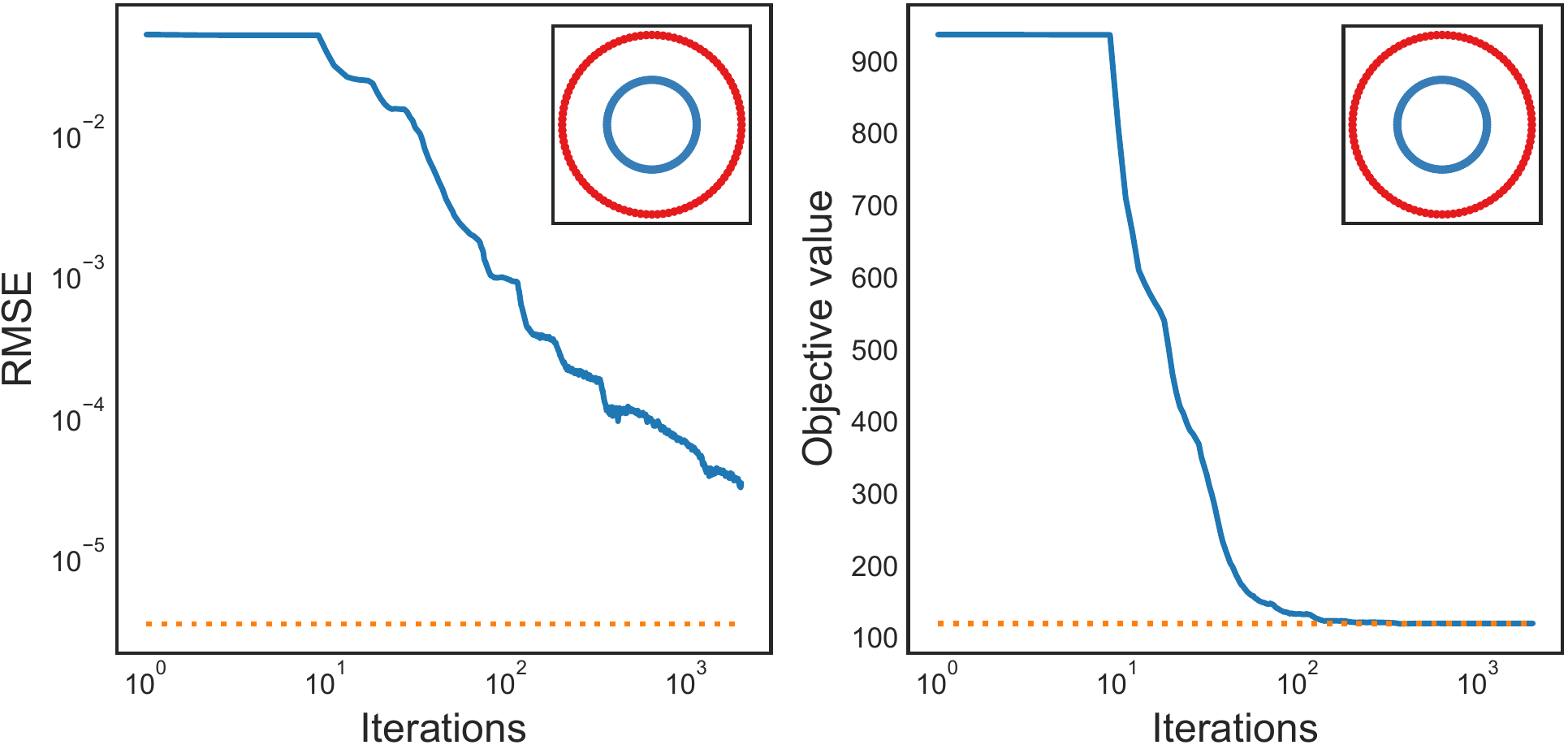}%
    \includegraphics[width=.5\linewidth]{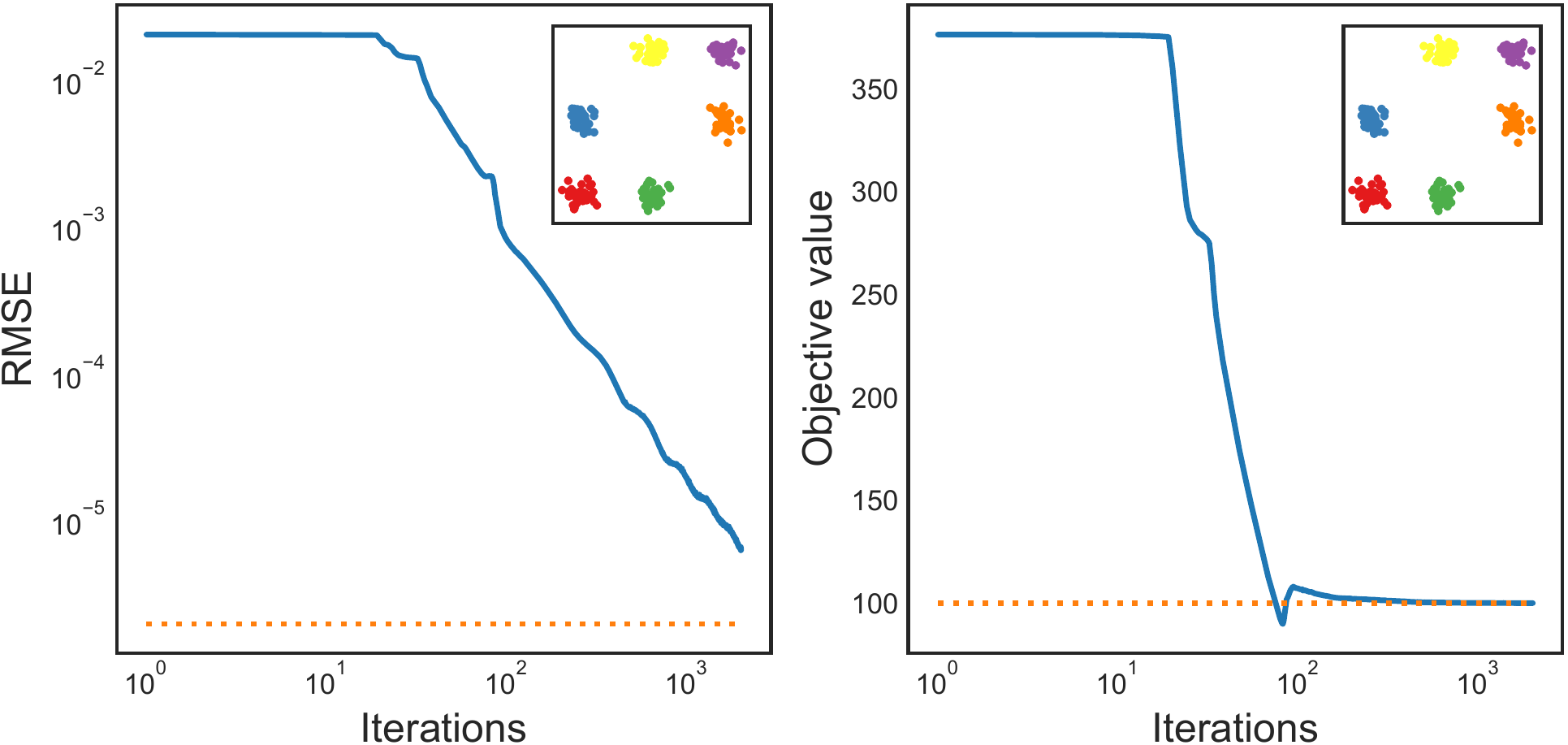}%
		
	\caption{Prototypical examples of the behavior of the proposed conditional gradient NOMAD solver as its iterations progress. On the left plots, we show the RMSE of $[\mat{Q}]_{-} = [\mat{P} + \mat{E}_n]_{-}$, with the average computed over its non-zero entries. After about 10 iterations, the RMSE drops linearly, as usual for the method of multipliers. On the right plots, we display the objective value $\traceone{\mat{D} \mat{P}}$, which usually converges in a few hundred iterations. In each case, as a reference, we show in orange the values returned by the standard SDP solver. The proposed algorithm enforces the nonnegativity constraint in NOMAD less accurately (although accurate enough for practical purposes), while exactly enforcing all the other constraints.}
	\label{fig:cvx_vs_cgm_convergence}
\end{figure}

We show the suitability of the proposed NOMAD solver in \cref{fig:cvx_vs_cgm}. In the vast majority of cases the solutions are the same. While the proposed method enforces the nonnegativity constraint less accurately than the standard solver, it enforces all the other constraints exactly. This is why in the teapot example, bottom left of \cref{fig:cvx_vs_cgm}, the solution of the proposed method looks less jagged than the one of the standard solver: the constraint $\mat{Q} \vect{1} = \vect{1}$ is more accurately enforced, resulting in a more ``circulant'' representation.

\begin{figure}
	\centering
    \begin{tabu} to \textwidth {@{\hspace{0pt}}X[c,m] X[c,m] @{\hspace{0pt}}}
		\includegraphics[width=\linewidth]{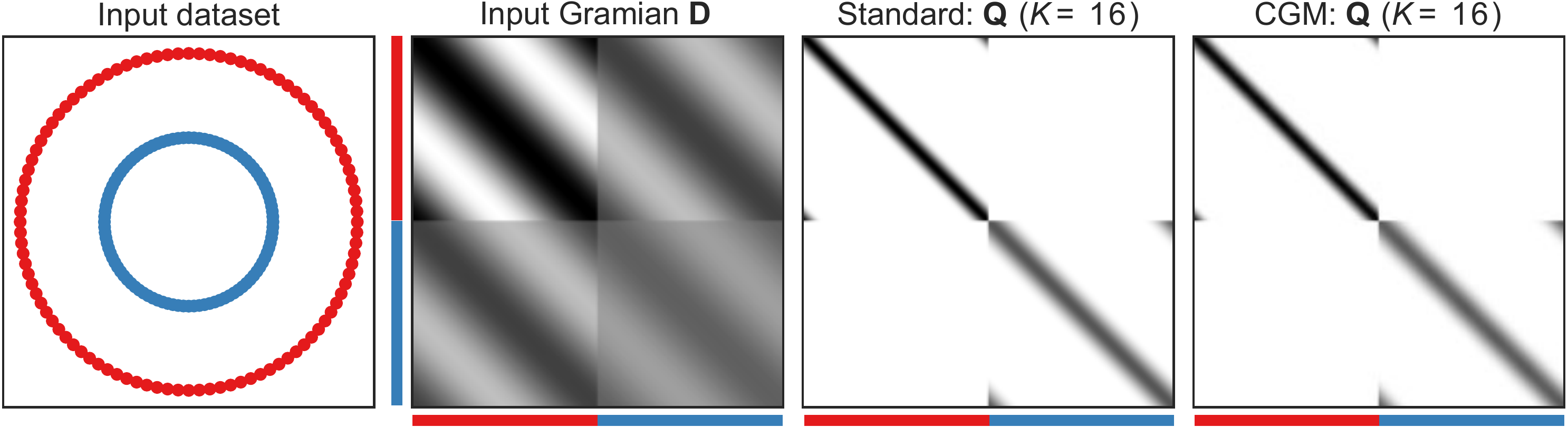} &
        \includegraphics[width=\linewidth]{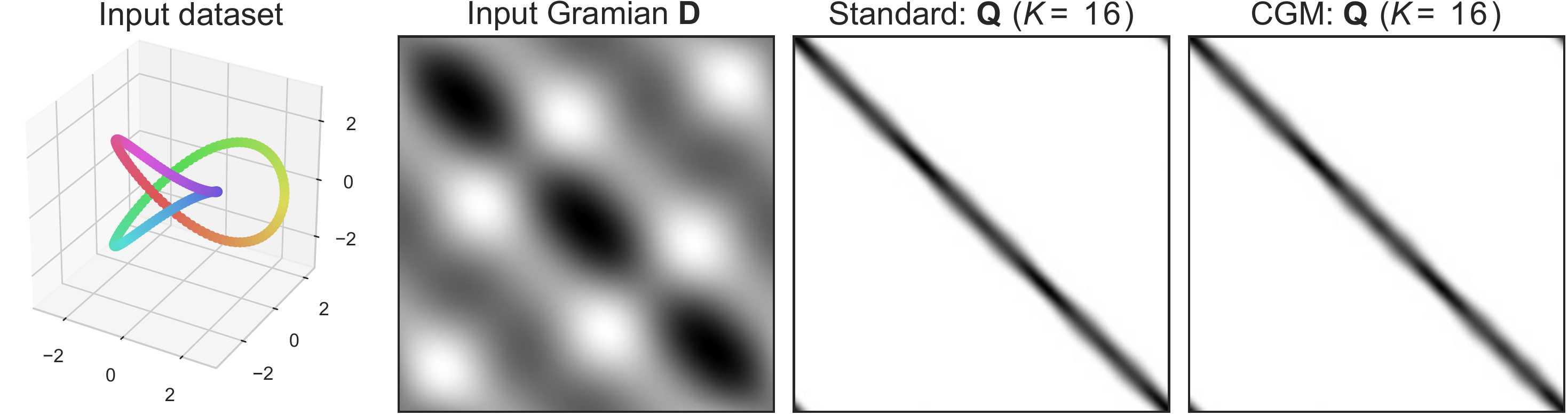} \\

		\addlinespace[1em]

    	\includegraphics[width=\linewidth]{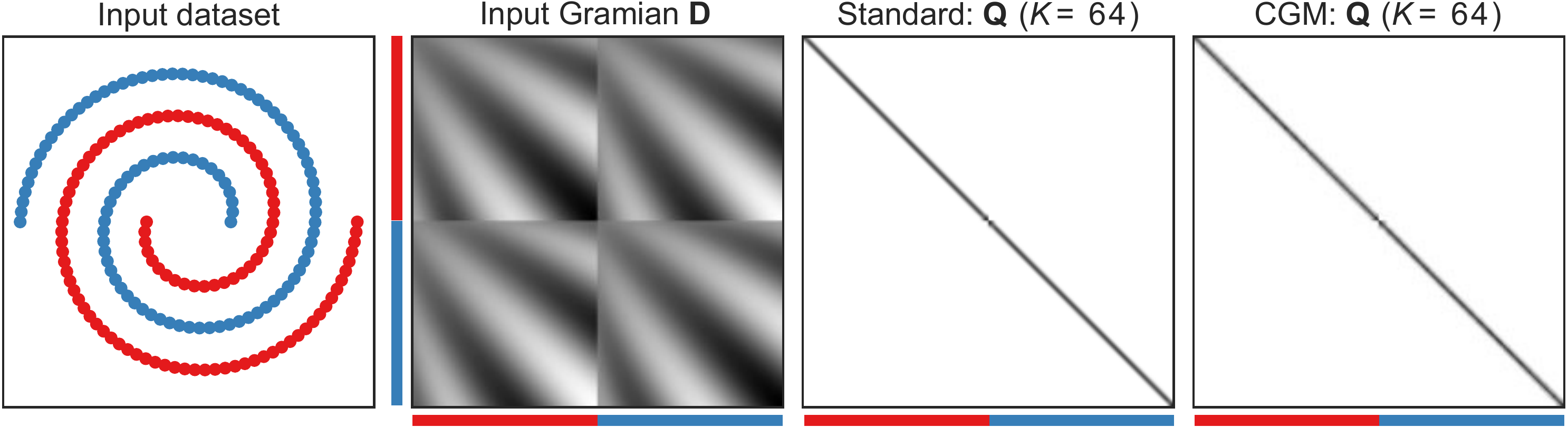} &
    	\includegraphics[width=\linewidth]{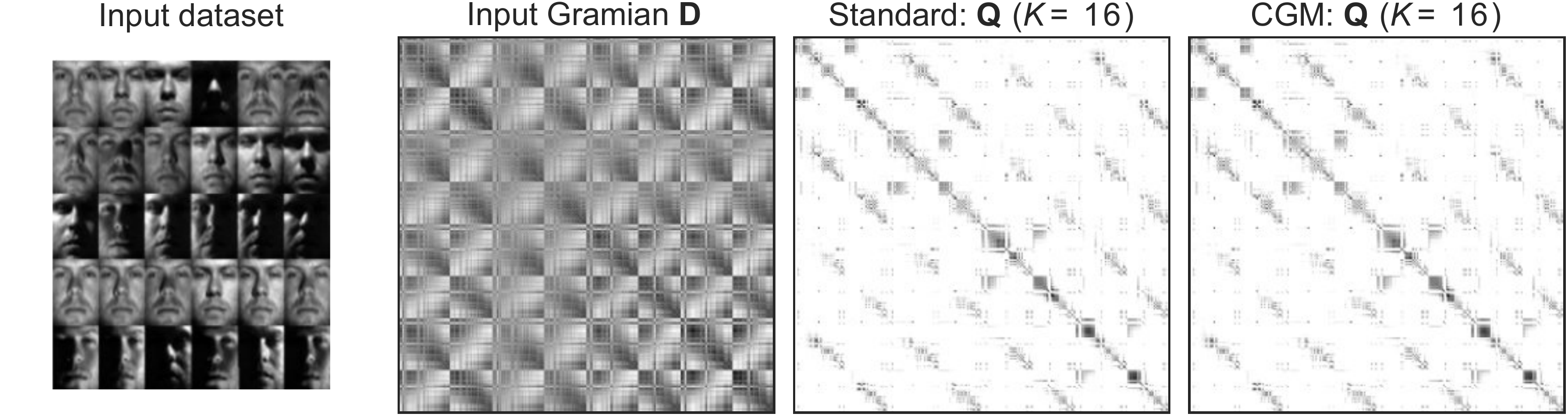} \\
        
        \addlinespace[1em]

    	\includegraphics[width=\linewidth]{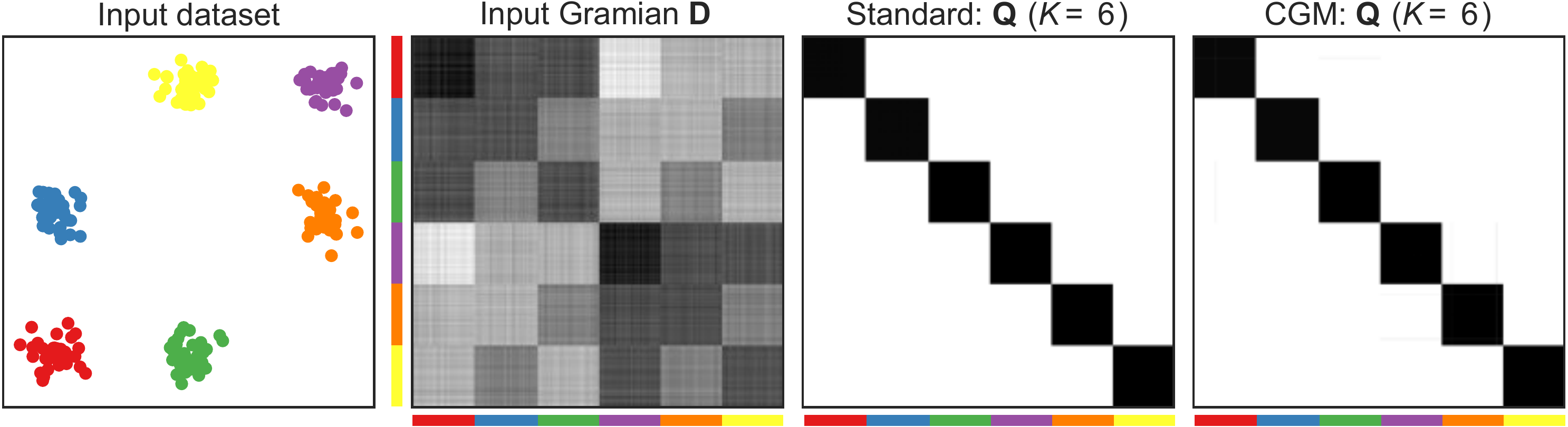} &
    	\includegraphics[width=\linewidth]{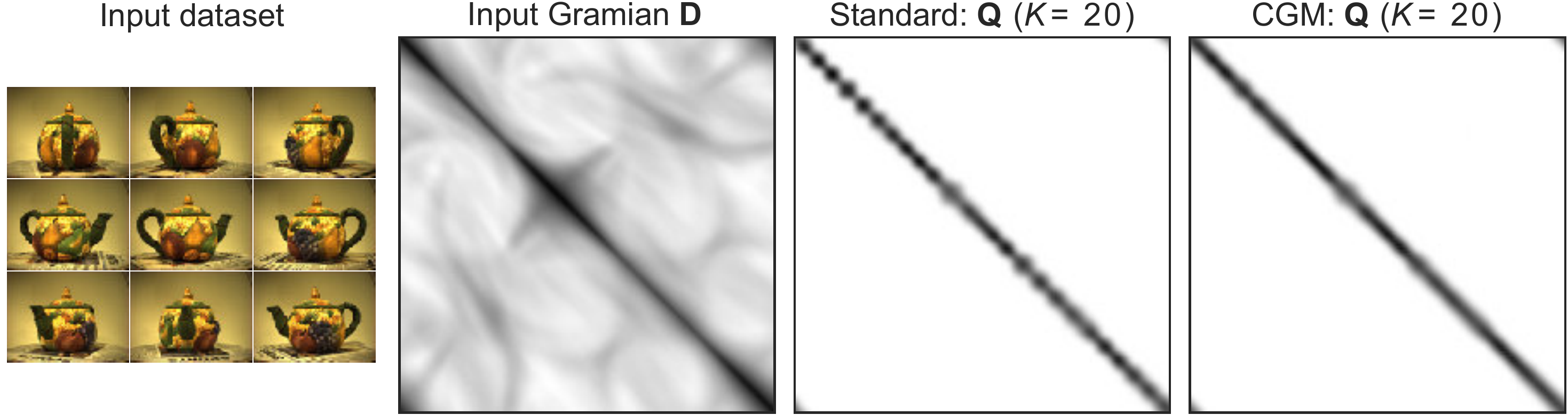} \\
	\end{tabu}
    
	\caption{Comparison of the standard SDP solver with the proposed conditional gradient solver (CGM) for NOMAD on different datasets. In most cases, the results are practically indistinguishable while being delivered much faster.}
	\label{fig:cvx_vs_cgm}
\end{figure}

In \cref{fig:mnist_timing}, we present the speed comparison of computing NOMAD with three different methods: two state-of-the-art SDP solvers, SCS~\citep{SCSsolver} and SDPNAL+~\citep{sdpnalplus}, the low-rank Burer-Monteiro solver (discussed in \cref{sec:BM}), and the proposed conditional gradient method.
The Burer-Monteiro method is the fastest. Keep in mind that the latter does not guarantee convergence to the global optimal solution; this is particularly true specially in its fastest setting, i.e., by keeping $r$ relatively small, see \cref{sec:BM}.
Among solvers that solve a convex problem, for very small problems (up to 250 points), standard SDP solvers are the fastest. For larger problems the proposed solver is significantly faster. It is important to point out that, in theory, the speed difference grows significantly larger. This is hard to show in practice as standard solvers either run out of memory very quickly (SCS) or are implemented to time out for big instances (SDPNAL+); the proposed solver has a much more efficient use of memory.

We highlight the extended computational capabilities of the proposed conditional gradient method with an example that cannot be handled by standard SDP solvers. We use as input the $9603 \times 9603$ Gramian formed by all (vectorized) images of the digit zero in MNIST. The proposed algorithm is able to compute a solution to NOMAD with ease for a problem size about 100 times larger than the upper size limit for standard solvers. In the 2D embedding of the solution (see \cref{sec:manifold} for details about its computation), shown in \cref{fig:mnist_n9603_k128_embedding}, we can clearly see that the images are organized by their intrinsic characteristics (elongation and rotation).

\begin{figure}
	\centering
	\includegraphics[width=.45\linewidth]{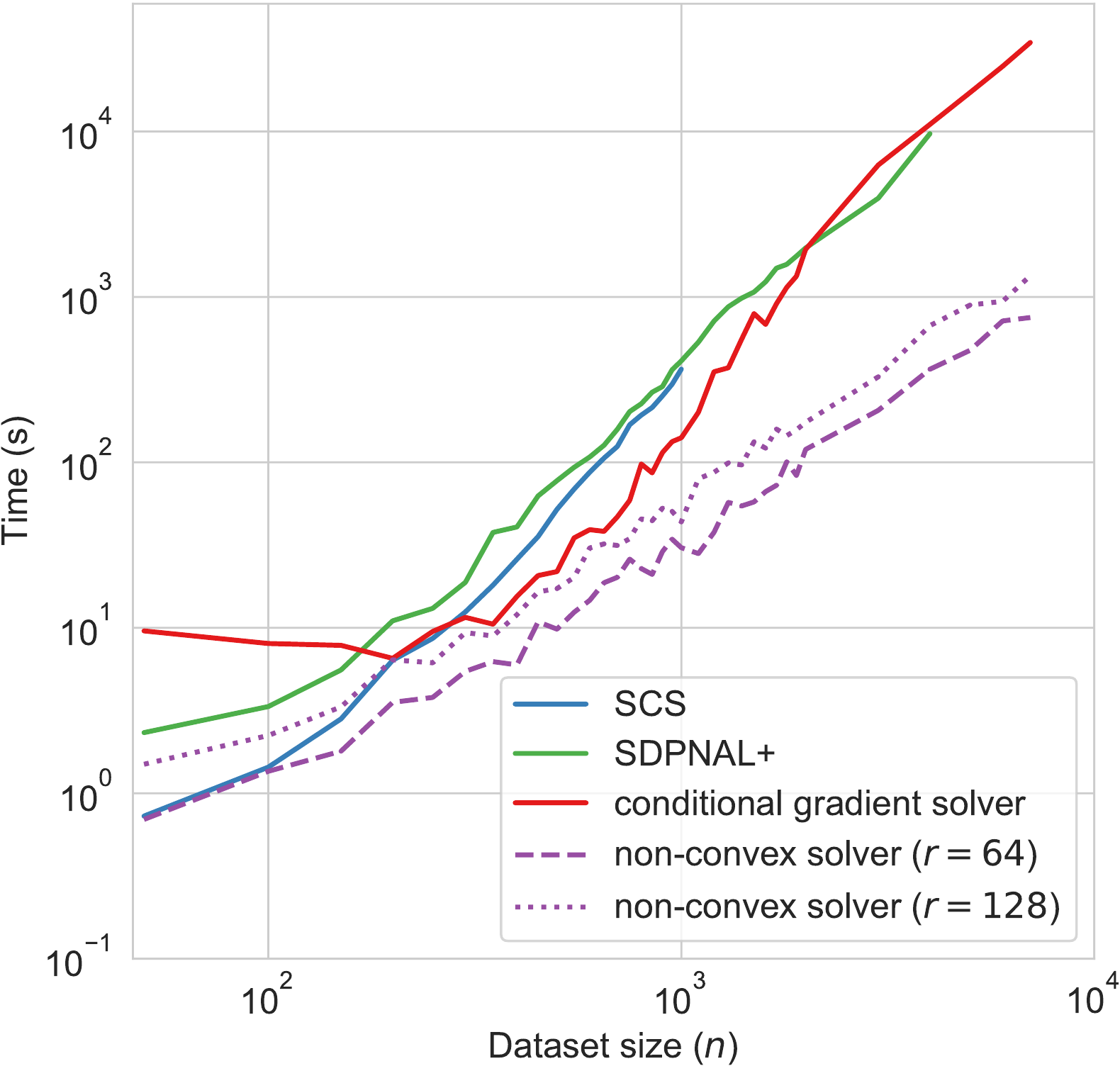}
		
	\caption{Running time comparison (smaller is better) of different NOMAD solvers for $K=16$ (SCS~\citep{SCSsolver} and SDPNAL+~\citep{sdpnalplus} are written in a highly optimized C/C++ code, while we use our non-optimized Python code for the others). The non-convex solver is much faster than the convex ones. Unfortunately, it may yield different results, see \cref{fig:completely_positive_burer-monteiro}, and may not converge to the global maximum.
    The conditional gradient algorithm proposed in this paper is much faster than SCS and SDPNAL+ (about three times faster for $n=10^3$) but guarantees converging to the global optimum.
    Additionally, the proposed algorithm handles large problems seamlessly: in our desktop with 128GB of RAM, SCS (running under CVXPY) runs out of memory with instances larger than $n=1200$) while SDPNAL+ times out before converging for instances larger than $n=4000$.
    }
	\label{fig:mnist_timing}
\end{figure}

\begin{figure}
	\centering
	\includegraphics[width=.6\linewidth]{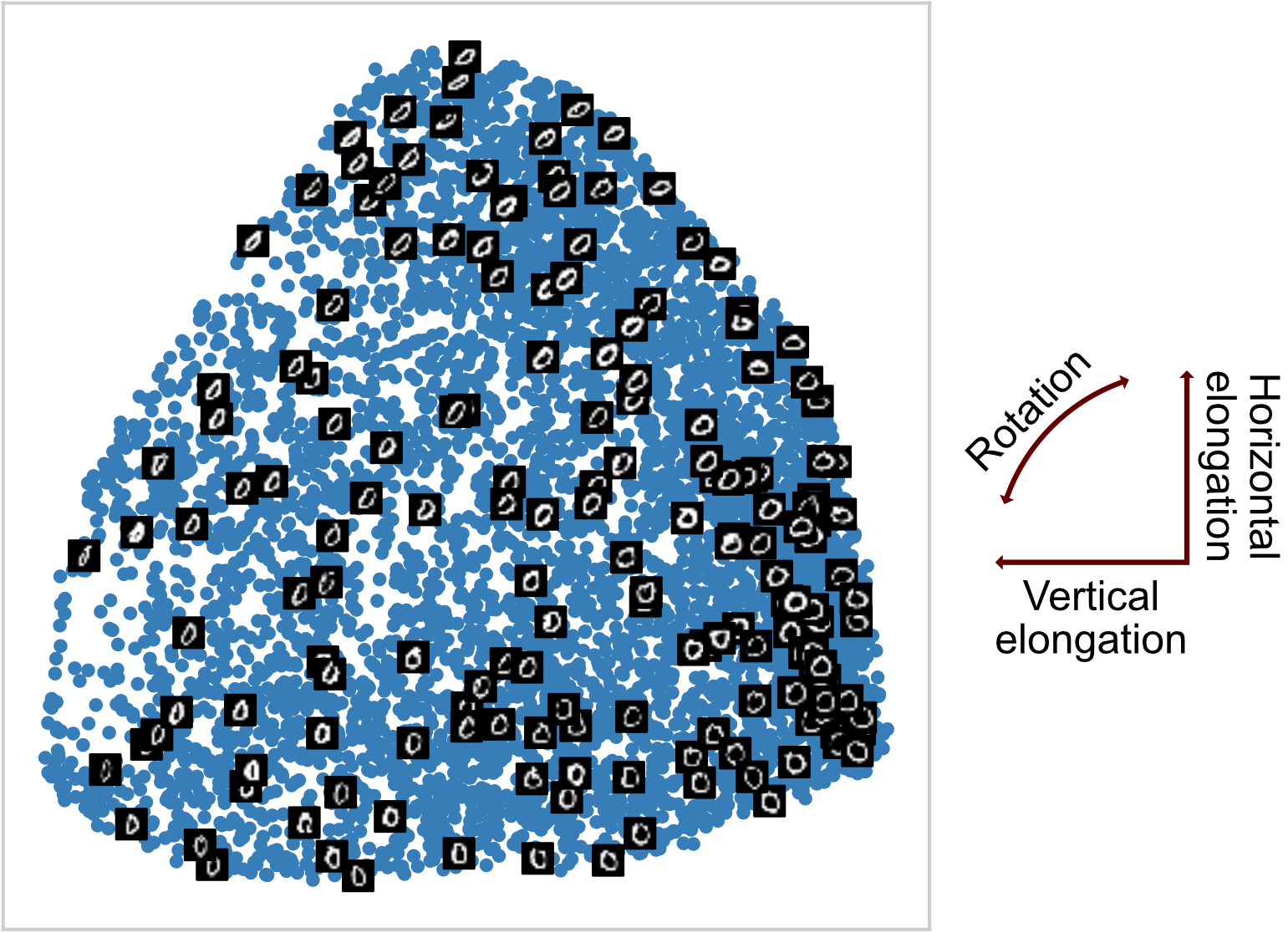}
		
	\caption{We show the 2D embedding of the digit 0 in MNIST, computed in the same fashion as in \cref{fig:embedding_real}. In this case, we use all 9603 images of the digit and obtain a $9603 \times 9603$ matrix. We compute the solution of NOMAD with $K=128$ using the proposed conditional gradient method (\cref{algo:cgm_sdp-km}). In contrast, traditional SDP-solvers can only handle dense matrices approximately 100 times smaller.
    As in \cref{fig:embedding_real}, the data gets organized according to different visual characteristics of the hand-written digit (e.g., orientation and elongation).
    }
	\label{fig:mnist_n9603_k128_embedding}
\end{figure}

\section{Conclusions}
\label{sec:conclusions}

In this work, we showed that NOMAD can learn multiple low-dimensional data manifolds in high-dimensional spaces. An SDP instance, it is convex and can be solved in polynomial-time. Unlike most manifold learning algorithms, the user does not need to select/use a kernel and no nearest-neighbors searches are involved.

We also studied the computational performance of NOMAD.
We first focused on a non-convex Burer-Monteiro-style algorithm and performed both theoretical and empirical analysis.
Finally, we presented a new algorithm for NOMAD based on the conditional gradient method. The proposed algorithm is convex and yet efficient. This algorithm allows us, for the first time, to analyze the behavior of NOMAD on large datasets.

\paragraph{Related and future work.}
It has not escaped our attention that NOMAD can be considered as an instance of kernel alignment \citep{Cristianini2002}.
In supervised setting, kernel alignment has been previously formulated as an SDP 
\citep[e.g.,][]{Lanckriet2004,Cortes2012_centeredalignment}. Even beyond the distinction between the supervised and unsupervised scenarios, this body of work differs significantly from NOMAD. Its goal is to optimally combine pre-computed kernel matrices, whereas NOMAD learns such a matrix from scratch.
Nonetheless, we find this connection with kernel learning very promising and plan to investigate it further in the future.

\section*{Acknowledgments}
We thank Afonso Bandeira, Alexander Genkin, Victor Minden, and Cengiz Pehlevan for helpful discussions.

\newpage
\bibliography{sdpkm,cgm}

\newpage
\appendix

\section{Relationship with $K$-means}
\label{sec:kmeans}

$K$-means seeks to cluster a dataset $\set{X} = \{ \vect{x}_i \}_{i=1}^{n}$ by computing
\begin{equation}
	\tag{$K$-means}
	\min_{\bm{\set{C}}_K}
	\sum_{k=1}^{K} \sum_{\vect{x}_i \in \set{C}_k}
	\norm{\vect{x}_i - \tfrac{1}{|\set{C}_k|} \sum_{\vect{x}_j \in \set{C}_k} \vect{x}_j }{2}^2 ,
	\label[problem]{eq:kmeans}
\end{equation}
where $\bm{\set{C}}_K = \{ \set{C}_k\}_{k=1}^{K}$ is a partition of $\set{X}$, i.e., $\set{C}_{k} \cap \set{C}_{k'} = \emptyset$ and  $\bigcup_k \set{C}_k = \set{X}$.
Albeit its popularity, it is known to be NP-Hard and, in practice, users employ an heuristic \citep[][originally developed in 1957]{Lloyd1982} to find a solution.
The objective function of K-means, henceforth denoted $J_K$, can be rewritten (dropping the  terms that are constant with respect to $\set{C}_k$) as
\begin{equation}
	J_K = 
	- \sum_{k=1}^{K}
	\tfrac{1}{|\set{C}_k|} \sum_{\vect{x}_i, \vect{x}_j \in \set{C}_k} \transpose{\vect{x}}_i \vect{x}_j .
\end{equation}

Let $\mat{X} \in \Real^{d \times n}$ be the matrix formed by horizontally concatenating the vectors in $\set{X}$.
Let $\vect{z}_k \in \{ 0, 1 \}^{n}$ be the indicator vector of set $\set{C}_k$.
Let $\mat{Y}$ be the $k \times n$ matrix with rows $\vect{Y}_{k:} = |\set{C}_k|^{-1/2} \vect{z}_k$. We have
\begin{subequations}
\begin{align}
	J_K 
	&= - \sum_{k=1}^{K} \tfrac{1}{|\set{C}_k|} \sum_{i, j} \transpose{\vect{x}}_i \vect{x}_j \cdot (\transpose{\vect{z}}_k \vect{z}_k)_{ij} \\
	&= - \sum_{i, j} (\transpose{\mat{X}} \mat{X})_{ij} (\transpose{\mat{Y}} \mat{Y})_{ij} \\
	&= - \traceone{\transpose{\mat{X}} \mat{X} \transpose{\mat{Y}} \mat{Y}} .
	\label{eq:trace_xtxyty}
\end{align}
\end{subequations}

By construction, the matrix $\mat{Q} = \transpose{\mat{Y}} \mat{Y}$ exhibits the following properties
\begin{align}
	\mat{Q} \vect{1} &= \vect{1} , \\
	\traceone{\mat{Q}} &= K .
\end{align}

Let $\mat{D}$ be the Gramian matrix, i.e., $\mat{D} = \transpose{\mat{X}} \mat{X}$.
We can then re-cast $K$-means as the optimization problem
\begin{equation}
	\max_{\mat{Y} \in \set{V}_{\mat{Y}}^{k \times n}}
	\traceone{\mat{D} \mat{Q}}
	\quad\text{s.t.}\quad
	\begin{aligned}
		& \mat{Q} \vect{1} = \vect{1} , \\
		& \traceone{\mat{Q}} = K , \\
		& \mat{Q} = \transpose{\mat{Y}} \mat{Y} .
	\end{aligned}
	\label[problem]{eq:kmeans_qyyt}
\end{equation}
where $\set{V}_{\mat{Y}} = \{ 0 \} \cup \left\{ |\set{C}_k|^{-1/2} \right\}_{k=1}^{K}$.
Seeking to apply the desirable properties of SDP to $K$-means, we can pose \citep{Kulis2007,Peng2007_sdk-kmeans}
\begin{equation}
	\max_{\mat{Q} \in \Real^{n \times n}}
	\traceone{\mat{D} \mat{Q}}
	\quad\text{s.t.}\quad
	\begin{aligned}
		& \mat{Q} \vect{1} = \vect{1} , \\	
		& \traceone{\mat{Q}} = K , \\
		& \rank{\mat{Q}} = K , \\
		& \mat{Q} \succeq 0 ,
		\mat{Q} \geq \mat{0} .
	\end{aligned}
	\label[problem]{eq:sdp_kmeans_lowrank}	
\end{equation}
where mixed-integer program is relaxed into the real-valued nonnegative program, directly optimizing over $\mat{Q}$. 
SDP-KM is as a relaxation of this problem, simply obtained by removing the rank constraint.

\section{On the complete positivity of SDP-KM solutions on circulant matrices}
\label{sec:proofs}



\begin{proposition}
	If the solution $\mat{Q}_*$ to SDP-KM is a circulant matrix,
	then it is CP for every $K \leq 3/2$ or $K \geq \tfrac{n}{2}$.
	\label{prop:circulantCP}%
\end{proposition}

\begin{proof}
	For $K \leq 3/2$, plugging the constraint $\mat{Q}_* \vect{1} = \vect{1}$ into Corollary 2.6 in \citep[][p.~7]{So2013} gives the desired result.
	
	Let us address $K \geq \tfrac{n}{2}$.
	\citet{Kaykobad1987} proved that every diagonally dominant matrix $\mat{A}$, i.e.,
	$|(\mat{A})_{ii}| \geq \sum_{j \neq i} |(\mat{A})_{ij}|$ for all $i$, is a CP matrix.
	We have to prove then that $\mat{Q}_* \geq \mat{0}$ is diagonally dominant.
	We have $\traceone{\mat{Q}_*} = K$ and, since $\mat{Q}_*$ is circulant, all $(\mat{Q}_*)_{ii}$ have the same value. Then, $(\mat{Q}_*)_{ii} = K/n$.
	From $\mat{Q}_* \vect{1} = \vect{1}$,
	$\sum_{j \neq i} (\mat{Q}_*)_{ij} = 1 - (\mat{Q}_*)_{ii} = 1 - K/n$.
	Hence, $\mat{Q}_*$ is diagonally dominant for $K \geq \tfrac{n}{2}$.
\end{proof}

\section{Additional results}
\label{sec:additional_multilayer_results}

We include additional results of the multi-layer NOMAD algorithm using different values of $k$ in each layer.

\begin{figure}
    \includegraphics[width=.618\linewidth]{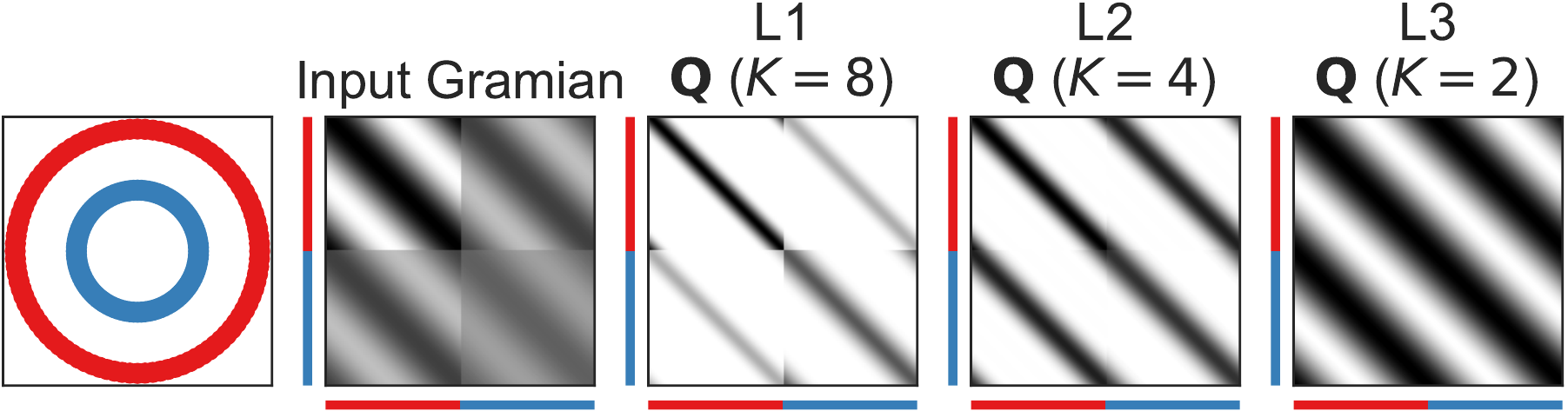}%
    \\[2pt]
    \includegraphics[width=.618\linewidth]{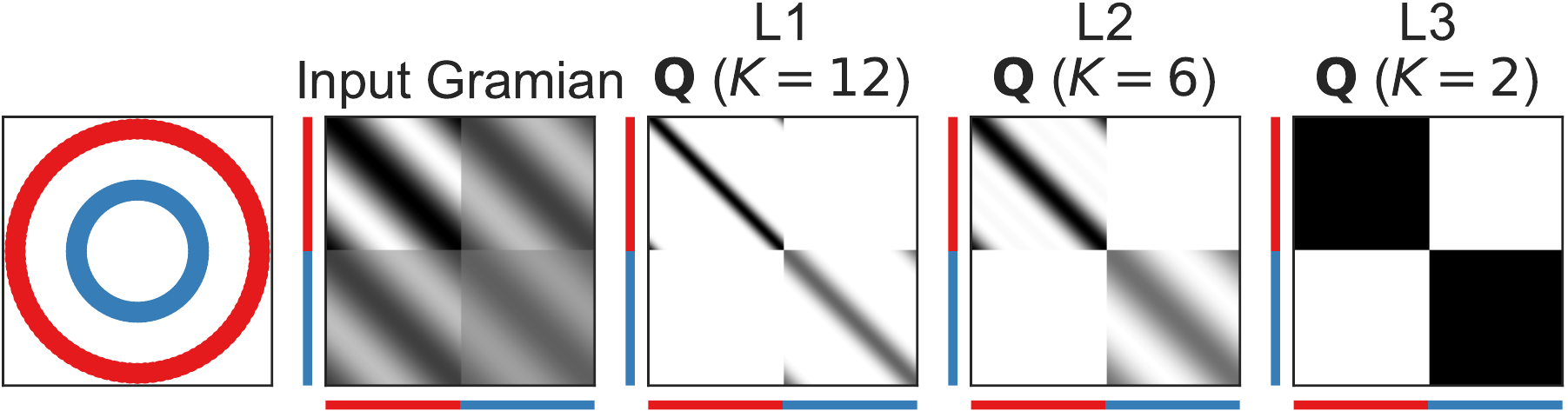}%
    \\[2pt]
	\includegraphics[width=.618\linewidth]{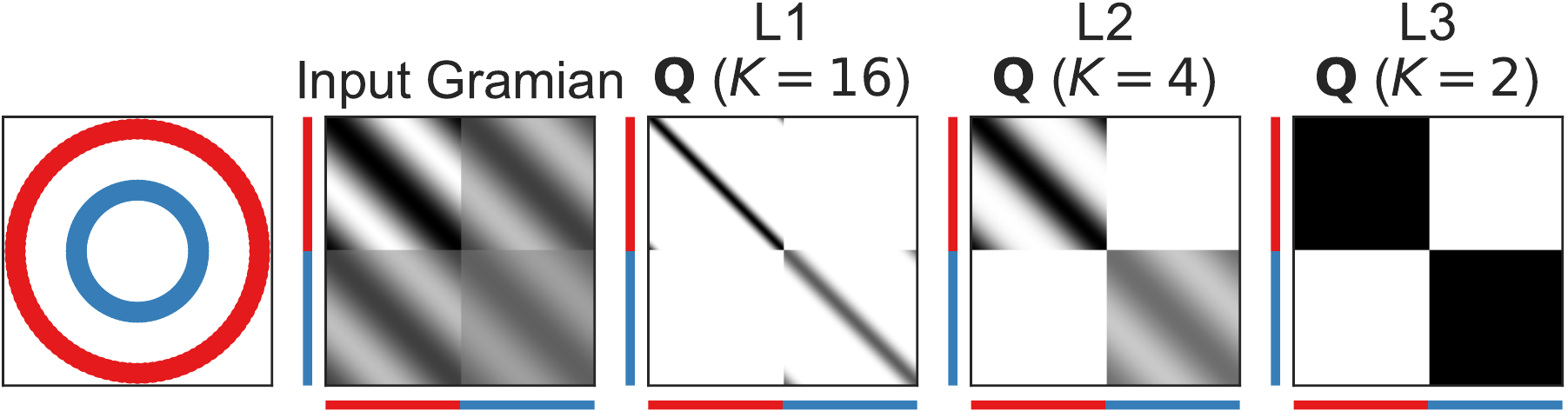}%
    \\[2pt]
	\includegraphics[width=.744\linewidth]{multilayer/circles-16-8-4-2}%
    \\[2pt]
    \includegraphics[width=\linewidth]{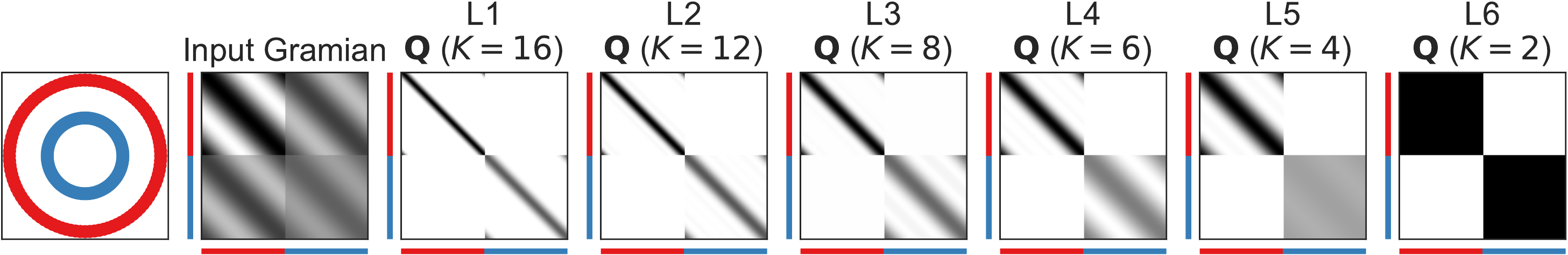}%
    \\[2pt]
	\includegraphics[width=.873\linewidth]{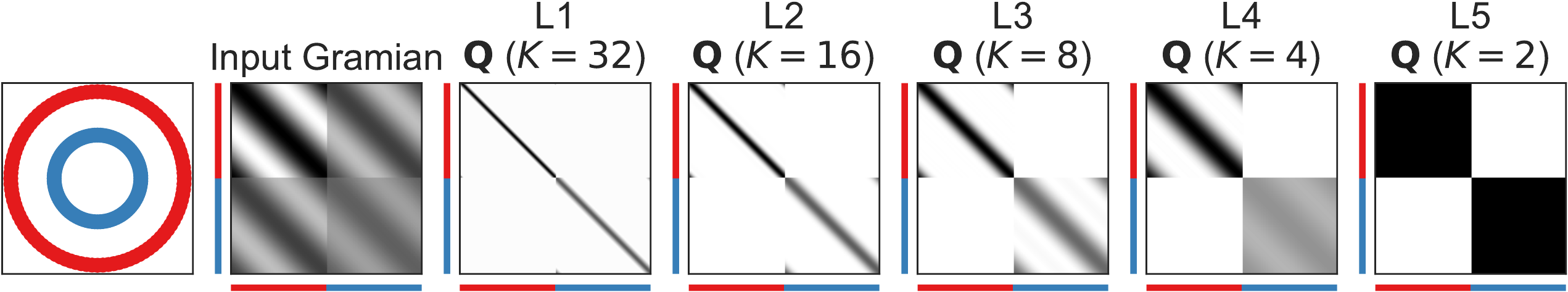}%
    \\[2pt]

	\caption{Additional results of multi-layer NOMAD.}
    \label{fig:additional_multilayer_rings}
\end{figure}

\begin{figure}
    \includegraphics[width=.618\linewidth]{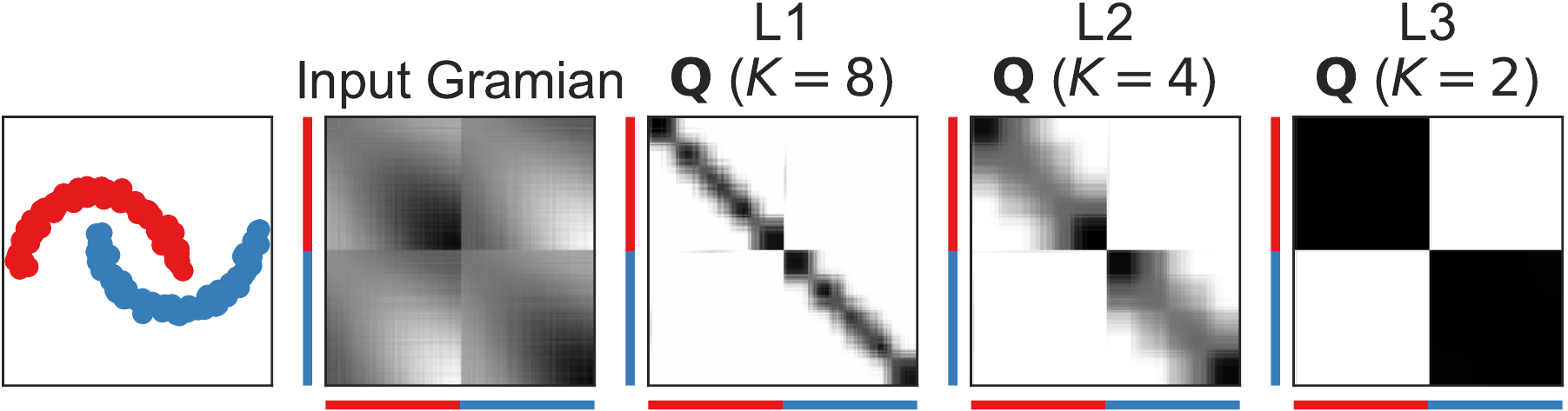}%
    \\[2pt]
    \includegraphics[width=.618\linewidth]{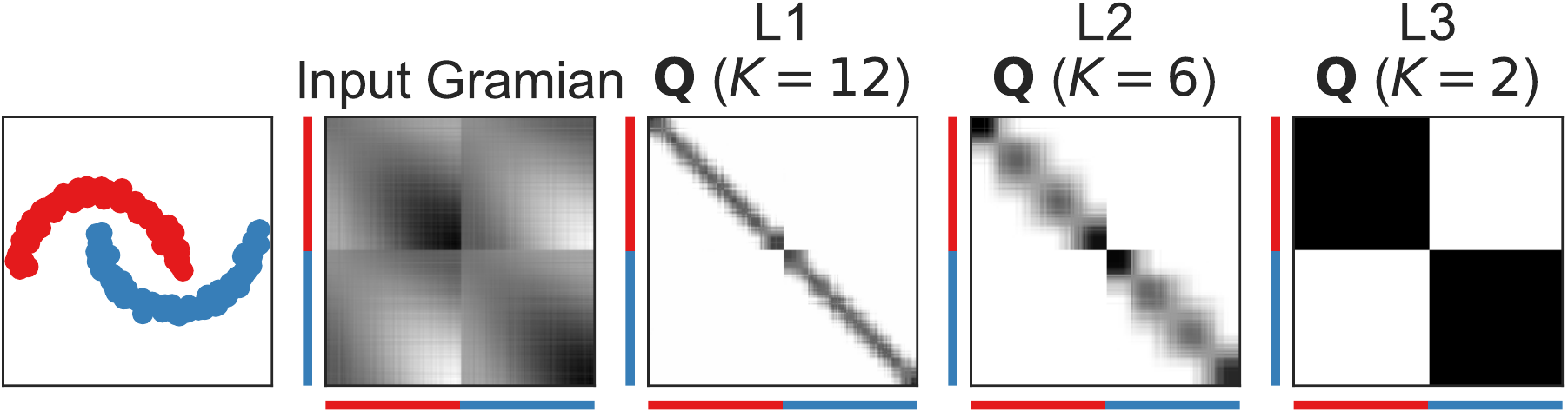}%
    \\[2pt]
	\includegraphics[width=.618\linewidth]{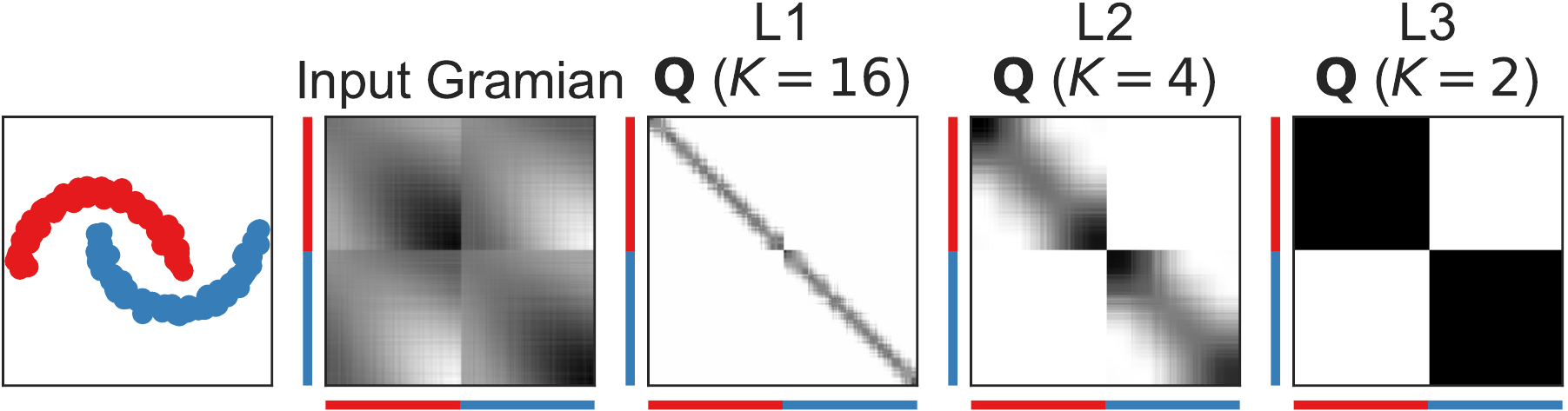}%
    \\[2pt]
    \includegraphics[width=\linewidth]{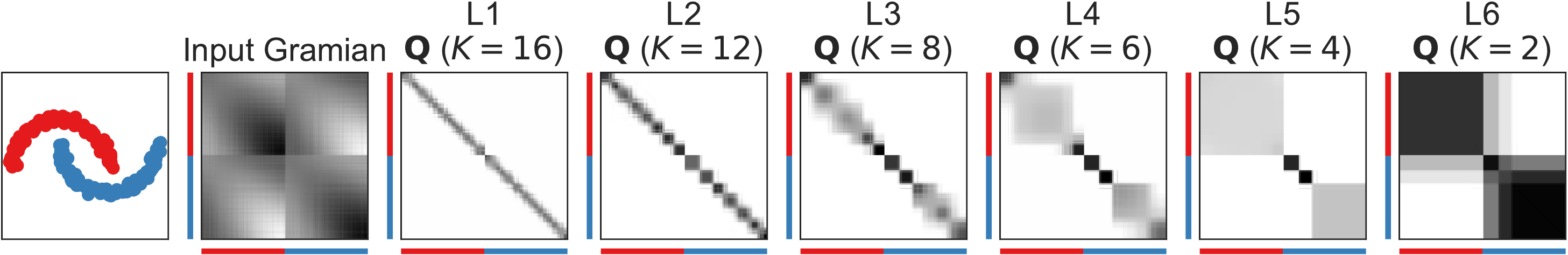}%
    \\[2pt]
	\includegraphics[width=.873\linewidth]{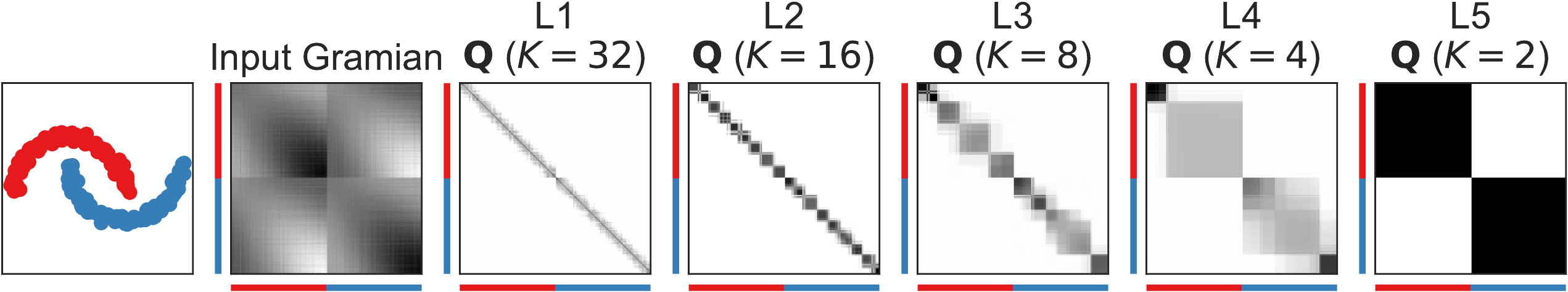}%

	\caption{Additional results of multi-layer NOMAD.}
    \label{fig:additional_multilayer_moons}
\end{figure}

\begin{figure}
    \includegraphics[width=.618\linewidth]{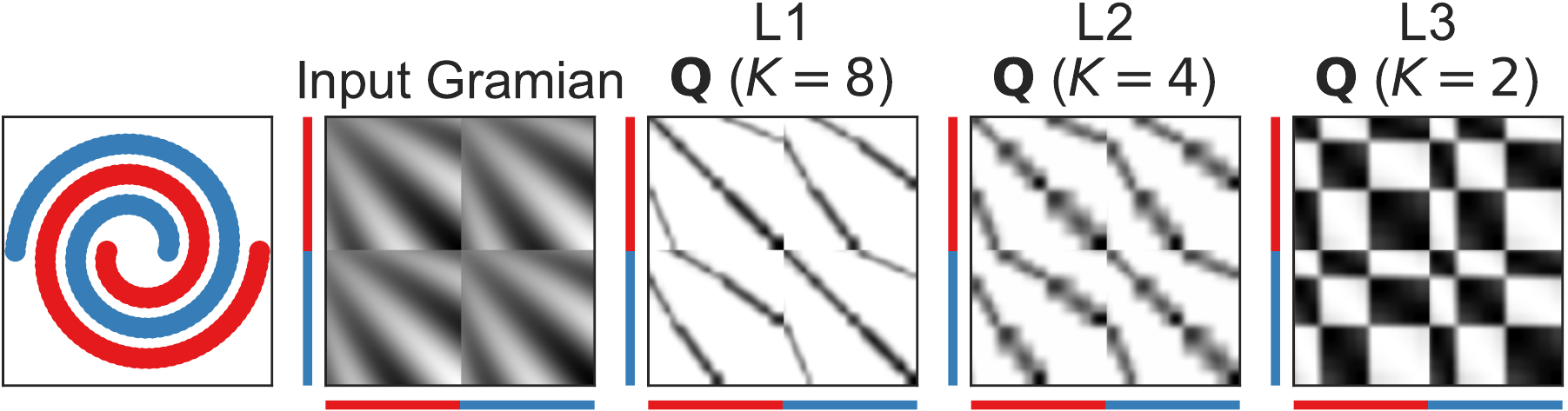}%
    \\[2pt]
    \includegraphics[width=.618\linewidth]{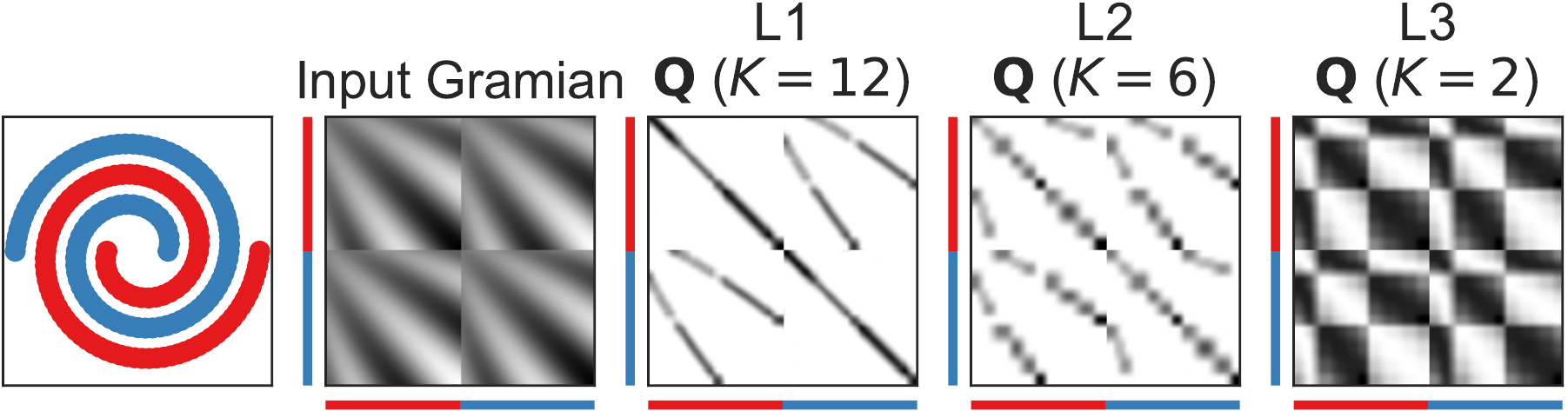}%
    \\[2pt]
	\includegraphics[width=.618\linewidth]{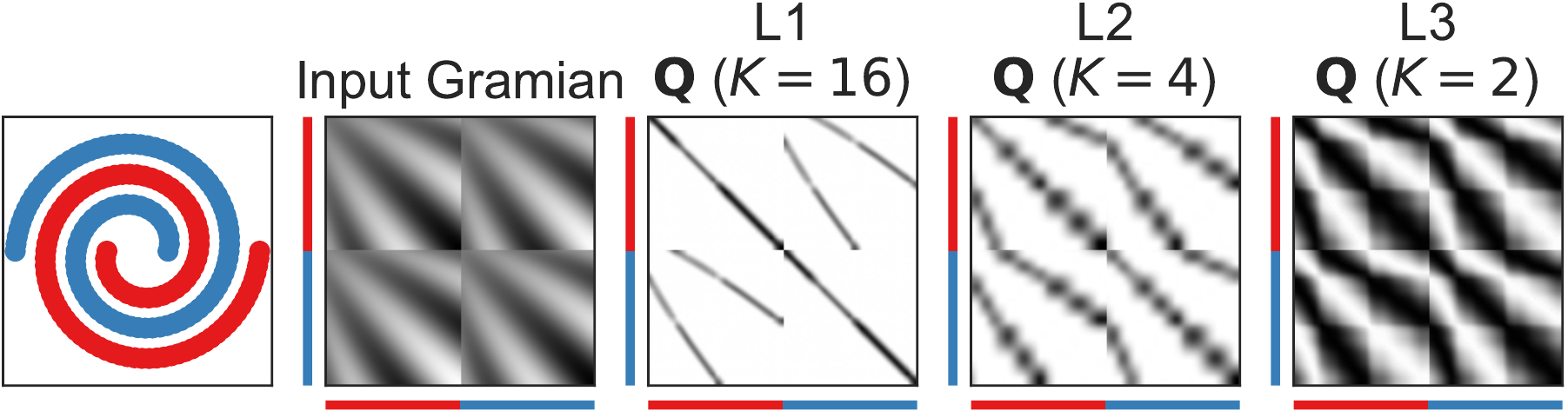}%
    \\[2pt]
    \includegraphics[width=\linewidth]{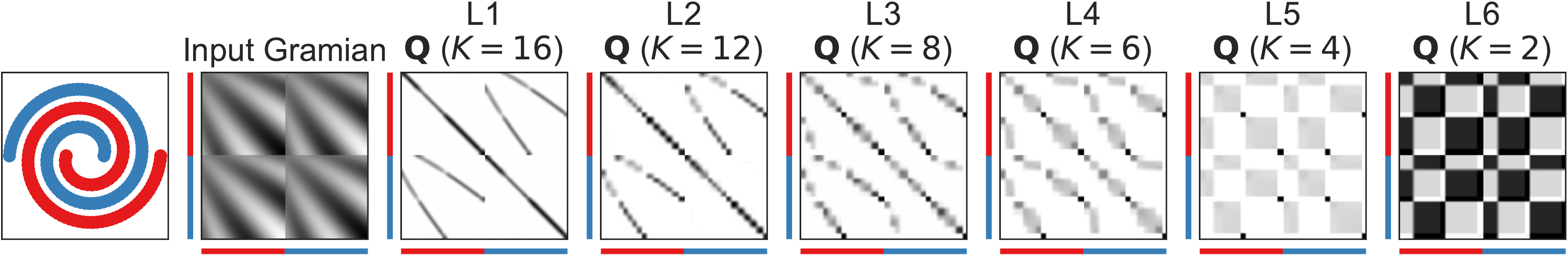}%
    \\[2pt]
	\includegraphics[width=.873\linewidth]{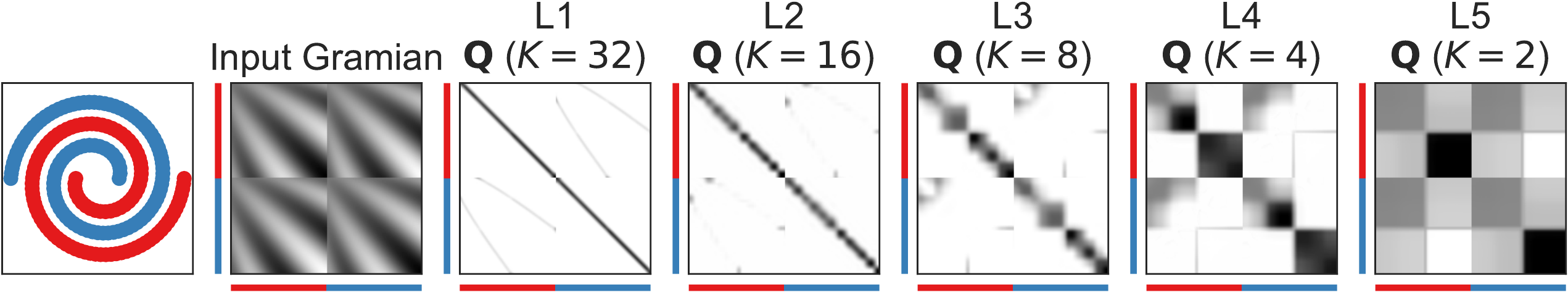}%
    \\[2pt]
	\includegraphics[width=.744\linewidth]{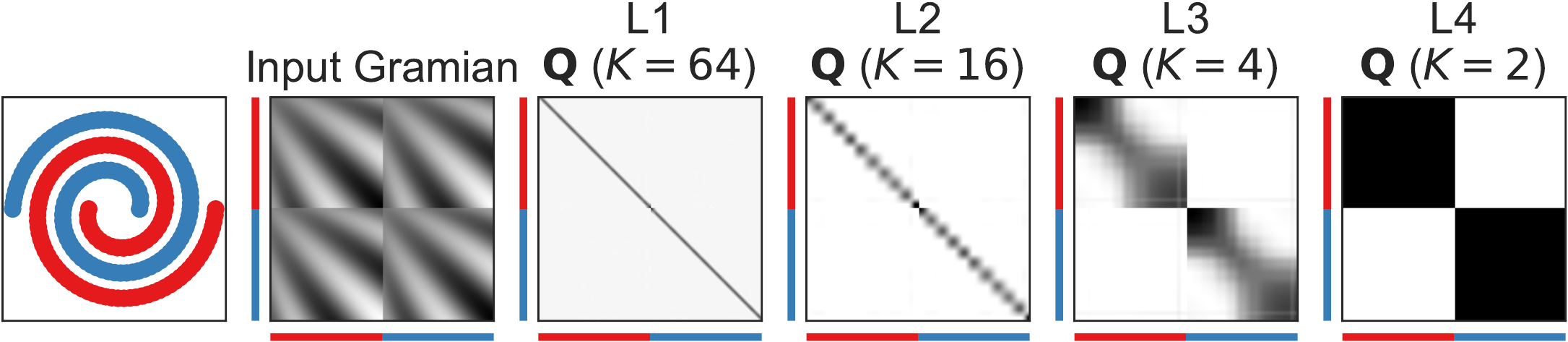}%
    \\[2pt]
	\includegraphics[width=.744\linewidth]{multilayer/double_swiss_roll-64-16-4-2}%
    \\[2pt]
	\includegraphics[width=\linewidth]{multilayer/double_swiss_roll-64-32-16-8-4-2}%

	\caption{Additional results of multi-layer NOMAD.}
    \label{fig:additional_multilayer_double_swiss_roll}
\end{figure}

\section{Symmetric NMF}
\label{sec:snmf}

In this section, we present the algorithm used to compute the symmetric NMF  of a matrix $\mat{A} \in \Real^{n \times n}$, defined as
\begin{equation}
	\min_{\mat{Y} \in \Real^{n \times r}}
	\norm{\mat{A} - \mat{Y} \transpose{\mat{Y}}}{F}^2
	\quad\text{s.t.}\quad
	\mat{Y} \geq \mat{0}.
	\tag{SNMF}
	\label[problem]{eq:snmf}
\end{equation}
We use the alternating direction method of multipliers (ADMM) to solve it.
In short, ADMM solves convex optimization problems by breaking them into smaller subproblems, which are individually easier to handle. It has also been extended to handle non-convex problems, e.g., to solve several flavors of NMF~\citep{Fevotte2011,xu2012nmf,Tepper2014consensus,tepper2016compressed}.

\cref{eq:snmf} can be equivalently re-formulated as
\begin{equation}
	\min_{\mat{Y} \in \Real^{n \times r}}
	\norm{\mat{A} - \mat{Y} \transpose{\mat{X}}}{F}^2
	\quad\text{s.t.}\quad
	\mat{Y} = \mat{X},\
	\mat{Y} \geq \mat{0},\ \mat{X} \geq \mat{0} ,
\end{equation}
and we consider its augmented Lagrangian,
\begin{equation}
	\mathscr{L} \left( \mat{X}, \mat{Y}, \mat{\Gamma} \right) =
	\tfrac{1}{2} \norm{\mat{A} - \mat{Y} \transpose{\mat{X}}}{F}^2
	+ \tfrac{\sigma}{2} \norm{\mat{Y} - \mat{X}}{F}^2
	- \traceone{ \transpose{\mat{\Gamma}} \left( \mat{Y} - \mat{X} \right)}
\end{equation}
where $\mat{\Gamma}$ is a Lagrange multiplier, $\sigma$ is a penalty parameter.

The ADMM algorithm works in a coordinate descent fashion, successively minimizing $\mathscr{L}$ with respect to $\mat{X}, \mat{Y}$, one at a time while fixing the other at its most recent value and then updating the multiplier $\mat{\Gamma}$.
For the problem at hand, these steps are
\begin{subequations}
\begin{align}
	\mat{X}^{(t+1)} &= \argmin_{\mat{X} \geq \mat{0}} \mathcal{L} \left( \mat{X}, \mat{X}^{(t)}, \mat{\Gamma}^{(t)} \right) ,\\
	\mat{Y}^{(t+1)} &= \argmin_{\mat{Y} \geq \mat{0}} \mathcal{L} \left( \mat{X}^{(t+1)}, \mat{Y}, \mat{\Gamma}^{(t)} \right) ,\\
	\mat{\Gamma}^{(t+1)} &= \mat{\Gamma}^{(t)} - \eta \sigma \left( \mat{X}^{(t+1)} - \mat{Y}^{(t+1)} \right) .
\end{align}
\end{subequations}
In our experiments, we fix $\eta$ and $\sigma$ to 1. We initialize the algorithm with a random matrix.

\section{Non-convex SDP solver}
\label{sec:burer-monteiro}

We follow the algorithm proposed in \cite{Kulis2007,Hou2015} to solve \cref{eq:sdp_kmeans_lowrank_fast}. Our approach has a small but fundamental difference: instead of setting $r = K$, we allow for $r \geq K$.
We define the augmented Lagrangian of \cref{eq:sdp_kmeans_lowrank_fast} as
\begin{equation}
	\begin{aligned}
		\mathcal{L} \left( \mat{Y}, \vect{\mu}, \lambda \right) =
		& -\traceone{\mat{D} \transpose{\mat{Y}} \mat{Y}}
		+ \tfrac{\sigma}{2} \norm{\transpose{\mat{Y}} \mat{Y} \vect{1} - \vect{1}}{2}^2
		- \transpose{\vect{\mu}} (\transpose{\mat{Y}} \mat{Y} \vect{1} - \vect{1}) \\
		& + \tfrac{\varphi}{2} (\traceone{\transpose{\mat{Y}} \mat{Y}} - K)^2
		- \vect{\lambda} (\traceone{\transpose{\mat{Y}} \mat{Y}} - K) ,
	\end{aligned}
\end{equation}
where $\vect{\mu}, \vect{\lambda}$ are Lagrange multipliers, $\sigma, \varphi$ are penalty parameters.
We obtain $\mat{Y}$ by running the steps
\begin{subequations}
\begin{align}
	\mat{Y}^{(t+1)} &= \argmin_{\mat{Y} \geq \mat{0}} \mathcal{L} \left( \mat{Y}, \vect{\mu}^{(t)}, \lambda^{(t)} \right) ,\\
	\vect{\mu}^{(t+1)} &= \vect{\mu}^{(t)} - \eta \sigma \left( \transpose{\mat{Y}} \mat{Y} \vect{1} - \vect{1} \right) ,\\
	\lambda^{(t+1)} &= \lambda^{(t)} - \eta \varphi \left( \traceone{\transpose{\mat{Y}} \mat{Y}} - K \right) .
\end{align}
\end{subequations}
This is a non-standard approach since the minimization over $\mat{Y}$ (the gradient $\partial \mathcal{L} / \partial \mat{Y}$ is given in \cite{Hou2015}) is a non-convex problem.
Although there are no guarantees about the convergence of the procedure, theoretical assurances for related problems have been presented in \cite{boumal2016non}.
To perform the minimization with respect to $\mat{Y}$, we use the L-BFGS-B algorithm \citep{Byrd1995} with bound constraints ($(\mat{Y})_{ij} \in [0, 1]$). Finally, the initialization to the overall iterative algorithm is done with symmetric nonnegative matrix factorization, see \cref{sec:snmf}.
In our experiments, we fix $\eta$, $\varphi$, and $\sigma$ to 1 and prenormalize $\mat{D}$ (dividing by its Frobenius norm).

\section{Proofs for the conditional gradient algorithm}
\label{sec:cgm_proofs}

This section closely follows the works of \citet{Hazan2008sdp} and \citet{Clarkson2010fw}.

\begin{lemma}
	The dual objective of
    \begin{equation}
		\max_{\mat{Z}}
    	f(\mat{Z})
		\quad\text{s.t.}\quad
    	\mat{Z} \vect{b} = \vect{0} ,
    	\quad
		\traceone{\mat{Z}} = k - 1 ,
   		\quad
		\mat{Z} \succeq \mat{0}
	\end{equation}
    is
    \begin{equation}
    	\min_{\mat{Z}} w(\mat{Z}) ,
    \end{equation}    
    where
    \begin{align}
    	w(\mat{Z}) &= s \, \phi(\mat{Z}) + f(\mat{Z}) - \traceone{\mat{Z} \nabla f(\mat{Z})} ,
        \label{eq:generic_sdp_ortho_dual_w} \\
    	\phi(\mat{Z}) &=
        \max_{\substack{
            \norm{\vect{v}}{2} = 1\\
            \transpose{\vect{v}} \vect{b} = 0
        }}
        \transpose{\vect{v}} \nabla f(\mat{Z}) \vect{v} .
        \label{eq:generic_sdp_ortho_dual_phi}
    \end{align}
    \label[lemma]{thm:generic_sdp_ortho_dual}
\end{lemma}

\proof{
	The Lagrangian relaxation of \cref{eq:generic_sdp_ortho} is
    \begin{equation}
    \begin{aligned}
    	\ell = -\max_{\mat{\Psi}, \psi, \vect{y}}
        \min_{\mat{Z}}
        &\ -f(\mat{Z}) - \traceone{\mat{\Psi} \mat{Z}} \\
        &+ \psi (\traceone{\mat{Z}} - s) + \transpose{\vect{y}} \mat{Z} \vect{b} ,
        \label{eq:generic_sdp_ortho_lagrangian}
    \end{aligned}
    \end{equation}
    where $\mat{\Psi} \succeq \mat{0}$, $\psi \in \Real$, and $\vect{y} \in \Real^n$ are Lagrange multipliers.
    Differentiating and equating to zero we get
    $\mat{\Psi} = -\nabla f(\mat{Z}) + \psi \mat{I} + \vect{y} \transpose{\vect{b}}$.
    Note that $\mat{\Psi} \succeq \mat{0}$ implies that $\vect{y} = c \vect{b}$ for some $c \in \Real$, yielding
    \begin{equation}
    	\mat{\Psi} = -\nabla f(\mat{Z}) + \psi \mat{I} + c \vect{b} \transpose{\vect{b}} \succeq \mat{0} .
        \label{eq:generic_sdp_ortho_U}
    \end{equation}    
    Plugging \cref{eq:generic_sdp_ortho_U} and $\vect{y} = c \vect{b}$ in \cref{eq:generic_sdp_ortho_lagrangian} we get 
    \begin{align}
    	\phantom{\ell}
        &\begin{aligned}
    		\mathllap{\ell}
        	=
			\max_{\psi}
          \min_{\mat{Z}}
          \ &f(\mat{Z}) - \traceone{\left( \nabla f(\mat{Z}) + \psi \mat{I} + c \vect{b} \transpose{\vect{b}} \right) \mat{Z}} \\
          &+ \psi (\traceone{\mat{Z}} - s) + c \transpose{\vect{b}} \mat{Z} \vect{b} \\
        \end{aligned}
        \\
        &= \max_{\psi \in \Real}
        \min_{\mat{Z}}
        f(\mat{Z}) - \traceone{\nabla f(\mat{Z}) \ \mat{Z}} + \psi s \\
        &= \max_{\substack{
        	\mat{Z} \\
            \psi \in \Real
        }}
        f(\mat{Z}) - \traceone{\nabla f(\mat{Z}) \ \mat{Z}} - \psi s .
    \end{align}
    
    From \cref{eq:generic_sdp_ortho_U}, $\psi \mat{I} \succeq \nabla f(\mat{Z}) - c \vect{b} \transpose{\vect{b}}$, implying
    \begin{subequations}
    \begin{align}
    	 \psi
         &\geq \lambda_{\text{max}} \left\{ \nabla f(\mat{Z}) - c \vect{b} \transpose{\vect{b}} \right\} \\
         &\geq \lambda_{\text{max}} \left\{ \nabla f(\mat{Z}) - d \vect{b} \transpose{\vect{b}} \right\}
         \quad \forall d > c \\
         &\geq \phi(\mat{Z}).
    \end{align}    
    \end{subequations}
    The last inequality comes from taking $d \rightarrow +\infty$, thus shifting the eigenvalue associated with $\vect{b}$ (should there be one) away from the maximum.
    Without loss of generality, we set $\psi = \phi(\mat{Z})$, finally obtaining $\ell = \min_{\mat{Z}} w(\mat{Z})$.
}

\begin{proposition}
Let $\mat{X}, \mat{Z} \in \set{P}_s$ and $\mat{Y} = \mat{X} + \alpha (\mat{Z} - \mat{X})$ and $\alpha \in \Real$.
    The curvature constant of $f$ is
	\begin{equation}
	C_f \defeq
    \sup_{\mat{X}, \mat{Z}, \alpha}
    \tfrac{1}{\alpha^2}
        [ f(\mat{X}) - f(\mat{Y}) + \tracetwo{(\mat{Y} - \mat{X})}{\nabla f(\mat{X})} ] .
	\end{equation}    
	Let $\mat{Z}^{\star}$ be the solution to 
    \begin{equation}
		\max_{\mat{Z}}
    	f(\mat{Z})
		\quad\text{s.t.}\quad
    	\mat{Z} \vect{b} = \vect{0} ,
    	\quad
		\traceone{\mat{Z}} = k - 1 ,
   		\quad
		\mat{Z} \succeq \mat{0} .
	\end{equation}
	The iterates $\mat{Z}_t$ of \cref{algo:cgm_generic_sdp_ortho} satisfy for all $t > 1$
    \begin{equation}
    	f(\mat{Z}^{\star}) - f(\mat{Z}_t) \leq  \tfrac{8 C_f}{t+2} .
    \end{equation}
\end{proposition}

\proof{
    Let
    \begin{equation}
    	h(\mat{Z}) = \tfrac{1}{4 C_f} [ f(\mat{Z}^{\star}) - f(\mat{Z}) ] .
    \end{equation}
    Proving that $h(\mat{Z}_{t}) \leq \tfrac{2}{t + 2}$ yields the desired result.
	From \cref{thm:generic_sdp_ortho_dual}, we have
    \begin{equation}
    	w(\mat{Z}) \geq w(\mat{Z}^{\star}) \geq f(\mat{Z}^{\star}) \geq f(\mat{Z}) .
        \label{eq:weak_duality}
    \end{equation}
	By \cref{eq:generic_sdp_ortho_dual_w,eq:generic_sdp_ortho_dual_phi}, we have
    \begin{align}
    	\transpose{\vect{v}}_t \nabla f(\mat{Z}_t) \vect{v}_t
        &= \phi(\mat{Z}) \\
        &= w(\mat{Z}_t) - f(\mat{Z}_t) + \traceone{\mat{Z} \ \nabla f(\mat{Z}_t)} .
    \end{align}
    Therefore,
    \begin{equation}
    	\traceone{(\mat{H}_t - \mat{Z}_t) \ \nabla f(\mat{Z}_t)} = w(\mat{Z}_t) - f(\mat{Z}_t) .
        \label{eq:w_minus_f}
    \end{equation}
    Now, using \cref{eq:curvature_constant}
    \begin{subequations}
    \begin{align}
    	f(\mat{Z}_{t+1})
        &= f(\mat{Z}_{t} + \alpha_t ( \mat{H}_t - \mat{Z}_{t} ) ) \\
        &\geq f(\mat{Z}_{t}) + \alpha_t \traceone{( \mat{H}_t - \mat{Z}_{t}) \ \nabla f(\mat{Z}_t) } + \alpha^2 C_f .
        \label{eq:f_inequality1}
    \end{align}
    \end{subequations}    
    Putting \cref{eq:f_inequality1} together with \cref{eq:weak_duality,eq:w_minus_f},
    \begin{subequations}
    \begin{align}
    	f(\mat{Z}_{t+1})
        &\geq f(\mat{Z}_{t}) + \alpha_t ( w(\mat{Z}_t) - f(\mat{Z}_t) ) - \alpha_t^2 C_f \\
        &\geq f(\mat{Z}_{t}) + \alpha_t ( f(\mat{Z}^{\star}) - f(\mat{Z}_t) ) - \alpha_t^2 C_f \\
        %
        &= f(\mat{Z}_{t}) + 4 \alpha_t C_f h(\mat{Z}_{t}) - \alpha_t^2 C_f .
        \label{eq:f_inequality}
    \end{align}
    \end{subequations}
    From the definition of $h(\mat{Z}_{t})$, this implies
    \begin{equation}
    	h(\mat{Z}_{t+1}) \leq h(\mat{Z}_{t}) - \alpha_t h(\mat{Z}_{t}) + \tfrac{\alpha_t^2}{4} .
    \end{equation}
    Finally, we prove inductively that $h(\mat{Z}_{t}) \leq \tfrac{2}{t + 2}$. In the first iteration, $\alpha_1 = 1$ and $h(\mat{Z}_{2}) \leq \tfrac{1}{4}$. By taking $\alpha = \tfrac{2}{t + 2}$, we have
    \begin{subequations}
    \begin{align}
    	h(\mat{Z}_{t+1})
        &\leq h(\mat{Z}_{t}) - \alpha_t h(\mat{Z}_{t}) + \tfrac{\alpha_t^2}{4} \\
        &\leq (1 - \alpha_t) h(\mat{Z}_{t}) + \tfrac{\alpha_t^2}{2} \\
        &= (1 - \tfrac{2}{t + 2}) \tfrac{2}{t + 2} + \tfrac{2}{(t + 2)^2} \leq \tfrac{2}{t + 3} .
    \end{align}
    \end{subequations}
}

\end{document}